\documentclass{article}
\usepackage{natbib}
\usepackage{fullpage}
\setcitestyle{authoryear,round,citesep={;},aysep={,},yysep={;}}

\usepackage[utf8]{inputenc} 
\usepackage[T1]{fontenc}    
\usepackage{hyperref}       
\usepackage{url}            
\usepackage{booktabs}       
\usepackage{amsfonts}       
\usepackage{nicefrac}       
\usepackage{microtype}      
\usepackage[dvipsnames]{xcolor}
\usepackage{amsmath}
\usepackage{graphicx}
\usepackage{wrapfig}

\usepackage{parskip}

\usepackage{tikz}
\usetikzlibrary{positioning}

\usepackage{caption}
\usepackage{subcaption}

\usepackage{algorithm}
\usepackage{forloop}
\usepackage{algpseudocode}

\usepackage{overpic}

\usepackage{amssymb}
\usepackage{pifont}
\newcommand{\cmark}{\ding{51}}%
\newcommand{\xmark}{\ding{55}}%

\hypersetup{
    colorlinks=true,
    linkcolor=blue,
    filecolor=magenta,
    urlcolor=magenta,
    citecolor=RedViolet,
}

\usepackage{amsthm}
\newtheorem{theorem}{Theorem}[section]

\newtheorem{lemma}[theorem]{Lemma}
\newtheorem{remark}[theorem]{Remark}
\newtheorem{statement}{Statement}[section]

\newtheorem{definition}{Definition}[section]
\newtheorem{assumption}{Assumption}[section]

\usepackage{tcolorbox}
\definecolor{tabblue}{RGB}{123,192,230}
\definecolor{dkgreen}{rgb}{0,0.6,0}

\DeclareMathOperator*{\argmax}{arg\,max}
\DeclareMathOperator*{\argmin}{arg\,min}

\newcommand{\colorizeoperator}[2]{%
  \begingroup\def\qopname##1##2##3{%
    \xdef#1{%
      \noexpand\qopname
      \unexpanded{##1}%
      ##2%
      {\begingroup\noexpand\color{#2}##3\endgroup}%
    }%
  }%
  #1%
  \endgroup
}
\makeatletter
\newcommand{\DeclareColoredMathOperator}{%
  \@ifstar
    {\def\DCMO@@{\DeclareMathOperator*}\DCMO@}
    {\def\@DCMO{\DeclareMathOperator}\DCMO@}%
}
\newcommand\DCMO@[3]{%
  \DCMO@@{#1}{\begingroup\color{#3}#2\endgroup}%
}
\makeatother
\DeclareColoredMathOperator*{\redargmin}{arg\,min}{purple}
\DeclareColoredMathOperator*{\blueargmax}{arg\,max}{teal}
\DeclareMathOperator*{\rbarg}{arg}
\DeclareColoredMathOperator*{\rbmin}{min}{purple}
\DeclareColoredMathOperator*{\rbmax}{max}{teal}

\newcommand{\norm}[1]{\left|\left| #1 \right|\right|}
\newcommand{\pp}[1]{\left( #1 \right)}

\newcommand{\cS}{\mathcal{S}}

\newcommand{\LV}{\mathcal{L}_V}

\newcommand{\boldw}{\mathbf{w}}
\newcommand{\boldv}{\mathbf{v}}
\newcommand{\boldu}{\mathbf{u}}
\newcommand{\boldlam}{\boldsymbol{\lambda}}
\newcommand{\boldy}{\mathbf{y}}
\newcommand{\boldM}{\mathbf{M}}
\newcommand{\boldH}{\mathbf{H}}
\newcommand{\boldJ}{\mathbf{J}}
\newcommand{\boldQ}{\mathbf{Q}}
\newcommand{\boldI}{\mathbf{I}}
\newcommand{\boldZ}{\mathbf{Z}}

\newcommand{\boldU}{\mathbf{U}}

\newcommand{\boldzero}{\boldsymbol{0}}

\newcommand{\boldPhi}{\mathbf{\Phi}}
\newcommand{\boldphi}{\boldsymbol{\phi}}

\newcommand{\boldA}{\mathbf{A}}
\newcommand{\boldB}{\mathbf{B}}
\newcommand{\boldC}{\mathbf{C}}
\newcommand{\boldD}{\mathbf{D}}
\newcommand{\boldP}{\mathbf{P}}

\newcommand{\boldd}{\mathbf{d}}
\newcommand{\bolde}{\mathbf{e}}
\newcommand{\boldb}{\mathbf{b}}

\newcommand{\boldr}{\mathbf{r}}
\newcommand{\boldf}{\mathbf{f}}

\newcommand{\boldh}{\mathbf{h}}

\newcommand{\inParams}{\boldw}
\newcommand{\outParams}{\boldu}
\newcommand{\inSpace}{\mathcal{W}}
\newcommand{\outSpace}{\mathcal{U}}
\newcommand{\inObj}{f}
\newcommand{\outObj}{F}
\newcommand{\hInObj}{\hat{\inObj}}

\title{On Implicit Bias in Overparameterized Bilevel Optimization}
\date{}

\author{Paul Vicol$^1$,\,
Jonathan Lorraine$^1$,\,
Fabian Pedregosa$^2$,\,
David Duvenaud$^1$,
Roger Grosse$^1$ \\ \\
$^1$University of Toronto \& Vector Institute \qquad $^2$Google Brain
}

\begin{document}

\maketitle

\begin{abstract}
Many problems in machine learning involve bilevel optimization (BLO), including hyperparameter optimization, meta-learning, and dataset distillation. Bilevel problems consist of two nested sub-problems, called the outer and inner problems, respectively. In practice, often at least one of these sub-problems is overparameterized. In this case, there are many ways to choose among optima that achieve equivalent objective values. Inspired by recent studies of the implicit bias induced by optimization algorithms in single-level optimization, we investigate the implicit bias of gradient-based algorithms for bilevel optimization. We delineate two standard BLO methods---cold-start and warm-start---and show that the converged solution or long-run behavior depends to a large degree on these and other algorithmic choices, such as the hypergradient approximation. We also show that the inner solutions obtained by warm-start BLO can encode a surprising amount of information about the outer objective, even when the outer parameters are low-dimensional. We believe that implicit bias deserves as central a role in the study of bilevel optimization as it has attained in the study of single-level neural net optimization.
\end{abstract}

\section{Introduction}
\label{sec:introduction}

Bilevel optimization (BLO) problems consist of two nested sub-problems, called the \textit{outer} and \textit{inner} problems respectively, where the outer problem must be solved subject to the optimality of the inner problem.
Let $\outParams \in \outSpace$, $\inParams \in \inSpace$ denote the outer and inner parameters, respectively, and let $\outObj, \inObj : \outSpace \times \inSpace \to \mathbb{R}$ denote the outer and inner objectives, respectively.
Using this notation, the bilevel problem is:
\begin{align}\label{eqn:intro-bilevel-def}
    \outParams^\star \in \text{``}\argmin_{\outParams \in \outSpace}\text{''} \outObj(\outParams, \inParams^\star_{\outParams})
    \qquad \text{s.t.} \qquad
    \inParams^\star_{\outParams} \in \cS(\outParams) = \argmin_{\inParams \in \inSpace} f(\outParams, \inParams) 
\end{align}
where $\cS : \outSpace \rightrightarrows \inSpace$ is the set-valued \textit{best-response mapping}.
We use quotes around the outer $\text{``}\argmin\text{''}$ to denote the ambiguity in the definition of the bilevel problem when the inner objective has multiple solutions, a point we expand upon in Section~\ref{sec:background}.
Examples of bilevel optimization in machine learning include hyperparameter optimization~\citep{domke2012generic,maclaurin2015gradient,mackay2019self,lorraine2020optimizing,shaban2019truncated}, dataset distillation~\citep{wang2018dataset,zhao2020dataset}, influence function estimation~\citep{koh2017understanding}, meta-learning~\citep{finn2017model,franceschi2018bilevel,rajeswaran2019meta}, example reweighting~\citep{bengio2009curriculum,ren2018learning}, neural architecture search~\citep{zoph2016neural,liu2018darts} and adversarial learning~\citep{goodfellow2014generative,pfau2016connecting} (see Table~\ref{table:innerO}).

\begin{figure*}[t]
    \centering
    \begin{subfigure}{0.28\textwidth}
        \includegraphics[trim={0cm 1.1cm 1.2cm 0},clip,width=\linewidth]{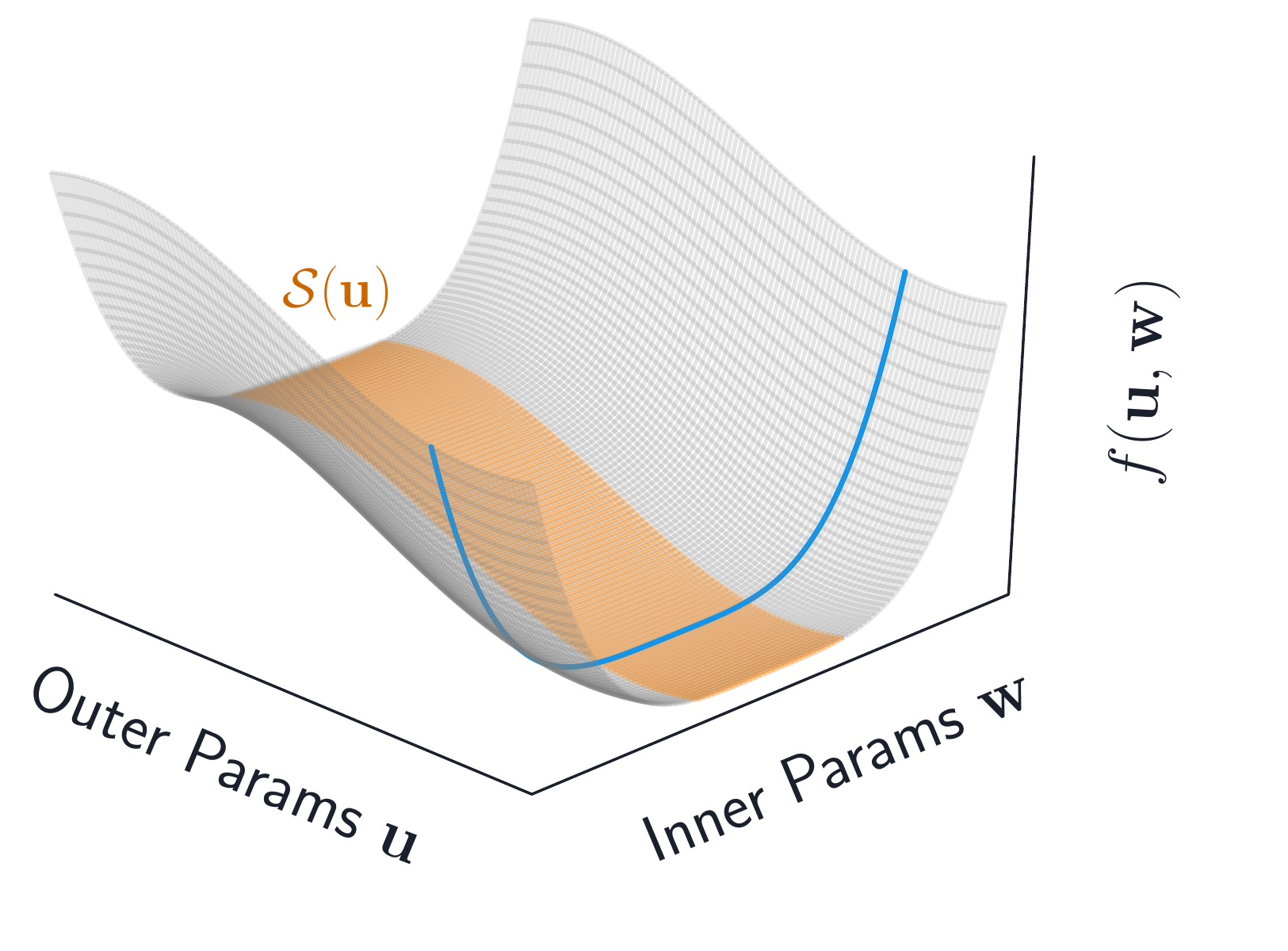}
        \caption{}
        \label{fig:diagrams-a}
    \end{subfigure}
    \hfill
    \begin{subfigure}{0.25\textwidth}
        \includegraphics[width=\linewidth]{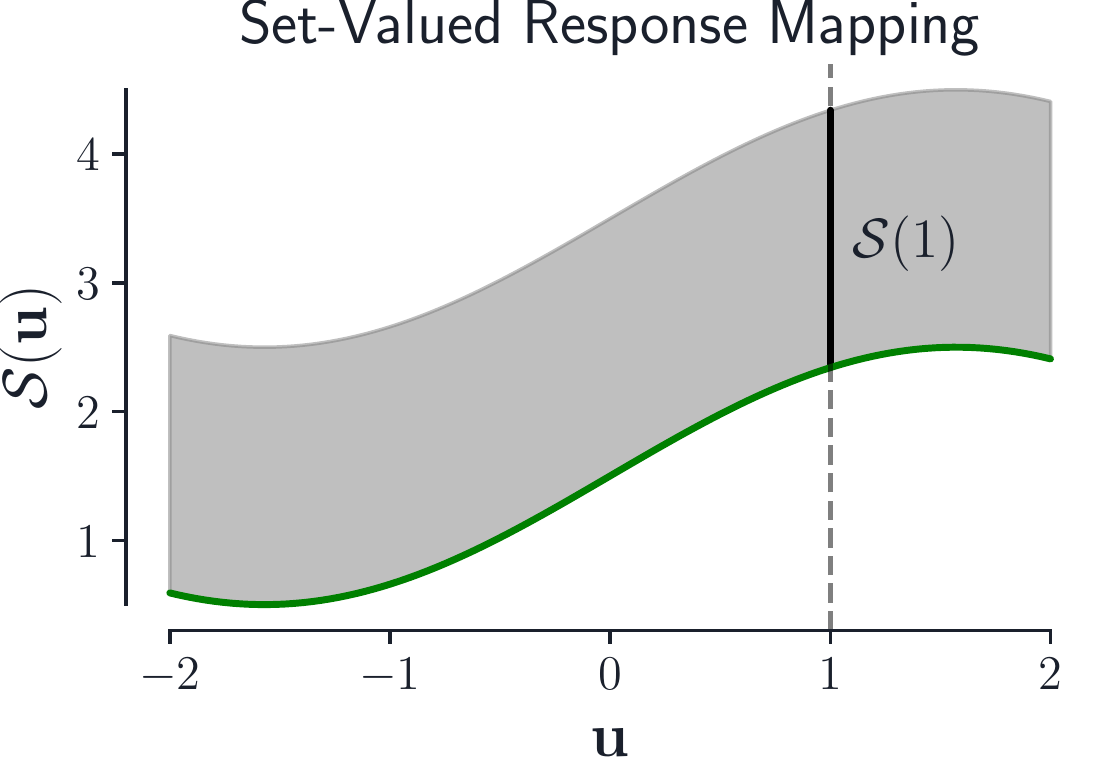}
        \caption{}
        \label{fig:diagrams-b}
    \end{subfigure}
    \hfill
    \begin{subfigure}{0.2\textwidth}
        \includegraphics[width=\linewidth]{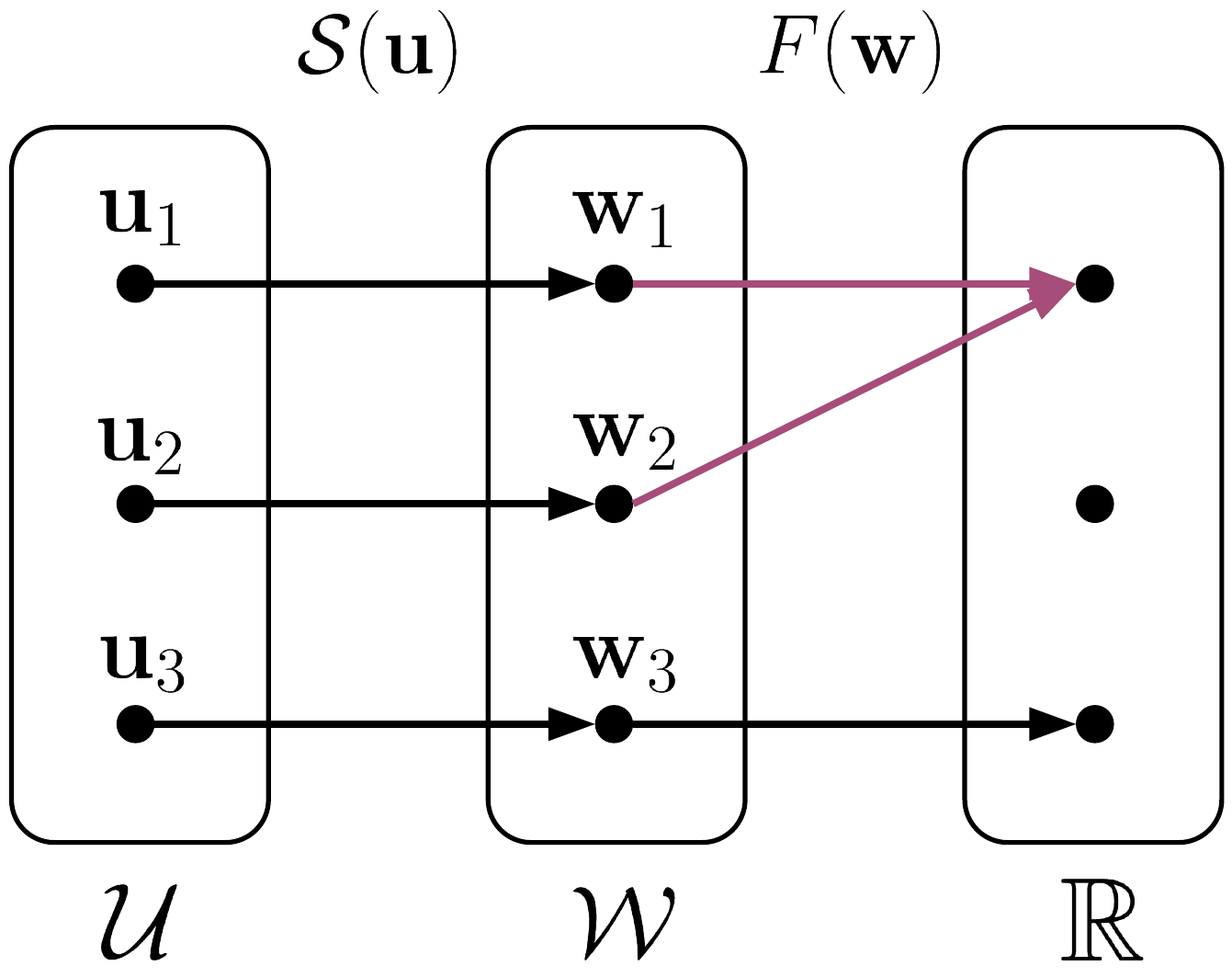}
        \caption{}
        \label{fig:diagrams-c}
    \end{subfigure}
    \hfill
    \begin{subfigure}{0.2\textwidth}
        \includegraphics[width=\linewidth]{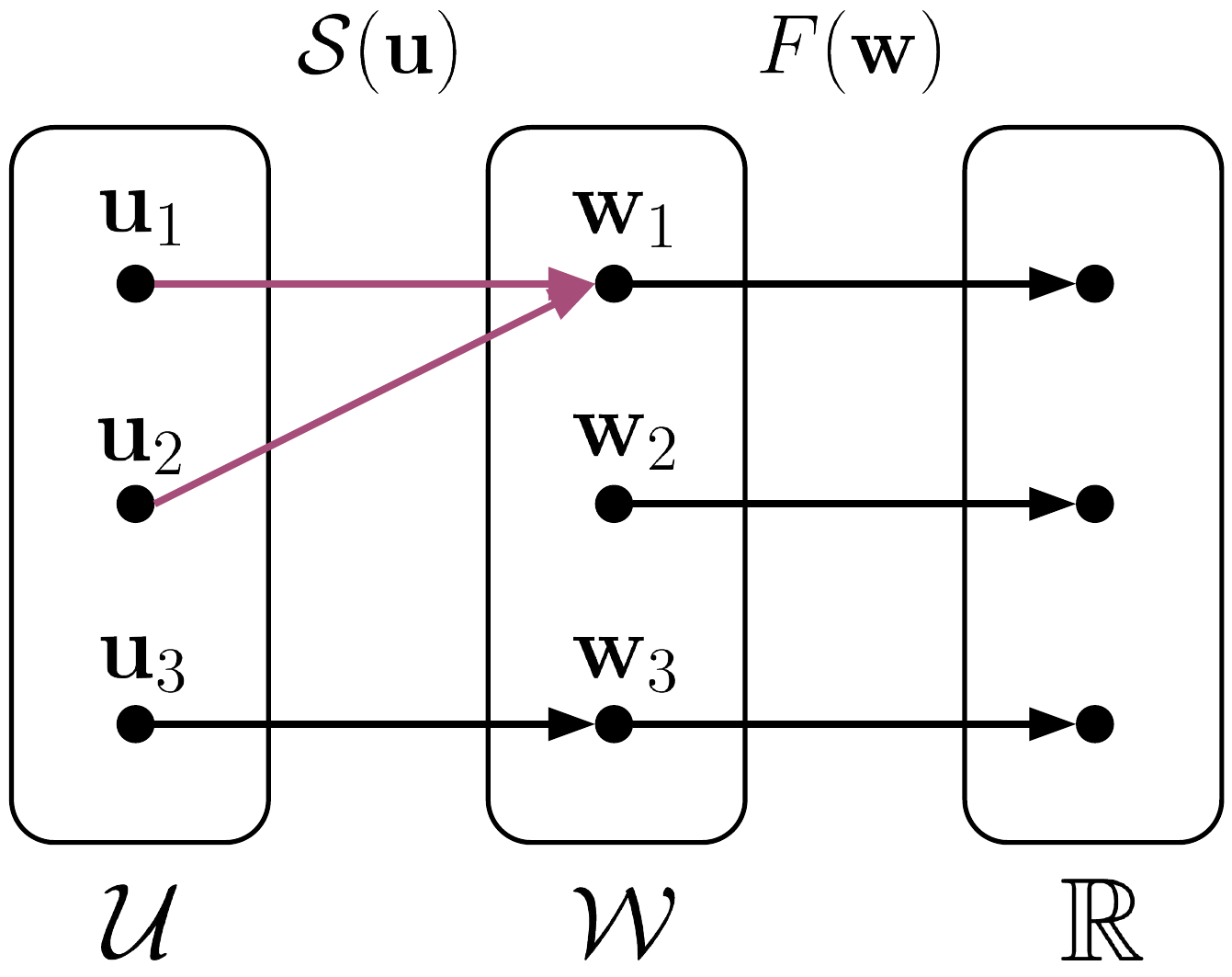}
        \caption{}
        \label{fig:diagrams-d}
    \end{subfigure}
    \caption{\small \textbf{Underspecification in BLO.} \textbf{(a)} Simplified visualization of inner underspecification, yielding a manifold of optimal solutions $\cS(\outParams) = \argmin_{\inParams} f(\outParams, \inParams)$ for each $\outParams$ (a slice is highlighted in {\color{cyan}blue} for a fixed $\outParams$).
    The {\color{orange}orange region} along the floor of the valley highlights the set-valued best-response mapping.
    \textbf{(b)} The set-valued best-response for $\outParams, \inParams \in \mathbb{R}$, where the shaded gray region contains values of $\inParams$ that achieve equivalent performance on the inner objective.
    The {\color{OliveGreen}green curve} highlights the minimum-norm inner solutions for each $\outParams$.
    \textbf{(c)} Outer underspecification due to $F$ mapping a range of inner parameters to the same value;
    or \textbf{(d)} from $\cS(\outParams)$ associating a range of outer parameters with the same inner parameters.}
    \label{fig:underspecification}
\end{figure*}

Most algorithmic contributions to BLO make the simplifying assumption that the inner problem has a unique solution (e.g., $|\cS(\outParams)| = 1, \forall \outParams \in \outSpace$), and give approximation methods that provably get close to the solution~\citep{shaban2019truncated,pedregosa2016hyperparameter,yang2021provably,hong2020two,grazzi2020iteration}.
In practice, however, these algorithms are often run in settings where either the inner or outer problem is \textit{underspecified}; i.e., the set of optima forms a non-trivial manifold.
Underspecification often occurs due to overparameterization, e.g., given a learning problem with more parameters than datapoints, there are typically infinitely many solutions that achieve the optimal objective value~\citep{belkin2021fit}.
Analyses of the dynamics of overparameterized single-objective optimization have shown that the nature of the converged solution, such as its pattern of generalization, depends in subtle ways on the details of the optimization dynamics~\citep{lee2019wide,arora2019fine,bartlett2020benign,amari2020does}; this general phenomenon is known as \emph{implicit bias}.
We extend this investigation to the bilevel setting: What are the implicit biases of practical BLO algorithms?
What types of solutions do they favor, and what are the implications for how the trained models behave?

\begin{table}[t]
\footnotesize
\centering
\setlength{\tabcolsep}{3pt}
\begin{tabular}{@{}ccccc@{}}
\toprule
\textbf{Task}        & \textbf{Inner ($\inParams$)} & \textbf{Outer ($\outParams$)} & \textbf{Inner Overparam} & \textbf{Outer Overparam} \\ \midrule
Dataset Distillation~\citep{wang2018dataset} & NN weights    & Synthetic data      & \cmark & \xmark \\
Data Augmentation~\citep{hataya2020meta}     & NN weights    & Augmentation params & \cmark & \xmark \\
GANs~\citep{goodfellow2014generative}        & Discriminator & Generator           & \cmark & \cmark \\
Meta-Learning~\citep{finn2017model}          & NN weights    & NN weights          & \cmark & \cmark \\
NAS~\citep{liu2018darts}                     & NN weights    & Architectures       & \cmark & \cmark \\
Hyperopt~\citep{maclaurin2015gradient}       & NN weights    & Hyperparams         & \cmark & \xmark \\
Example reweighting~\citep{ren2018learning}  & NN weights    & Example weights     & \cmark & \cmark \\
\bottomrule
\end{tabular}
\vspace{-0.2cm}
\caption{\small \textbf{Inner and outer overparameterization in common bilevel tasks.} For each task, we reference whether the common use-case includes inner overparametrization and/or outer overparametrization (InnerU/OuterU).}
\label{table:innerO}
\vspace{-0.2cm}
\end{table}

We identify two sources of implicit bias in underspecified BLO: 1) bias resulting from the optimization algorithm, either cold-start or warm-start BLO (described below); and 2) bias resulting from the hypergradient approximation.
Classic literature on bilevel optimization~\citep{dempe2002foundations} considers two ways to break ties between solutions in the set $\cS(\outParams)$: optimistic and pessimistic equilibria (discussed in Section~\ref{sec:background}).
However, as we will show, neither describes the solutions obtained by the most common BLO algorithms.
We focus on two algorithms that are relevant for practical machine learning, which we term \textit{cold-start} and \textit{warm-start} BLO.
Cold-start BLO~\citep{maclaurin2015gradient,metz2019understanding,micaelli2020non} defines the inner solution $\inParams^\star$ as the fixpoint of an update rule, which captures the notion of running an optimization algorithm to convergence.
For each outer update, the inner parameters are re-initialized to $\inParams_0$ and the inner optimization is run from scratch.
This is computationally expensive, as it requires full unrolls of the inner optimization for each outer step.
Warm-start BLO is a more tractable and widely-used alternative~\citep{luketina2016scalable,mackay2019self,lorraine2020optimizing,tang2020onlineaugment} that alternates gradient descent steps on the inner and outer parameters: in each step of warm-start BLO, the inner parameters are optimized from their \textit{current values}, rather than the initialization $\inParams_0$.

In Section~\ref{sec:theory}, we characterize solution concepts that capture the behavior of the cold- and warm-start BLO algorithms and show that, for quadratic inner and outer objectives, cold-start BLO yields minimum-norm outer parameters.
In Section~\ref{sec:experiments}, we show that warm-start BLO induces an implicit bias on the inner parameter iterates, regularizing the updates to maintain proximity to previous solutions.
In addition to the BLO algorithm, another source of implicit bias is the hypergradient approximation.
In Sections~\ref{sec:theory} and~\ref{sec:experiments}, we investigate the effect of using approximate implicit differentiation to compute the hypergradient $\frac{d \outObj}{d \outParams}$, where different approximations can lead to vastly different outer solutions.
In this paper, we consider both inner and outer underspecification.

\paragraph{Inner Underspecification.}
Many bilevel tasks in machine learning train a neural network in the inner level, which typically yields an underspecified problem, as shown in Table~\ref{table:innerO}.
Figure~\ref{fig:underspecification}(a,b) illustrates set-valued response mappings in underspecified BLO.

\paragraph{Outer Underspecification.}
Outer underspecification can arise in two ways: either 1) the mapping $\cS(\outParams)$ is a function (e.g., not set-valued) and maps a range of outer parameters to the same inner parameter; or 2) the outer objective $F$ maps a range of inner parameters to the same objective value.
These two pathways are illustrated in Figure~\ref{fig:underspecification}(c,d).

Our code is available on~\href{https://github.com/asteroidhouse/implicit-bias-bilevel}{Github}, and we provide~\href{https://colab.research.google.com/drive/1ryvCHItCSdkH44O9yX6I3Xw1lnJ2tTZ-?usp=sharing}{a Colab notebook here}.

\vspace{-3mm}
\section{Background}
\label{sec:background}

In this section, we provide an overview of several key concepts that we build on in this paper.\footnote{We provide a table of notation in Appendix~\ref{app:notation}.}
First, we review two methods for computing approximate hypergradients: differentiation through unrolling and implicit differentiation.
Then, we discuss optimistic and pessimistic BLO, which are two existing solution concepts for problems with non-unique inner solutions, that differ in how they break ties among equally-valid inner parameters.
Finally, we describe the task of dataset distillation, which we use for our empirical investigations in Section~\ref{sec:experiments}.

\paragraph{Hypergradient Approximation.}
We assume that $F$ and $f$ are differentiable, and consider gradient-based BLO algorithms, which require the hypergradient $\frac{d \outObj(\outParams, \inParams_{\outParams}^\star)}{d \outParams} = \frac{\partial \outObj(\outParams, \inParams_{\outParams}^\star)}{\partial \outParams} + \left(\frac{d \inParams_{\outParams}^\star}{d \outParams} \right)^\top \frac{\partial \outObj(\outParams, \inParams_{\outParams}^\star)}{\partial \inParams_{\outParams}^\star}$.
The main challenge to computing the hypergradient lies in computing the \textit{response Jacobian} $\frac{d \inParams_{\outParams}^\star}{d \outParams}$, which captures how the converged inner parameters depend on the outer parameters.
Exactly computing this Jacobian is often intractable for large-scale problems.
The two most common approaches to approximate it are: 1) \textit{iterative differentiation}, which unrolls the inner optimization to reach an approximate best-response (BR) and backpropagates through the unrolled computation graph to compute the approximate BR Jacobian~\citep{domke2012generic,maclaurin2015gradient,shaban2019truncated}; and 2) \textit{implicit differentiation}, which leverages the implicit function theorem (IFT) to compute the BR Jacobian, assuming that the inner parameters are at a stationary point of the inner objective~\citep{larsen1996design,chen1999optimal,bengio2000gradient,foo2008efficient,pedregosa2016hyperparameter,lorraine2020optimizing,hataya2020meta,raghu2020teaching}.
The IFT states that, assuming uniqueness of the inner solution and under certain regularity conditions, we can compute the response Jacobian as $\frac{\partial \inParams^\star_{\outParams}}{\partial \outParams} = \left( \frac{\partial^2 f}{\partial \inParams \partial \inParams^\top} \right)^{-1} \frac{\partial^2 f}{\partial \inParams \partial \outParams}$.
The term inside the inverse is the Hessian of the inner objective, which is typically intractable to store or invert directly (as neural nets can easily have millions of parameters).
Thus, multiplication by the inverse Hessian is typically approximated using an iterative linear solver, such as truncated conjugate gradient (CG)~\citep{pedregosa2016hyperparameter}, GMRES \citep{blondel2021efficient}, or the Neumann series~\citep{liao2018reviving,lorraine2020optimizing}.
The (un-truncated) Neumann series for an invertible Hessian $\frac{\partial^2 f}{\partial \inParams \partial \inParams^\top}$ is defined as:
\begin{equation}
\left( \frac{\partial^2 \inObj}{\partial \inParams \partial \inParams^\top} \right)^{-1}
=
\alpha \sum_{j=0}^\infty \left( \boldI - \alpha \frac{\partial^2 \inObj}{\partial \inParams \partial \inParams^\top} \right)^j \,.
\end{equation}
This series is known to converge when the largest eigenvalue of $\boldI -\alpha \frac{\partial^2 \inObj}{\partial \inParams \partial \inParams^\top}$ is $<$ 1.
In practice, the Neumann series is typically truncated to the first $K$ terms.
~\citet{lorraine2020optimizing} showed that differentiating through $K$ steps of unrolling starting from optimal inner parameters $\inParams^\star$ is equivalent to approximating the inverse Hessian with the first $K$ terms in the Neumann series, a result we review in Appendix~\ref{app:unrolling-and-neumann}.
Approximate implicit differentiation can be implemented with efficient Hessian-vector products using modern autodiff libraries~\citep{pearlmutter1994fast, tensorflow2015-whitepaper,paszke2017automatic,jax2018github}. 
However, when the inner problem is overparameterized, the Hessian $\frac{\partial^2 f}{\partial \inParams \partial \inParams^\top}$ is singular.
In Section~\ref{sec:theory}, we discuss how the $K$-term truncated Neumann series approximates the inverse of the \textit{damped Hessian} $(\boldH + \epsilon \boldI)^{-1}$ (which is always non-singular), and we discuss the implications of this approximation.

\paragraph{Response Jacobians for Overparameterized Problems.}
Let $\boldH = \frac{\partial^2 f(\outParams, \inParams)}{\partial \inParams \partial \inParams^\top}$, $\boldM = \frac{\partial^2 f(\outParams, \inParams)}{\partial \outParams \partial \inParams}$, and $\boldJ = \frac{\partial \inParams_{\outParams}}{\partial \outParams}$.
Then, for $\inParams_{\outParams} \in \argmin_{\inParams} f(\outParams, \inParams)$, we have $\boldH \boldJ = \boldM$.
If $\boldH$ is invertible, then the unique response Jacobian is $\boldJ = \boldH^{-1} \boldM$.
However, when the inner problem is overparameterized, $\boldH$ is singular, and the Jacobian is not uniquely defined---it may be any matrix $\boldJ$ satisfying $\boldH \boldJ = \boldM$. If $\boldJ$ is one such solution, then $\boldJ + \boldZ$ is also a solution, where each column of $\boldZ$ is in the nullspace of $\boldH$.
In this case, we define the canonical response Jacobian to be $\boldJ = \boldH^+ \boldM$, where $\boldH^+$ denotes the Moore-Penrose pseudoinverse of $\boldH$. This is the matrix with minimum Frobenius norm out of all solutions to the linear system (see Appendix~\ref{app:min-norm-jacobian}).
We refer to the hypergradient computed using this canonical response Jacobian as the ``exact'' hypergradient.

\begin{figure*}
    \centering
    \includegraphics[width=0.95\linewidth]{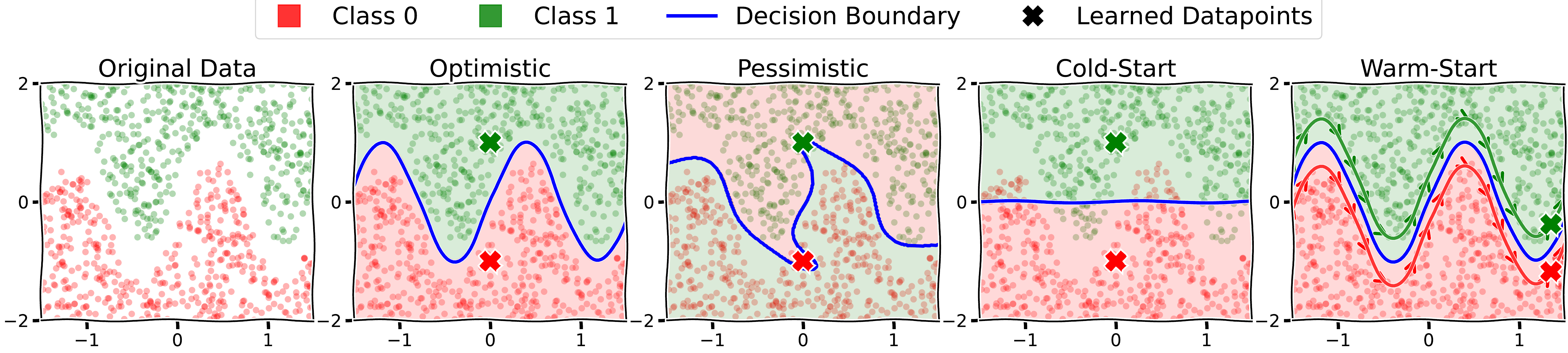}
    \caption{\small
    \textbf{Dataset distillation sketch illustrating four solution concepts for BLO: optimistic, pessimistic, {\color{cyan}cold-start}, and {\color{orange}warm-start}.}
    We distill the dataset consisting of {\color{red}red} and {\color{OliveGreen}green} points into two synthetic datapoints denoted by {\color{red} \ding{54}} and {\color{OliveGreen} \ding{54}} (one per class).
    The learned datapoints  {\color{red} \ding{54}} and {\color{OliveGreen} \ding{54}} are the \textit{outer parameters}, and the inner parameters correspond to a model (classifier) trained on these synthetic datapoints.
    We assume an overparameterized inner model, which can fit the synthetic datapoints  with many different decision boundaries.
    All of the solutions here correctly classify the synthetic datapoints (and are thus valid solutions to the inner problem) but they differ in their behavior on the original dataset.
    \textbf{Optimistic:} finds the decision boundary that correctly classifies all the original datapoints;
    \textbf{Pessimistic:} finds the decision boundary that achieves the worst loss on the original datapoints---it correctly classifies the synthetic datapoints but \textit{incorrectly classifies} the original datapoints (note the \textit{flipped red/green shading} on either side of the decision boundary);
    {\color{cyan}\textbf{Cold-start:}} finds the min-norm solution that correctly classifies the synthetic datapoints;
    {\color{orange}\textbf{Warm-start:}} yields a trajectory of synthetic datapoints over time, which can allow for a model trained on two learned datapoints to fit the original data.
    }
    \label{fig:optimistic-pessimistic}
\end{figure*}

\paragraph{Optimistic and Pessimistic BLO.}
For scenarios where the inner problem has multiple solutions, i.e., $|\cS(\outParams)| > 1$, two solution concepts have been widely studied in the classical BLO literature: the \textit{optimistic}~\citep{sinha2017review,dempe2007new,dempe2002foundations,harker1988existence,lignola1995topological,lignola2001existence,outrata1993necessary} and \textit{pessimistic}~\citep{sinha2017review,dempe2002foundations,dempe2014necessary,loridan1996weak,lucchetti1987existence,wiesemann2013pessimistic,liu2018pessimistic,liu2020methods} solutions.\footnote{See \citep{sinha2017review} for a comprehensive review.}
In \textit{optimistic} BLO, $\inParams$ is chosen such that it leads to the best outer objective value.
In contrast, in \textit{pessimistic} BLO, $\inParams$ is chosen such that it leads to the worst outer objective value.
The corresponding equilibria are defined below, with the differences highlighted in \textcolor{purple}{purple} and \textcolor{teal}{teal}.
\begin{tcolorbox}[colback=tabblue!15,boxrule=0pt,colframe=white,coltext=black,arc=2pt,outer arc=0pt,valign=center]
\begin{definition}
Let $\cS : \outSpace \rightrightarrows \inSpace$ be the set-valued response mapping $\cS(\outParams) = \argmin_{\inParams \in \inSpace} f(\outParams, \inParams)$.
Then $(\outParams^\star, \inParams^\star_{\outParams})$ is an \emph{\textcolor{purple}{optimistic}/\textcolor{teal}{pessimistic} equilibrium} if:
\begin{align*}
    \outParams^\star \in \argmin_{\outParams} F(\outParams, \inParams^\star_{\outParams})
    \qquad
    \text{s.t.}
    \qquad
    \inParams^\star_{\outParams} \in \rbarg \rbmin \!/\! \rbmax_{\!\!\!\!\!\!\!\!\!\!\!\!\!\inParams \in \cS(\outParams)} F(\outParams, \inParams)
\end{align*}
\end{definition}
\end{tcolorbox}

\paragraph{Dataset Distillation.}
Given an original (typically large) dataset, the task of \textit{dataset distillation}~\citep{maclaurin2015gradient,wang2018dataset,lorraine2020optimizing} is to learn a smaller \textit{synthetic dataset} such that a model trained on the synthetic data generalizes well to the original data.
In this problem, the inner parameters are the weights of a model, and the outer parameters are the synthetic datapoints.
The inner objective is the loss of the inner model trained on the synthetic datapoints, and the outer objective is the loss of the inner model on the original dataset:
\begin{tcolorbox}[colback=tabblue!15,boxrule=0pt,colframe=white,coltext=black,arc=2pt,outer arc=0pt,valign=center]
\vspace{-0.4cm}
\begin{align}
\mathbf{u}^\star \in \text{``}\argmin_{\outParams}\text{''} \mathcal{L}(\mathcal{D}_{\text{original}}, \inParams^\star_{\outParams})
\qquad
\text{s.t.}
\qquad
\inParams^\star_{\outParams} \in \argmin_{\inParams} \mathcal{L}(\mathcal{D}(\outParams), \inParams)
\end{align}
Here, $\mathcal{L}$ denotes a loss function, $\mathcal{D}_{\text{original}}$ denotes the dataset we aim to distill, and $\mathcal{D}(\mathbf{u})$ denotes a synthetic dataset parameterized by $\mathbf{u}$.
\end{tcolorbox}
We focus on dataset distillation throughout this paper because the degree of inner and outer overparameterization can be modified by varying the number of original versus synthetic datapoints.
The solutions for this task are also easy to visualize and interpret.

\paragraph{Visualizing Solution Concepts.}
Figure~\ref{fig:optimistic-pessimistic} illustrates four different solution concepts for a toy 2D problem: the optimistic and pessimistic solutions, as well as two new solution concepts---cold- and warm-start equilibria---that we discuss in Section~\ref{sec:theory}.
We consider dataset distillation for binary classification, with red and green datapoints representing the two classes, and a sinusoidal ground-truth decision boundary.
We distill this dataset into two synthetic datapoints, represented by {\color{red} \ding{54}} and {\color{OliveGreen} \ding{54}}.
In each subplot, the blue curve shows the decision boundary of the inner classifier, and the background colors show which class is predicted on each side of the decision boundary.
See the caption of Figure~\ref{fig:optimistic-pessimistic} for a description of each solution concept on this task.

The optimistic solution can be tractable to compute~\citep{mehra2019penalty,ji2021lower}, while the pessimistic solution is known to be less tractable~\citep{sinha2017review}.
From Figure~\ref{fig:optimistic-pessimistic}, however, it is unclear whether either of these solution concepts is useful.
For example, in hyperparameter optimization, the optimistic and pessimistic solutions correspond to choosing hyperparameters such that the inner training loop achieves minimum or maximum performance on the validation set, and it is unclear why this would be beneficial.
Most widely-used BLO algorithms do not aim to find either of these solution concepts, which motivates our study of the behavior of the algorithms used in practice.

\subsection{Algorithms}
\label{app:algorithms}

Here, we provide the algorithms for cold-start (Algorithm~\ref{alg:cold-start}) and warm-start (Algorithm~\ref{alg:warm-start}) bilevel optimization.
Each algorithm can make use of different hypergradient approximations, including implicit differentiation with the truncated Neumann series (Algorithm~\ref{alg:neumann-hypergrad}), implicit differentiation with the damped Hessian inverse (Algorithm~\ref{alg:damped-hypergrad}), and differentiation through unrolling (Algorithm~\ref{alg:iterative-diff}).
\begin{figure*}[!htbp]
\begin{minipage}[t]{0.49\linewidth}
\begin{algorithm}[H]
  \caption{{\color{cyan}Cold-Start BLO}}
  \label{alg:cold-start}
\begin{algorithmic}[1]
    \State \textbf{Input:} $\alpha$: inner LR, $\beta$: outer LR
    \State \hspace{3.2em} $T$: number of inner optimization steps
    \State \hspace{3.2em} $K$: number of Neumann series terms
    \While{$||\hat{\nabla}_{\outParams} F(\outParams, \inParams)||^2 > 0$, iteration $t$}
        \State $\hat{\nabla}_{\outParams} F \gets \text{HypergradApprox}(\outParams_t, {\color{cyan}\inParams_0}, T, K)$
        \State $\outParams_{t+1} \gets \outParams_t - \beta \hat{\nabla}_{\outParams} F$
        \State 
    \EndWhile
    \State \Return $(\outParams_t, {\color{cyan}\text{Optimize}(\outParams_t, \inParams_0, T)})$
\end{algorithmic}
\end{algorithm}
\end{minipage}
\hfill
\begin{minipage}[t]{0.49\linewidth}
\begin{algorithm}[H]
  \caption{{\color{orange}Warm-start BLO}}
  \label{alg:warm-start}
\begin{algorithmic}[1]
    \State \textbf{Input:} $\alpha$: inner LR, $\beta$: outer LR
    \State \hspace{3.2em} $T$: number of inner optimization steps
    \State \hspace{3.1em} $K$: number of Neumann series terms
    \While{$||\hat{\nabla}_{\outParams} F(\outParams, \inParams)||^2 > 0$, iteration $t$}
        \State $\hat{\nabla}_{\outParams} F \gets \text{HypergradApprox}(\outParams_t, {\color{orange}\inParams_t}, T, K)$
        \State $\outParams_{t+1} \gets \outParams_t - \beta \hat{\nabla}_{\outParams} F$
        \State {\color{orange}$\inParams_{t+1} \gets \text{Optimize}(\outParams_{t+1}, \inParams_t, T)$}
    \EndWhile
    \State \Return $(\outParams_t, {\color{orange}\inParams_t})$
\end{algorithmic}
\end{algorithm}
\end{minipage}
\vspace{-0.2cm}
\caption{\small \textbf{Algorithms for cold-start and warm-start bilevel optimization.} We highlight the key differences in {\color{cyan}cyan} for cold-start and {\color{orange}orange} for warm-start. $\text{HypergradApprox}$ is a placeholder for one of three different algorithms: 1) $\text{NeumannHypergrad}$ (Algorithm~\ref{alg:neumann-hypergrad}), which computes the hypergradient using the IFT with the Neumann series approximation of the inverse Hessian; 2) $\text{DampedHypergrad}$ (Algorithm~\ref{alg:damped-hypergrad}), which similarly uses the IFT with the damped Hessian inverse $(\boldH + \epsilon \boldI)^{-1}$; or 3) $\text{IterDiffHypergrad}$ (Algorithm~\ref{alg:iterative-diff}), which computes the hypergradient by differentiating through unrolled inner optimization.}
\vspace{-0.4cm}
\end{figure*}
\vspace{-0.2cm}
\begin{figure*}[!htbp]
\begin{minipage}[t]{0.49\linewidth}
\begin{algorithm}[H]
  \caption{$\textnormal{Optimize}(\outParams, \inParams, T)$}
  \label{alg:optimize}
\begin{algorithmic}[1]
    \State \textbf{Input:} $\gamma$: learning rate
    \State \hspace{3.2em} $T$: unroll steps
    \State \hspace{3.2em} $\inParams$: initialization
    \For{$t = 1, \dots, T$}
        \State $\inParams \gets \inParams - \gamma \nabla_{\inParams} f(\outParams, \inParams)$
    \EndFor
    \State \Return $\inParams$
\end{algorithmic}
\end{algorithm}
\end{minipage}
\begin{minipage}[t]{0.49\linewidth}
\begin{algorithm}[H]
\caption{$\textnormal{IterDiffHypergrad}$}
\label{alg:iterative-diff}
\begin{algorithmic}[1]
    \State \textbf{Input:} $\alpha$: inner LR
    \State \hspace{3.2em} $\beta$: outer LR
    \State \hspace{3.2em} $T$: inner steps
    \State $\hat{\inParams}_T^* \gets \text{Optimize}(\outParams_t, \inParams_t, T)$
    \State $\hat{F} \gets F(\outParams, \hat{\inParams}^*_T)$
    \State Compute $\nabla_{\outParams} \hat{F}$ using auto-diff.
\end{algorithmic}
\end{algorithm}
\end{minipage}
\vspace{-0.2cm}
\caption{\small Helper functions to compute the implicit hypergradient using the Neumann series to approximate the inverse inner loss Hessian.}
\vspace{-0.4cm}
\end{figure*}
\vspace{-0.4cm}
\begin{figure*}[!htbp]
\begin{minipage}[t]{0.49\linewidth}
\begin{algorithm}[H]
\caption{$\textnormal{DampedHypergrad}$}
\label{alg:damped-hypergrad}
\begin{algorithmic}[1]
    \State \textbf{Input:} $\alpha$: inner LR
    \State \hspace{3.2em} $\beta$: outer LR
    \State \hspace{3.2em} $T$: inner steps
    \State \hspace{3.4em} $\epsilon$: damping coeff.
    \State $\hat{\inParams}_T^* \gets \text{Optimize}(\outParams_t, \inParams_t, T)$
    \State $\boldA \gets \left( \frac{\partial^2 f(\outParams_t, \hat{\inParams}_T^*)}{\partial \inParams \partial \inParams^\top} + \epsilon \boldI \right)^{-1}$
    \State $\boldM \gets \frac{\partial^2 f(\outParams_t, \hat{\inParams}_T^*)}{\partial \outParams \partial \inParams}$
    \State $\frac{d \hat{\inParams}_T^*}{d \outParams} \gets \boldA \boldM$
    \State $\hat{\nabla}_{\outParams} F \gets \frac{\partial F}{\partial \outParams} + \left( \frac{d \hat{\inParams}_T^*}{d \outParams} \right)^\top \frac{\partial F(\outParams, \hat{\inParams}_T^*)}{\partial \hat{\inParams}_T^*}$
    \State \Return $\hat{\nabla}_{\outParams} F$
\end{algorithmic}
\end{algorithm}
\end{minipage}
\begin{minipage}[t]{0.49\linewidth}
\begin{algorithm}[H]
\caption{$\textnormal{NeumannHypergrad}$}
\label{alg:neumann-hypergrad}
\begin{algorithmic}[1]
    \State \textbf{Input:} $\alpha$: inner LR
    \State \hspace{3.2em} $\beta$: outer LR
    \State \hspace{3.2em} $T$: inner steps
    \State \hspace{3.2em} $K$: Neumann terms
    \State $\hat{\inParams}_T^* \gets \text{Optimize}(\outParams_t, \inParams_t, T)$
    \State $\frac{d \hat{\inParams}_T^*}{d \outParams} \gets \alpha \sum_{j=0}^K \left(\boldI - \alpha \frac{\partial^2 f(\outParams, \inParams)}{\partial \inParams \partial \inParams^\top} \right)^j$
    \State $\hat{\nabla}_{\outParams} F \gets \frac{\partial F}{\partial \outParams} + \left( \frac{d \hat{\inParams}_T^*}{d \outParams} \right)^\top \frac{\partial F(\outParams, \hat{\inParams}_T^*)}{\partial \hat{\inParams}_T^*}$
    \State \Return $\hat{\nabla}_{\outParams} F$
\end{algorithmic}
\end{algorithm}
\end{minipage}
\vspace{-0.2cm}
\caption{\small $\text{DampedHypergrad}$ uses the IFT with the damped Hessian inverse $(\boldH + \epsilon \boldI)^{-1}$
and $\text{IterDiffHypergrad}$ unrolls the inner optimization and backpropagates through the unroll.
}
\end{figure*}

\clearpage

\section{Equilibrium Concepts}
\label{sec:theory}

We aim to understand the behavior of two popular BLO algorithms: \textit{cold-start} (Algorithm~\ref{alg:cold-start}) and \textit{warm-start} (Algorithm~\ref{alg:warm-start}).
A summary of the updates for each algorithm is shown in Table~\ref{table:warm-cold-start}.
In this section, we first define equilibrium notions that capture the solutions found by warm-start and cold-start BLO.
We then discuss properties of these equilibria, in particular focusing on the norms of the resulting inner and outer parameters.
Finally, we analyze the implicit bias induced by approximating the hypergradient with the truncated Neumann series.

\subsection{Cold-Start Equilibrium}

Here, we introduce a solution concept that captures the behavior of cold-start BLO.
We consider an iterative optimization algorithm used to compute an approximate solution to the inner objective.
We denote a step of inner optimization by $\inParams_{t+1} = \Xi(\outParams, \inParams_t)$; for gradient descent, we have $\Xi(\outParams, \inParams_t) = \inParams_t - \alpha \nabla_{\inParams} f(\outParams, \inParams_t)$.
We denote $K$ steps of inner optimization from $\inParams_0$ by $\Xi^{(K)}(\outParams, \inParams_0)$.
Under certain assumptions (e.g., that $f$ has a unique finite root and that the step size $\alpha$ for the update is chosen appropriately), repeated application of $\Xi$ will converge to a fixpoint.
We denote the fixpoint for an initialization $\inParams_0$ by $\Xi^{(\infty)}(\outParams, \inParams_0)$.
\begin{tcolorbox}[colback=tabblue!15,boxrule=0pt,colframe=white,coltext=black,arc=2pt,outer arc=0pt,valign=center]
\begin{definition}
Let $\boldr(\outParams, \inParams) \triangleq \Xi^{(\infty)}(\outParams, \inParams)$.
Then $(\outParams^\star, \inParams^\star)$ is a \textit{{{\color{cyan}cold-start equilibrium}}} for an initialization $\inParams_0$ if:
\begin{align*}
    \outParams^\star \in \argmin_{\outParams} F(\outParams, \inParams^\star) \quad \text{s.t.} \quad
    \inParams^\star = \boldr(\outParams^\star, \inParams_0)
\end{align*}
\end{definition}
\end{tcolorbox}
In some cases, we can compute the fixpoint of the inner optimization analytically.
In particular, when $f$ is quadratic, the analytic solution minimizes the displacement from $\inParams_0$: $\Xi^{(\infty)}(\outParams, \inParams_0) = \argmin_{\inParams \in \cS(\outParams)} \| \inParams - \inParams_0 \|_2^2$.

\subsection{Warm-Start Equilibrium}
Next, we introduce a solution concept intended to capture the behavior of warm-start BLO.
Warm-starting refers to initializing the inner optimization from the inner parameters obtained in the previous hypergradient computation.
One can consider two variants of warm-starting: (1) using \textit{full inner optimization}, that is, running the inner optimization to convergence starting from $\inParams_t$ to obtain the next iterate $\inParams_{t+1}$; or (2) \textit{partial inner optimization}, where we approximate the solution to the inner problem via a few gradient steps.
\begin{table}
\setlength{\tabcolsep}{4pt}
\centering
\begin{tabular}{@{}cc@{}}
\toprule
\multicolumn{1}{c}{\textbf{Method}}      & \multicolumn{1}{c}{\textbf{Inner Update}}  \\ \midrule
\textbf{{\color{cyan}Full Cold-Start}}
& $\inParams^\star_{k+1} = \Xi^{(\infty)}(\outParams_{k+1}, \inParams_0)$  \\[6pt]
\textbf{{\color{red}Full Warm-Start}}
&
$\inParams^\star_{k+1} = \Xi^{(\infty)}(\outParams_{k+1}, \inParams_k^\star)$ \\[6pt]
\textbf{{\color{orange}Partial Warm-Start}}
&
$\inParams^\star_{k+1} = \Xi^{(T)}(\outParams_{k+1}, \inParams_k^\star)$  \\[6pt]
\bottomrule
\end{tabular}
\caption{\small
Inner parameter updates for {\color{cyan}cold-start}, {\color{red}full warm-start}, and {\color{orange}partial warm-start} BLO.
}
\label{table:warm-cold-start}
\end{table}
Full warm-start can be expressed as computing $\inParams_{k+1}^\star = \Xi^{(\infty)}(\outParams, \inParams_k^\star)$ in each iteration, starting from the previous inner solution $\inParams_k^\star$ rather than $\inParams_0$.
Similarly, partial warm-start can be expressed as $\inParams_{k+1}^\star = \Xi^{(T)}(\outParams, \inParams_k^\star)$, where $T$ is often small (e.g., $T < 10$).

\begin{tcolorbox}[colback=tabblue!15,boxrule=0pt,colframe=white,coltext=black,arc=2pt,outer arc=0pt,valign=center]
\begin{definition}
    Let $\boldr(\outParams, \inParams) \triangleq \Xi^{(T)}(\outParams, \inParams)$.
    Then $(\outParams^\star, \inParams^\star)$ is a \emph{{{\color{orange}({\color{red}full} or partial) warm-start equilibrium}}} if:
        \begin{align*}
            \smash{\outParams^\star \in \argmin_{\outParams} F(\outParams, \inParams^\star) \quad \text{s.t.} \quad
            \inParams^\star = \boldr(\outParams^\star, \inParams^\star)}
        \end{align*}
    For finite $T$, the solution is a {\color{orange}\textit{partial warm-start equilibrium}}.
    As $T \to \infty$, we obtain the {\color{red}\textit{full warm-start equilibrium}}.
\end{definition}
\end{tcolorbox}
\vspace{-0.2cm}

\subsection{Solution Properties}
In this section, we examine properties of cold-start and warm-start equilibria---we analyze the norms of the inner and outer parameters given by different methods.
First, we show that, when $F$ and $f$ are quadratic, cold-start equilibria yield minimum-norm outer parameters.
Then, we analyze the implicit bias induced by the hypergradient approximation, in particular by using the truncated Neumann series to estimate the inverse Hessian for implicit differentiation.
To make the analysis tractable, we make the following assumption:
\begin{tcolorbox}[colback=tabblue!15,boxrule=0pt,colframe=white,coltext=black,arc=2pt,outer arc=0pt,valign=center]
\begin{assumption}[Quadratic Objectives]
\label{assumption:quadratic}
We assume that both the inner and outer objectives, $f$ and $F$, are convex, lower-bounded quadratics.
We let:
\begin{align*}
    \inObj(\outParams, \inParams)
    =
    \frac{1}{2}
    \begin{bmatrix}
      \inParams^\top & \outParams^\top
    \end{bmatrix}
    \begin{bmatrix}
      \boldA & \boldB \\
      \boldB^\top & \boldC
    \end{bmatrix}
    \begin{bmatrix}
      \inParams \\
      \outParams
    \end{bmatrix} +
    \boldd^\top \inParams + \bolde^\top \outParams + c
\end{align*}
where $\boldA \in \mathbb{R}^{|\inSpace| \times |\inSpace|}$ is positive semi-definite, $\boldB \in \mathbb{R}^{|\inSpace| \times |\outSpace|}$, $\boldC \in \mathbb{R}^{|\outSpace| \times |\outSpace|}$, $\boldd \in \mathbb{R}^{|\inSpace|}$, $\bolde \in \mathbb{R}^{|\outSpace|}$, and $c \in \mathbb{R}$.
The PSD assumption implies that $f$ is a convex quadratic in $\inParams$ for a fixed $\outParams$.
For the outer objective, we restrict our analysis to the case where $\outObj$ only depends directly on $\boldw$ (which encompasses many common applications including hyperparameter optimization):
\begin{align*}
    \outObj(\outParams, \inParams)
    &=
    \frac{1}{2} \boldw^\top \boldP \boldw + \boldf^\top \boldw + h\,,
\end{align*}
where $\boldP \in \mathbb{R}^{|\inSpace| \times |\inSpace|}$ is positive semi-definite.
\end{assumption}
\end{tcolorbox}
\vspace{-0.2cm}

We assume that the inner objective is a convex (but not strongly-convex) quadratic, which may have \textit{many global minima}, providing a tractable setting to analyze implicit bias.
This allows us to model an overparameterized linear regression problem with data matrix $\boldsymbol{\Phi} \in \mathbb{R}^{n \times p}$, which corresponds to a quadratic with curvature $\boldsymbol{\Phi}^\top \boldsymbol{\Phi}$.
When the number of parameters is greater than the number of training examples ($p > n$), $\boldsymbol{\Phi}^\top \boldsymbol{\Phi}$ is PSD.
In this case, $f$ has a manifold of valid solutions~\citep{wu2020optimal}.

\begin{tcolorbox}[colback=tabblue!15,boxrule=0pt,colframe=white,coltext=black,arc=2pt,outer arc=0pt,valign=center]
\begin{statement}[{{\color{cyan}Cold-start BLO converges to a cold-start equilibrium.}}]
\label{thm:cold-start-equilibrium}
    Suppose $f$ and $F$ satisfy Assumption~\ref{assumption:quadratic}.
    If the inner parameters are initialized to $\inParams_0$ and learning rates are set appropriately to guarantee convergence, then the cold-start algorithm (Algorithm~\ref{alg:cold-start}) using exact hypergradients converges to a \textit{cold-start equilibrium}.
\end{statement}
\begin{proof}
    The proof is provided in Appendix~\ref{app:cold-start-equilibrium-proof}.
\end{proof}
\end{tcolorbox}
\vspace{-0.2cm}

\paragraph{Implicit Bias of Cold-Start for Inner Solutions.}
Because cold-start BLO optimizes the inner parameters \textit{from initialization to convergence} for each outer iteration, the inner solution inherits properties from the single-level optimization literature.
For example, in linear regression trained with gradient descent, the inner solution will minimize displacement from the initialization $\norm{\boldw - \boldw_0}_2^2$.
Recent work aims to generalize such statements to other model classes and optimization algorithms---e.g., obtaining min-norm solutions with algorithm-specific norms~\citep{gunasekar2018characterizing}.
Generalizing the results for various types of neural nets is an active research area~\citep{vardi2021implicit}.

\paragraph{Implicit Bias of Cold-Start for Outer Solutions.}
In the following theorem, we show that cold-start BLO from outer initialization $\outParams_0$, converges to the outer solution $\outParams^\star$ that minimizes displacement from $\outParams_0$.

\begin{tcolorbox}[colback=tabblue!15,boxrule=0pt,colframe=white,coltext=black,arc=2pt,outer arc=0pt,valign=center]
\begin{theorem}[Min-Norm Outer Parameters]
\label{thm:min-norm-outer}
Consider cold-start BLO (Algorithm~\ref{alg:cold-start}) with exact hypergradients starting from outer initialization $\outParams_0$.
Assume that for each outer iteration, the \textit{inner parameters} are re-initialized to $\inParams_0 = \boldzero$ and optimized with an appropriate learning rate that ensures convergence.
Then cold-start BLO converges to an outer solution $\outParams^\star$ with minimum $L_2$ distance from $\outParams_0$:
\begin{equation}
\outParams^\star = \argmin_{\outParams \in \argmin_{\outParams} F^\star(\outParams)}  \norm{\outParams - \outParams_0}^2\,.
\end{equation}
where we define $F^\star(\outParams) \triangleq F(\outParams, \inParams^\star_{\outParams})$, where $\inParams^\star_{\outParams}$ is the minimum-displacement inner solution from $\inParams_0$, that is, $\inParams^\star_{\outParams} = \argmin_{\inParams \in \cS(\outParams)} \| \inParams - \inParams_0 \|^2$.
\end{theorem}
\begin{proof}
    The proof is provided in Appendix~\ref{app:proof-min-norm-outer}.
\end{proof}
\end{tcolorbox}

Next, we observe that the full warm-start and cold-start algorithms are equivalent for strongly-convex inner problems.

\begin{tcolorbox}[colback=tabblue!15,boxrule=0pt,colframe=white,coltext=black,arc=2pt,outer arc=0pt,valign=center]
\begin{remark}[Equivalence of Full Warm-Start and Cold- Start in the Strongly Convex Regime] \label{remark:cold-full-warm}
    When the inner problem $f(\outParams, \inParams)$ is strongly convex in $\inParams$ for each $\outParams$, then {\color{red} full warm-start} (Algorithm~\ref{alg:warm-start}) and {\color{cyan} cold-start} (Algorithm~\ref{alg:cold-start}) are equivalent.
\end{remark}
\begin{proof}
    The proof is provided in Appendix~\ref{app:cold-full-warm-equiv-proof}
\end{proof}
\end{tcolorbox}

We relate partial and full warm-start equilibria as follows:

\begin{tcolorbox}[colback=tabblue!15,boxrule=0pt,colframe=white,coltext=black,arc=2pt,outer arc=0pt,valign=center]
\begin{statement}[Inclusion of Partial Warm-Start Equilibria.] \label{thm:inclusion}
Every partial warm-start equilibrium is a full warm-start equilibrium.
In addition, if $\Xi(\outParams, \inParams) = \inParams - \alpha \nabla_{\inParams} f(\outParams, \inParams)$ with a fixed (non-decayed) step size $\alpha$, then full warm-start equilibria are also partial warm-start equilibria.
\end{statement}
\begin{proof}
    The proof is provided in Appendix~\ref{app:inclusion-proof}.
\end{proof}
\end{tcolorbox}

\subsubsection{Implicit Bias from Hypergradient Approximation}

\paragraph{Neumann Series.}
\begin{wrapfigure}[11]{r}{0.4\linewidth}
  \vspace{-0.4cm}
  \begin{center}
    \includegraphics[width=\linewidth]{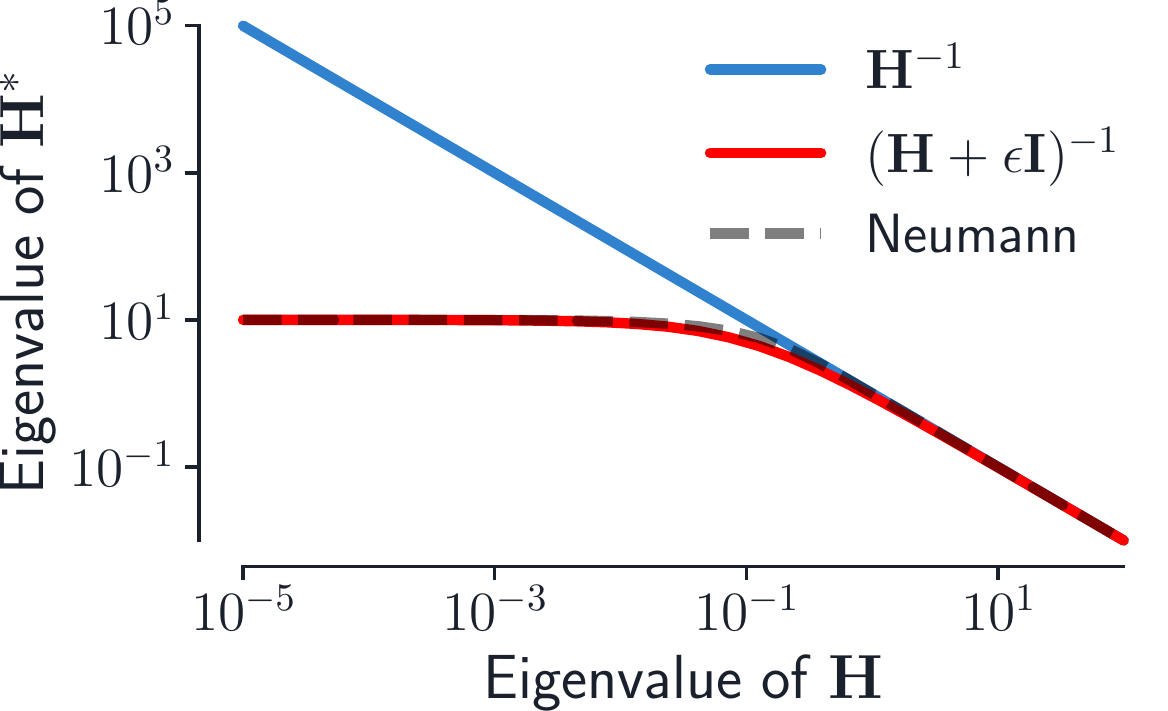}
  \end{center}
  \caption{\small
  Eigenvalues of various matrix functions of $\boldH$ (denoted $\boldH^\star$ in the diagram): $\boldH^{-1}$, $(\boldH + \epsilon \boldI)^{-1}$, and the Neumann series approximation $\alpha \sum_{j=0}^K (\boldI - \alpha \boldH)^j$.
  Here, $\epsilon = \frac{1}{\alpha K}$.
  }
  \label{fig:sv-comparison}
\end{wrapfigure}
Next, we investigate the impact of the hypergradient approximation on the converged outer solution.
In practice, one often uses the truncated Neumann series to estimate the inverse Hessian when computing the implicit hypergradient.
Note that:
\begin{tcolorbox}[colback=tabblue!15,boxrule=0pt,colframe=white,coltext=black,arc=2pt,outer arc=0pt,valign=center]
\vspace{-0.4cm}
\begin{align}
    \alpha \sum_{j=0}^K (\boldI - \alpha \boldH)^j \approx (\boldH + \epsilon \boldI)^{-1}
\end{align}
\end{tcolorbox}
where $\epsilon = \frac{1}{\alpha K}$.
This is known from the spectral regularization literature~\citep{gerfo2008spectral,RosascoLecture}.
We show in Appendix~\ref{app:proximal-br} that the damped Hessian inverse $(\boldH + \epsilon \boldI)^{-1}$ corresponds to the hypergradient of a proximally-regularized inner objective: $\hat{f}(\outParams, \inParams) = f(\outParams, \inParams) + \frac{\epsilon}{2} \| \inParams - \inParams' \|_2^2$, where $\inParams' \in \argmin_{\inParams} f(\outParams, \inParams)$.
The damping prevents the inner optimization from moving far in low-curvature directions.
Consider the spectral decomposition of the real symmetric matrix $\boldH = \boldU \boldD \boldU^\top$, where $\boldD$ is a diagonal matrix containing the eigenvalues of $\boldH$, and $\boldU$ is an orthogonal matrix.
Then,
$
    (\boldH + \epsilon \boldI)^{-1}
    =
    (\boldU \boldD \boldU^\top + \epsilon \boldI)^{-1}
    =
    (\boldU (\boldD + \epsilon \boldI) \boldU^\top)^{-1}
    =
    \boldU^\top (\boldD + \epsilon \boldI)^{-1} \boldU
$.
If $\lambda$ is an eigenvalue of $\boldH$, then the corresponding eigenvalue of $\boldH^{-1}$ is $\frac{1}{\lambda}$.
In contrast, the corresponding eigenvalue of $(\boldH + \epsilon \boldI)^{-1}$ is $\frac{1}{\lambda + \epsilon}$.
When $\epsilon \ll \lambda$, $\frac{1}{\lambda + \epsilon} \approx \frac{1}{\lambda}$; when $\lambda \ll \epsilon$, $\frac{1}{\lambda + \epsilon} \approx \frac{1}{\epsilon}$.
Thus, the influence of small eigenvalues (associated with low-curvature directions) is diminished, and the truncated Neumann series primarily takes into account high-curvature directions.
Figure~\ref{fig:sv-comparison} illustrates the relationship between the eigenvalues of $\boldH^{-1}$, $(\boldH + \epsilon \boldI)^{-1}$, and $\alpha \sum_{j=0}^K (\boldI - \alpha \boldH)^j$.

\paragraph{Unrolling.}
The implicit bias induced by approximating the hypergradient via $K$-step unrolled differentiation is qualitatively similar to the bias induced by the truncated Neumann series.
For quadratic inner objectives $f$, the truncated Neumann series coincides with the result of differentiating through $K$ steps of unrolled gradient descent, because the Hessian of a quadratic is constant over the inner optimization trajectory~\citep{lorraine2020optimizing}.
Note that gradient descent on a quadratic converges (for step-size $\alpha \leq 1 / \|\boldH\|_2$) more rapidly in high-curvature directions than in low-curvature directions.
Thus, the approximate best response obtained by truncated unrolling only takes into account high-curvature directions of the inner objective, and is less sensitive to low-curvature directions.
In turn, the response Jacobian will only capture how the inner parameters depend on the outer parameters in these high-curvature directions.
Note that for general objectives $f$ (e.g., when training neural networks), differentiating through unrolling only coincides with the Neumann series when the inner parameters are at a stationary point of $f$.

\section{Empirical Overparameterization Results}
\label{sec:experiments}

Many large-scale empirical studies---in particular in the areas of hyperparameter optimization and data augmentation---use warm-start bilevel optimization~\citep{hataya2020meta,ho2019population,mounsaveng2021learning,peng2018jointly,tang2020onlineaugment}.
In this section, we introduce simple tasks based on dataset distillation, designed to provide insights into the phenomena at play.
First, we show that when the inner problem is overparameterized, the inner parameters $\inParams$ can retain information associated with different settings of the outer parameters over the course of joint optimization.
Then, we show that when the outer problem is overparameterized, the choice of hypergradient approximation can affect which outer solution is found.
Experimental details and extended results are provided in Appendix~\ref{app:exp-details}.

\begin{figure*}[t]
\centering
\includegraphics[width=0.95\linewidth]{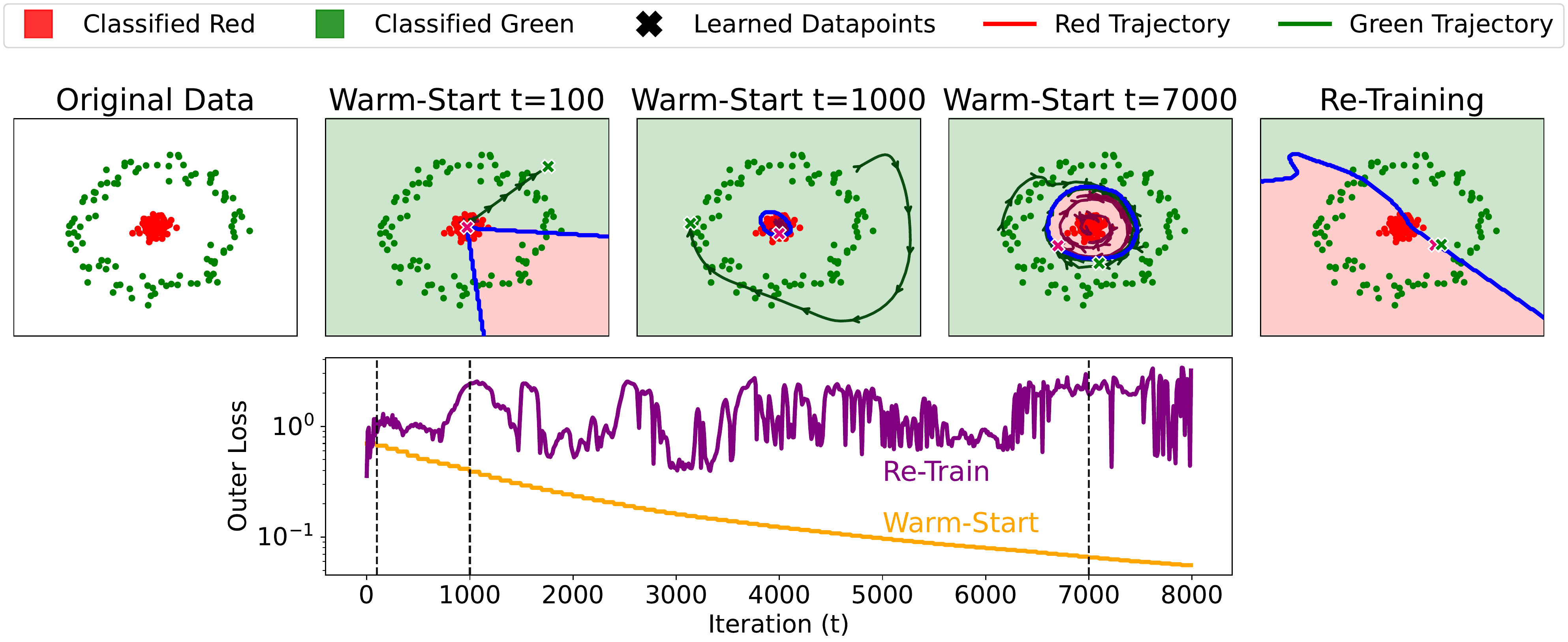}
\caption{\small
Dataset distillation for binary classification, with two learned datapoints (outer parameters) adapted jointly with the model weights (inner parameters).
\textbf{Top left:} The original data distribution we wish to distill; \textbf{Top middle three plots:} Visualizations of snapshots during training with warm-start BLO, at iterations $t \in \{100, 1000, 7000\}$ (indicated by the dashed vertical lines in the bottom plot).
In each middle figure, we plot the \textit{new portion of the synthetic datapoint trajectory} since the previous snapshot, shown by a curve with arrows; \textbf{Top right:} Decision boundary when re-training to convergence on the final values of the two distilled datapoints yields a poor solution for the original data.
\textbf{Bottom:} We show the outer loss over the course of a single run of warm-start BLO corresponding to the middle three plots in the top row.
We also demonstrate that none of the intermediate synthetic datapoint pairs along the trajectory is sufficient to fit the original data well, by re-training on each intermediate point (shown by the purple curve).
}
\label{fig:regression_and_classification}
\end{figure*}

\subsection{Inner Overparameterization: Dataset Distillation}
\label{sec:dataset-distillation}

In dataset distillation, the original dataset is only accessed in the outer objective; thus, one might expect that the lower-dimensional distilled dataset would act as an information bottleneck between the objectives.
Because the outer objective is only used directly to update the outer variables, it would seem intuitive that all of the information about the outer objective is compressed into the outer variables.
While this is correct for the cold-start equilibrium, we show that it does not hold for the warm-start equilibrium: \emph{a surprisingly large amount of information can leak from the original dataset to the inner variables (e.g., network weights)}.

Consider a 2D binary classification task where the classes form concentric rings (Figure~\ref{fig:regression_and_classification}).
We aim to learn \textit{two distilled datapoints} (one per class) to model the circular decision boundary.
One may \textit{a priori} expect that this would not be possible when training an MLP on only the two distilled points; indeed, we observe poor decision boundaries for the original dataset when training a model to convergence on the synthetic datapoints (Figure~\ref{fig:regression_and_classification}, Top Right).
However, when performing warm-start alternating updates on the MLP parameters and the learned datapoints, the datapoints follow a nontrivial \textit{trajectory}, tracing out the decision boundary between classes over time (Figure~\ref{fig:regression_and_classification}, middle three plots).
The model trained jointly with the two learned datapoints fits to the full trajectory of those datapoints.
Thus, warm-started BLO yields a model that achieves nearly the same outer loss as one trained directly on the original data, despite only training on a single datapoint per class.
See the caption of Figure~\ref{fig:regression_and_classification} for details on interpreting this result.
We provide additional results in Appendix~\ref{app:extended-dataset-distillation}, where we consider the three-class variant of this problem: we show that similar results hold when learning three distilled datapoints (one per class), and even when learning just two distilled datapoints along with their soft labels---in which case one datapoint switches its label over the course of training.

\paragraph{Warm-Start Memory.}
\begin{wrapfigure}[17]{r}{0.6\linewidth}
  \vspace{-0.6cm}
  \begin{center}
    \includegraphics[width=\linewidth]{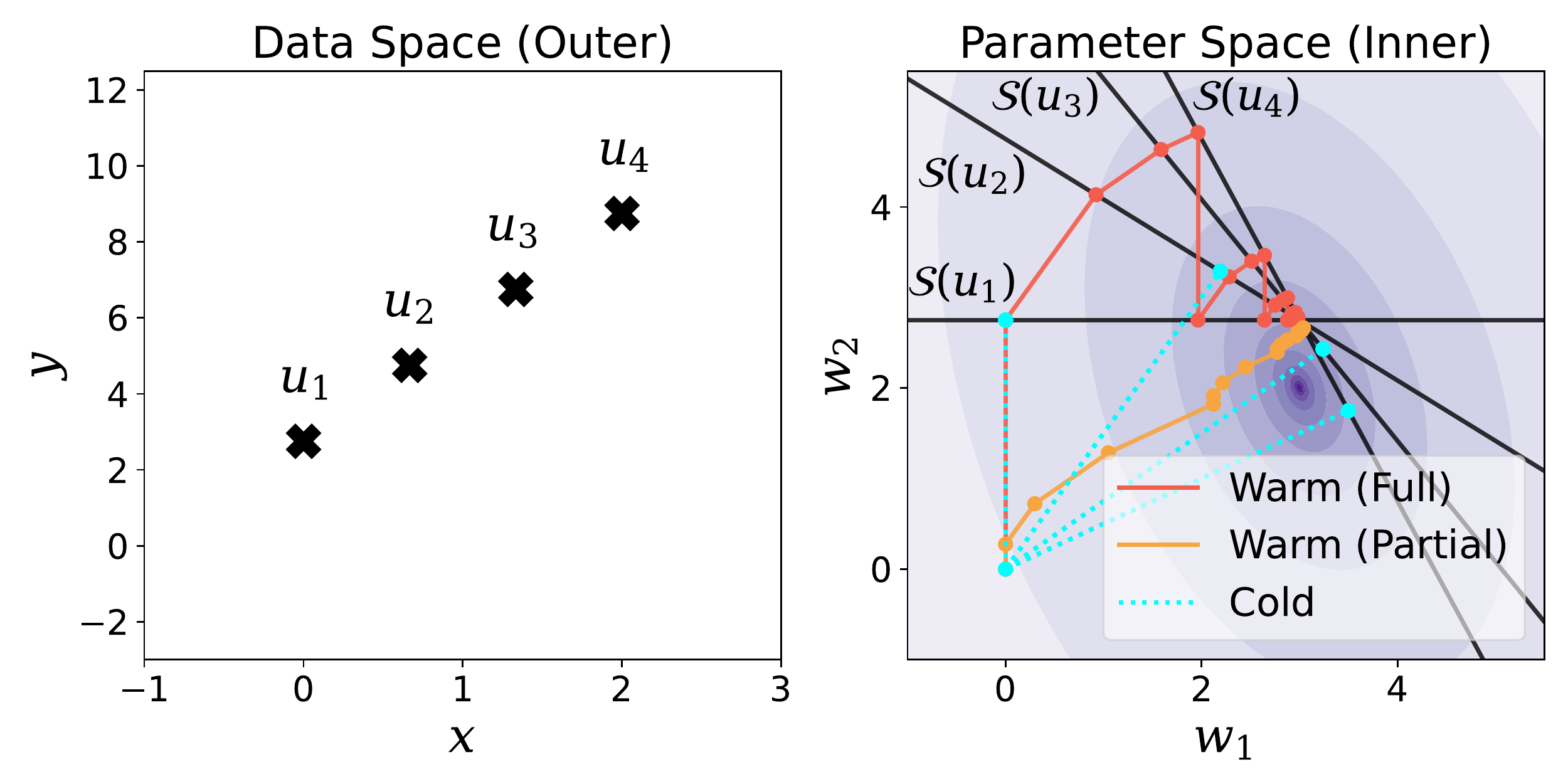}
  \end{center}
  \vspace{-0.2cm}
  \caption{\small
  Sketch illustrating projection onto a fixed sequence of solution sets $\{ \cS(\outParams_i) \}_{i=1}^4$, from different initializations. {\color{red}Warm-start with full inner optimization} projects from one solution set to the next; {\color{orange}warm-start with partial inner optimization} takes a step in the direction towards each solution set; and {\color{cyan}cold-start} re-projects from the origin.
  }
  \label{fig:projection}
\end{wrapfigure}
Here we provide intuition for the warm-start behavior observed in Figure~\ref{fig:regression_and_classification}.
In particular, we discuss how warm-start BLO can induce a \textit{memory effect}.
Figure~\ref{fig:projection} sketches warm- and cold-start algorithms on a toy example where we can visualize the steps of each algorithm in the inner parameter space.
The solution sets for four \textit{fixed} values of a single synthetic datapoint are shown by the solid black lines in Figure~\ref{fig:projection}, and outer loss contours are shown in the background.
The inner parameters are initialized at the origin.
Due to the implicit bias of the inner optimization problem, in each iteration of cold-start BLO, the initial inner parameters $\inParams_0 = \boldzero$ are projected onto the solution set corresponding to the current outer parameter: $\inParams_{k+1}^\star = \Pi(\boldzero, \cS(\outParams_k))$,
where $\Pi(\cdot, \mathcal{C})$ denotes projection onto the (convex) set $\mathcal{C}$.
In contrast, full warm-start projects the previous inner parameters onto the current solution set: $\inParams_{k+1}^\star = \Pi(\inParams_k^\star, \cS(\outParams_k))$.
If one cycles through the solution sets repeatedly, then full warm-start BLO is equivalent to the Kaczmarz algorithm~\citep{karczmarz1937angenaherte}, a classic alternating projection algorithm for finding a point in the intersection of the constraint sets (see Appendix~\ref{app:iterated-projection} for details).
In this case, the inner parameters will converge to the intersection of the solution sets $\{ \cS(\outParams_i) \}_{i=1}^4$, in effect yielding inner parameters that perform well for several outer parameters  simultaneously.
In the case of dataset distillation, this corresponds to model weights which fit all of the distilled datapoints over the outer optimization trajectory.

\begin{figure*}[t]
    \centering
    \begin{subfigure}{0.29\textwidth}
        \includegraphics[width=\linewidth]{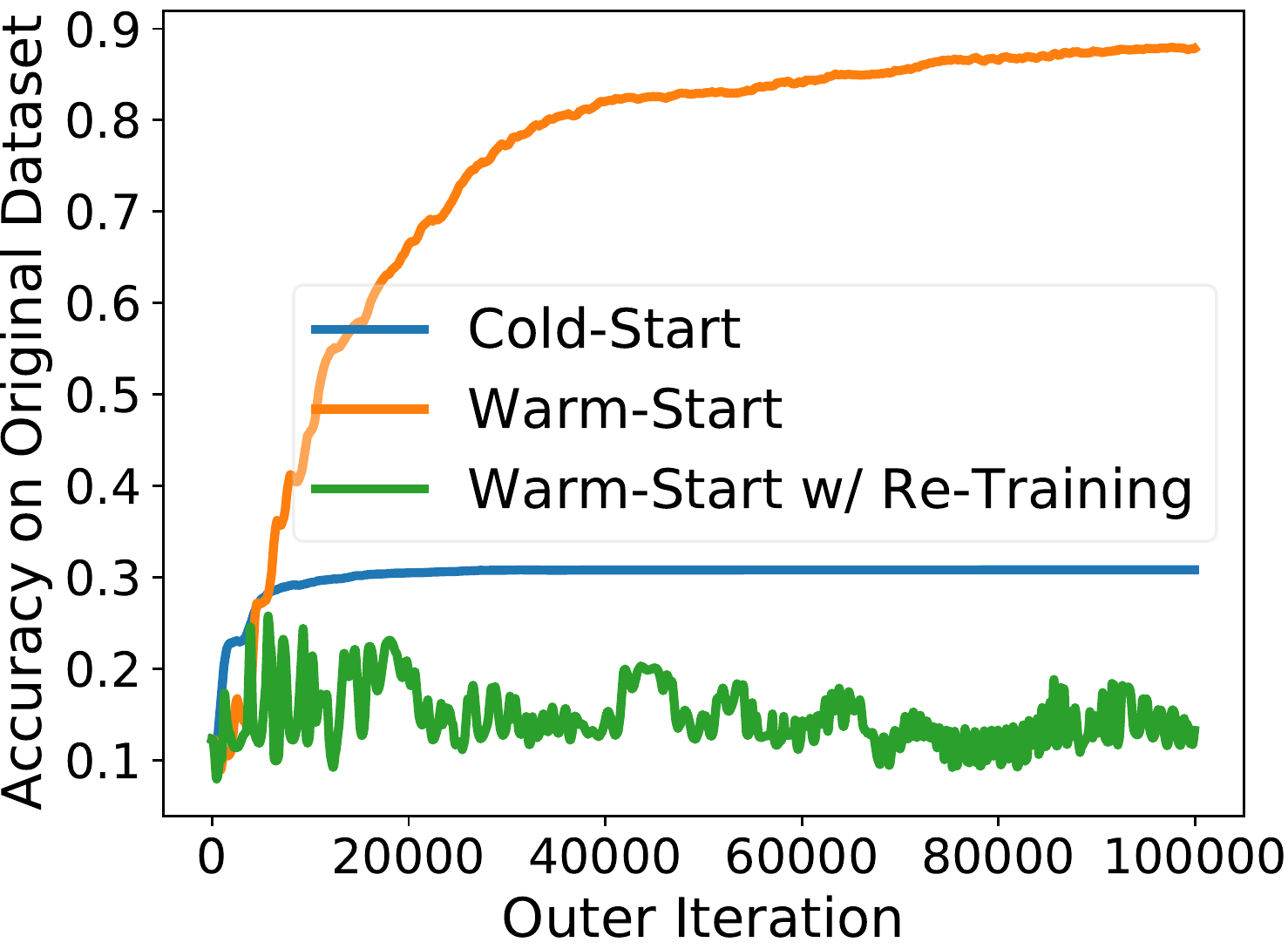}
        \caption{Accuracies}
        \label{fig:mnist-accs}
    \end{subfigure}
    \quad
    \begin{subfigure}{0.29\textwidth}
        \includegraphics[width=\linewidth]{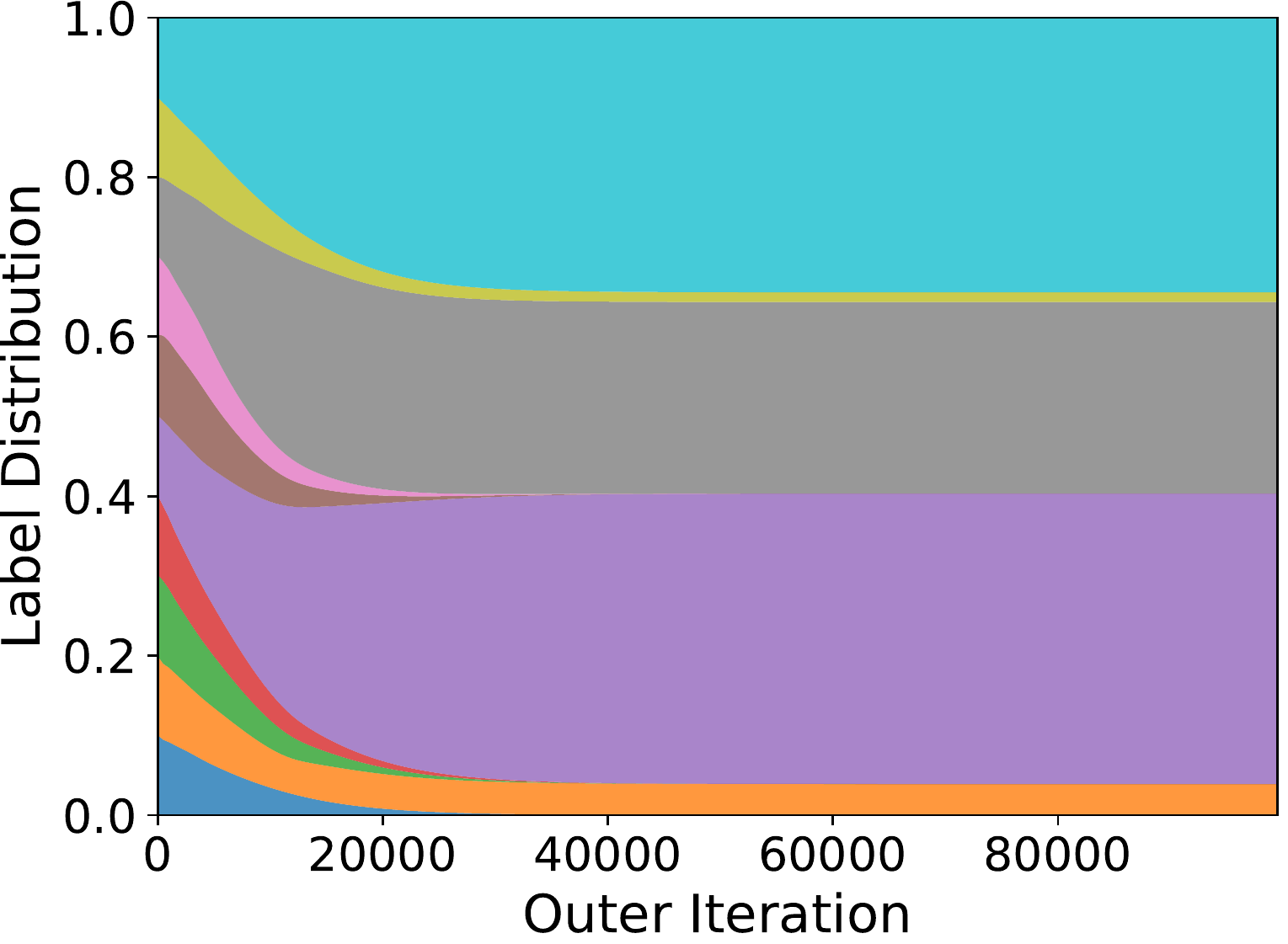}
        \caption{{\color{cyan}Cold-start} soft labels}
        \label{fig:mnist-soft-labels}
    \end{subfigure}
    \quad
    \begin{subfigure}{0.29\textwidth}
        \includegraphics[width=\linewidth]{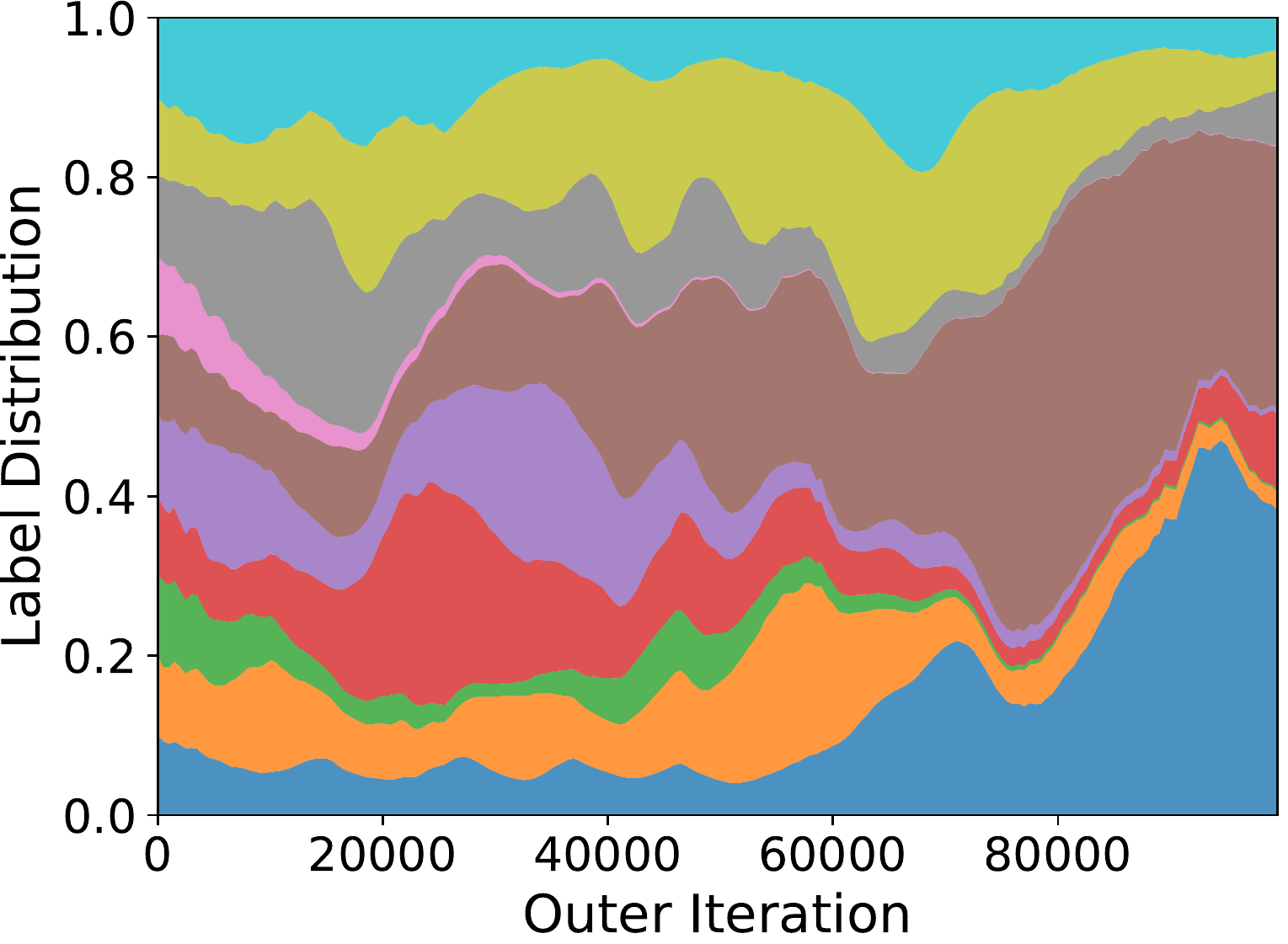}
        \caption{{\color{orange}Warm-start} soft labels}
        \label{fig:mnist-warm-labels}
    \end{subfigure}
    \label{fig:mnist-distillation}
    \caption{\small \textbf{MNIST dataset distillation using a linear classifier.} Here, we learn a single synthetic datapoint (a $28 \times 28$ image canvas) and corresponding soft label (a 10-dimensional vector representing class logits).}
\end{figure*}

\paragraph{Warm-Start vs Cold-Start in High-Dimensions.}
To illustrate warm-start phenomena in high-dimensional problems, we ran dataset distillation on MNIST using a linear classifier.
The original dataset to be distilled consists of 10000 examples from the MNIST training set.
The BLO methods are tasked with learning a single $28 \times 28$ image canvas and a 10-dimensional vector of soft labels---e.g., the outer parameters are $28^2 + 10 = 794$ dimensional.
Figure~\ref{fig:mnist-accs} shows the accuracy of the model on the original dataset, when optimizing a \textit{single synthetic datapoint and its soft label} with warm-start, cold-start, and warm-start + re-training (which trains a model from scratch on the warm-start synthetic datapoint at the current outer iteration).
We visualize the soft label evolution in Figures~\ref{fig:mnist-soft-labels} and \ref{fig:mnist-warm-labels}, showing classes 0-9 as colored regions.
While the cold-start soft label quickly converges to a mixture of three classes, warm-start continues to adapt the soft label over the course of joint optimization, effectively training the inner model on all 10 classes (e.g., by placing more weight on different classes at different timesteps along the training trajectory).
We obtained similar results on MNIST, FashionMNIST, and CIFAR-10, with 1 or 10 synthetic datapoints (see Table \ref{tab:high-dim-distill}).
In addition to dataset distillation, we observe similar phenomena in hyperparameter optimization.
We trained a linear data augmentation (DA) network that transforms inputs before feeding them into a classifier; here, the DA net parameters are hyperparameters $\outParams$, tuned on the validation set.
We used subsampled training sets consisting of 50 datapoints---so that data augmentation is beneficial---and evaluated performance on the full validation set.
Table~\ref{tab:high-dim-distill} compares warm- and cold-start BLO on this task.
Note that in order to compare with cold-start solutions (which need $10^3$--$10^4$ inner optimization steps), we \emph{require} tractable inner problems like linear models or MLPs.

\begin{table*}
\centering
\footnotesize
\begin{tabular}{@{}cccc|ccc@{}}
\toprule
 & \multicolumn{3}{c|}{\textbf{Dataset Distillation}}    & \multicolumn{3}{c}{\textbf{Data Augmentation Net}} \\ \midrule
\textbf{Method}                & \textbf{MNIST} & \textbf{Fashion} & \textbf{CIFAR-10} & \textbf{MNIST}  & \textbf{Fashion} & \textbf{CIFAR-10} \\ \midrule
{\color{cyan}\textbf{Cold-Start}}  & {\color{purple}30.7} / {\color{blue}89.1} & {\color{purple}33.2} / {\color{blue}83.0} & {\color{purple}17.6} / {\color{blue}46.9} & 84.86 & 84.04 & 45.38 \\
{\color{orange}\textbf{Warm-Start}} & {\color{purple}90.8} / {\color{blue}97.5} & {\color{purple}88.2} / {\color{blue}94.2} & {\color{purple}50.3} / {\color{blue}59.8} & 92.81 & 89.51 & 59.30      \\
\textbf{{\color{orange}Warm-Start} + Retraining} & {\color{purple}12.9} / {\color{blue}17.1} & {\color{purple}7.0} / {\color{blue}12.6} & {\color{purple}10.2} / {\color{blue}8.9} & 9.15 & 25.32 & 11.12  \\ \bottomrule
\end{tabular}
\vspace{-0.1cm}
\caption{\small \textbf{Columns 1-3:} Accuracy on original data with {\color{purple}1}/{\color{blue}10} synthetic samples. \textbf{Columns 4-6:} Learning a linear data augmentation network.}
\label{tab:high-dim-distill}
\end{table*}

\paragraph{Warm-Start Takeaways.}
Warm-start BLO yields outer parameters that \textit{fail to generalize} under re-initialization of the inner problem.
In both toy and high-dimensional problems, re-training with the final outer parameters (e.g., discarding the outer optimization trajectory) yields a model that performs poorly on the outer objective.
In addition, warm-start BLO \textit{leaks information} about the outer objective to the inner parameters, which can lead to overestimation of performance.
For example, when adapting a small number of hyperparameters online, warm-start BLO may overfit the validation set, yielding a model that fails to generalize to the test set.

\subsection{Outer Overparameterization: Anti-Distillation}
\label{sec:overparam-outer}

\begin{wrapfigure}[14]{r}{0.45\linewidth}
    \vspace{-0.4cm}
    \centering
    \includegraphics[width=\linewidth]{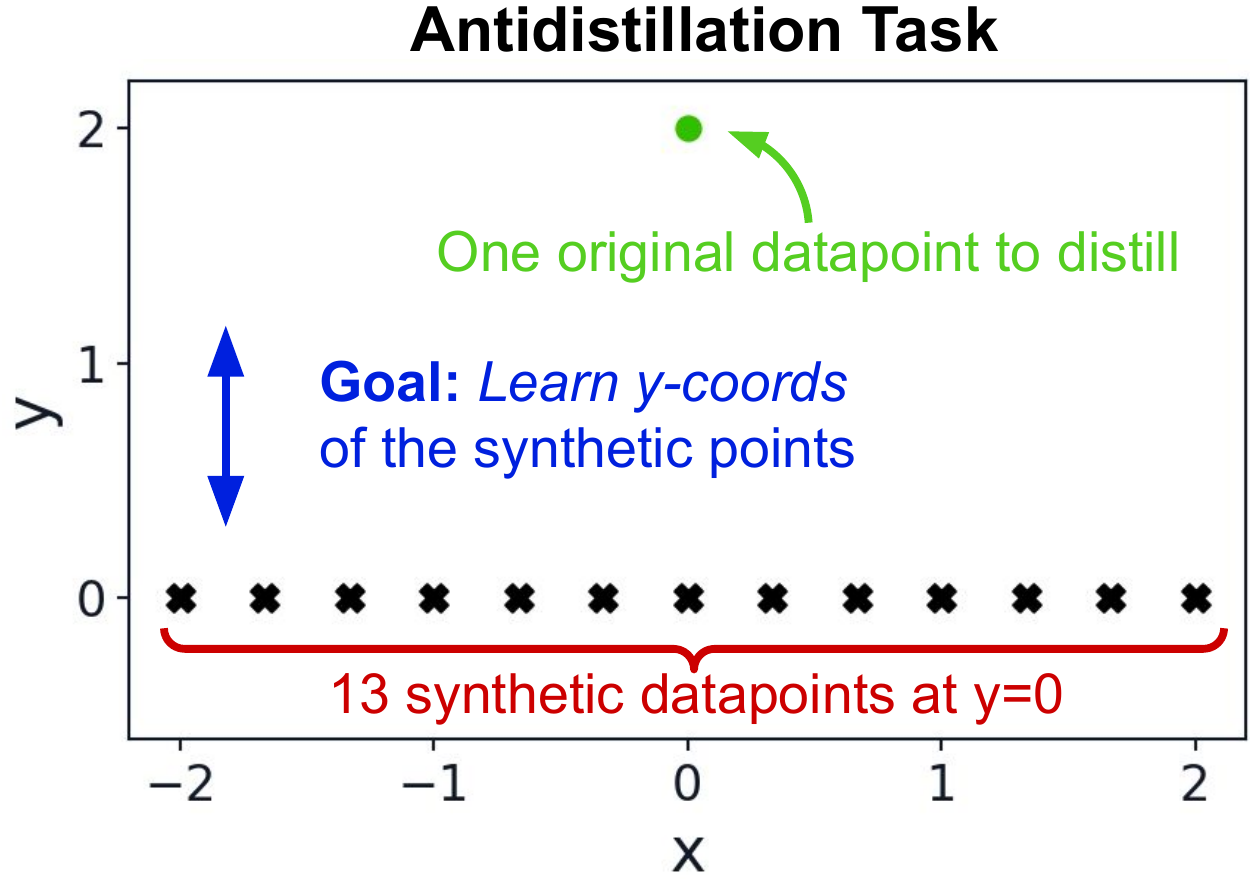}
    \vspace{-0.7cm}
    \caption{\small Antidistillation regression task setup.}
    \label{fig:antidistillation-setup}
\end{wrapfigure}
In the previous section, we investigated implicit bias resulting from the choice of cold- vs warm-start BLO algorithms.
Next, we focus on cold-start BLO, and show that when the outer problem is overparameterized, different hypergradient approximations can lead to vastly different outer solutions.
In particular, we construct a simple problem that illustrates the impact of using the truncated Neumann series or damping to approximate the inverse Hessian when computing the hypergradient.
We propose a task related to dataset distillation, but where we have \textit{more} learned datapoints than original dataset examples, which we term \textit{anti-distillation} (see Figure~\ref{fig:antidistillation-setup} for an illustration of the problem setup).
Here, there are many valid ways to set the learned datapoints such that a model trained on those points achieves good performance on the original data: any solution that places one learned datapoint on top of the original point perfectly fits the outer objective.

\paragraph{Anti-Distillation.}
Consider a regression task with inner objective $f(\outParams, \inParams) = \frac{1}{2} \| \boldPhi \inParams - \outParams \|_2^2$, where the outer parameters $\outParams$ represent only the targets, not the inputs.
The outer objective computes the loss on a fixed set of original datapoints $\{ (x^{(o)}_i, y^{(o)}_i) \}_{i=1}^P$: $F(\outParams, \inParams) = \frac{1}{2} \| \boldPhi^{(o)} \inParams - \boldy \|_2^2$.
Here, $f$ and $F$ are quadratics that satisfy Assumption~\ref{assumption:quadratic}.
The distillation task is to learn the $y$-coordinates of $N$ synthetic datapoints $\{ (x_i, y_i) \}_{i=1}^N$ with fixed $x$-coordinates.
The design matrix $\boldPhi$ is $N \times D$, where each row is a $D$-dimensional feature vector given by the following Fourier basis mapping:
\begin{align*}
    \boldphi(x)
    =
    \begin{bmatrix}
      1
      &
      \smash{
      \underbrace{
      \begin{matrix}
      2^{L-1} \cos(x)
      &
      2^{L-2} \cos(2x)
      &
      \cdots
      &
      \cos(Lx)
      \end{matrix}}_{L \cos \text{terms}}
      }
      &
      \smash{
      \underbrace{
      \begin{matrix}
      2^{L-1} \sin(x)
      &
      2^{L-2} \sin(2x)
      &
      \cdots
      &
      \sin(Lx)
      \end{matrix}}_{L \sin \text{terms}}
      }
    \end{bmatrix}
\end{align*}
\vspace{0.05cm}

%
We designed $\boldphi(x)$ such that low-frequency terms have larger amplitudes than high-frequency terms, yielding an inductive bias for smoothness: it is easier (e.g., requires a smaller adjustment to the weights $\inParams$) to fit a given curve using the low-frequency terms than using the high-frequency terms.
Thus, the minimum-norm solution found by gradient descent will explain as much as possible using low frequency terms.
Note that neural networks have also been observed to have a bias towards low-frequency information~\citep{basri2020frequency,rahaman2019spectral,tancik2020fourier}.\footnote{In Appendix~\ref{app:extended-anti-distillation}, we show a similar experiment using an MLP in the inner problem instead of a fixed Fourier basis mapping, where we observe qualitatively similar behavior.}

\paragraph{Hypergradient Approximations.}
We consider the cold-start solution for the inner problem from initialization $\inParams_0 = \boldzero$, which can be computed analytically as $\inParams^\star = \boldPhi^+ \outParams$.
At this inner solution, we compare three different methods to compute the hypergradient $\nabla_{\outParams} F(\outParams, \inParams)$, which differ in how they estimate the response Jacobian $\frac{\partial \inParams}{\partial \outParams}$: 1) differentiation through the closed-form, min-norm solution yields $\frac{\partial (\boldPhi^+ \outParams)}{\partial \outParams} = \boldPhi^+$; 2) implicit differentiation using the truncated, $K$-term Neumann series yields $\alpha \left( \sum_{j=0}^K (\boldI - \alpha \boldH)^j \right) \boldM$; and 3) implicit differentiation using the damped Hessian inverse yields $(\boldH + \epsilon \boldI)^{-1} \boldM$.

\begin{figure*}
    \centering
    \begin{subfigure}{.3\textwidth}
    \vspace{-0.2cm}
    \includegraphics[width=\linewidth]{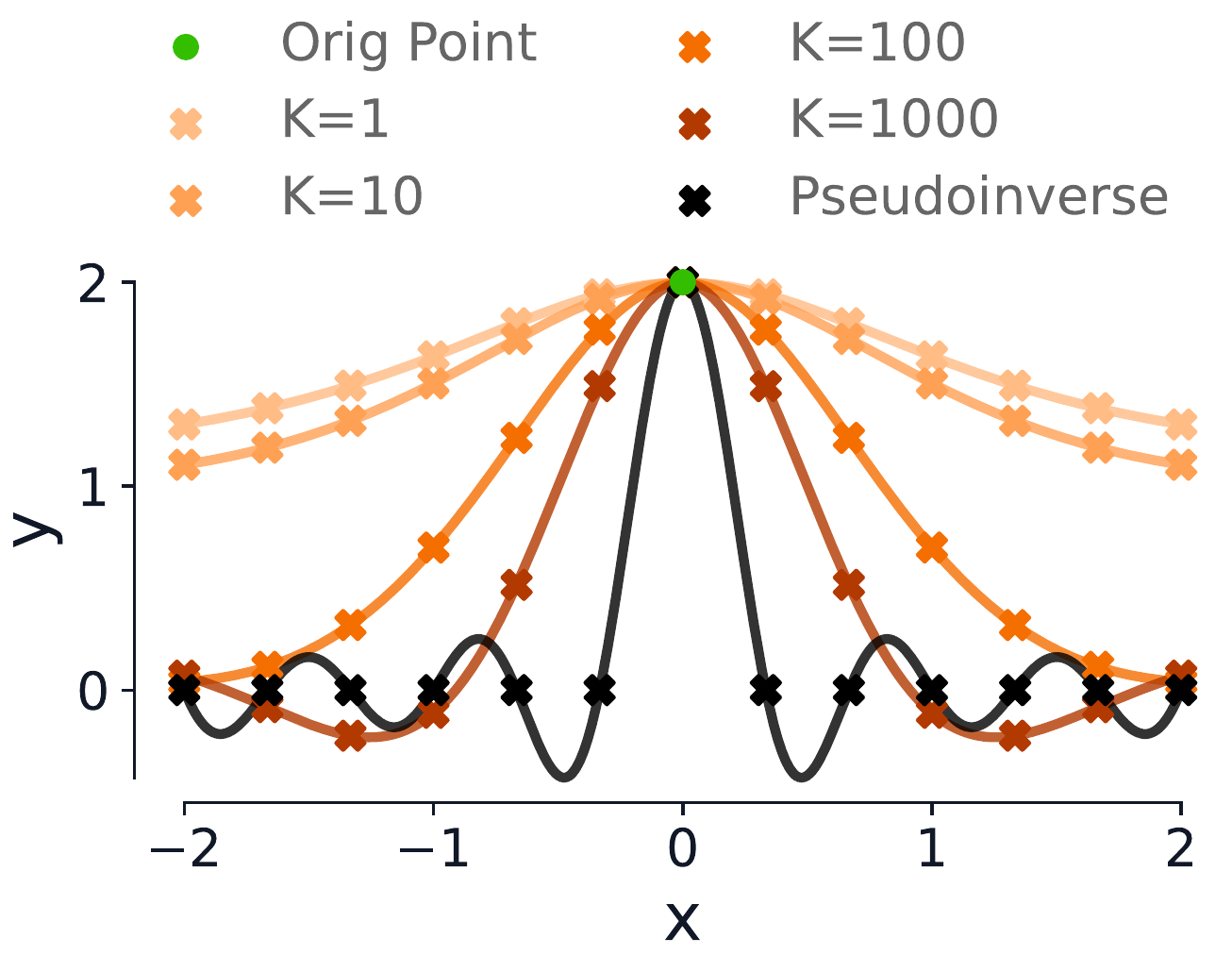}
      \vspace{-0.4cm}
      \caption{Neumann/unrolling}
      \label{fig:antidist-neumann}
    \end{subfigure}
    \quad
    \begin{subfigure}{.31\textwidth}
    \vspace{-0.2cm}
    \includegraphics[width=\linewidth]{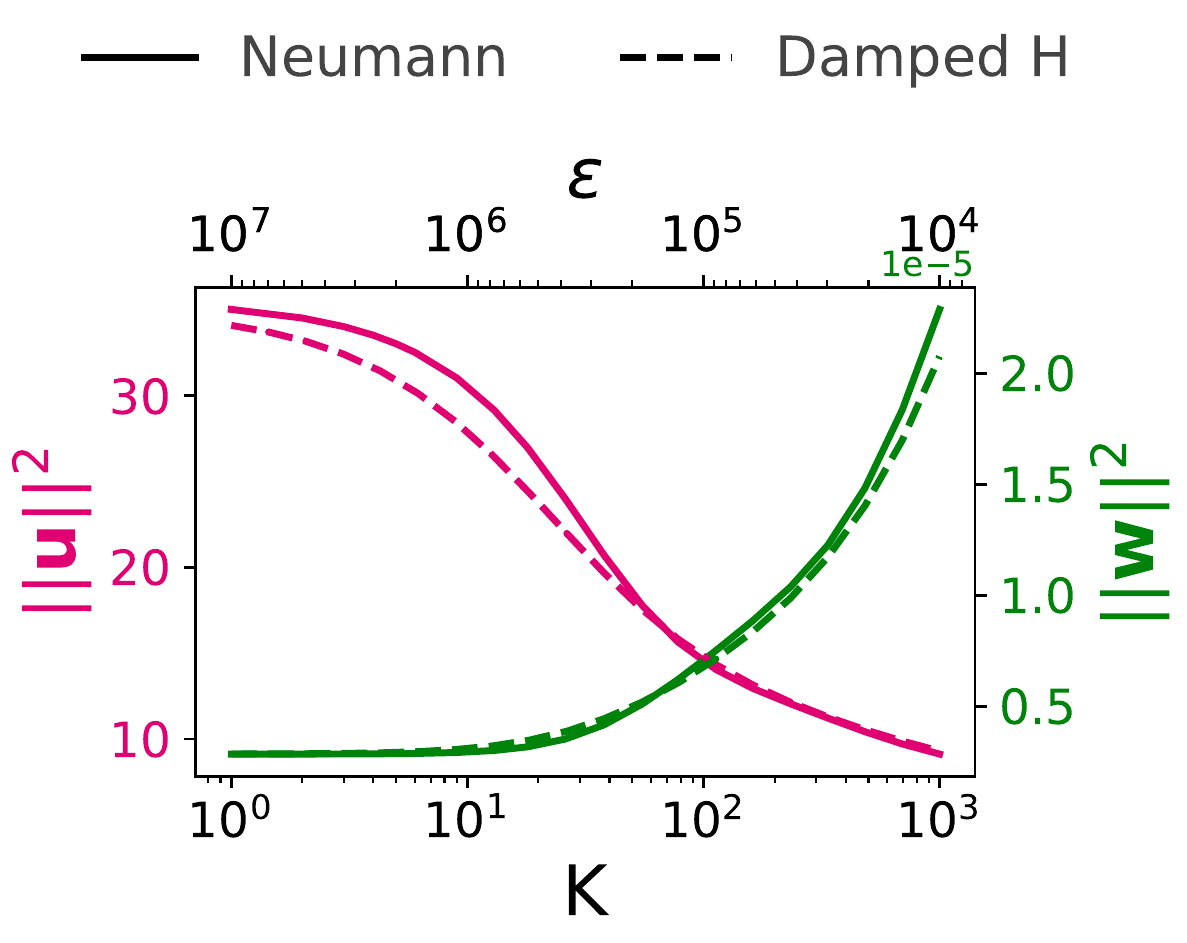}
    \vspace{-0.6cm}
    \caption{Parameter norms}
    \label{fig:antidist-norms}
    \end{subfigure}
    \quad
    \begin{subfigure}{.24\textwidth}
    \includegraphics[width=\linewidth]{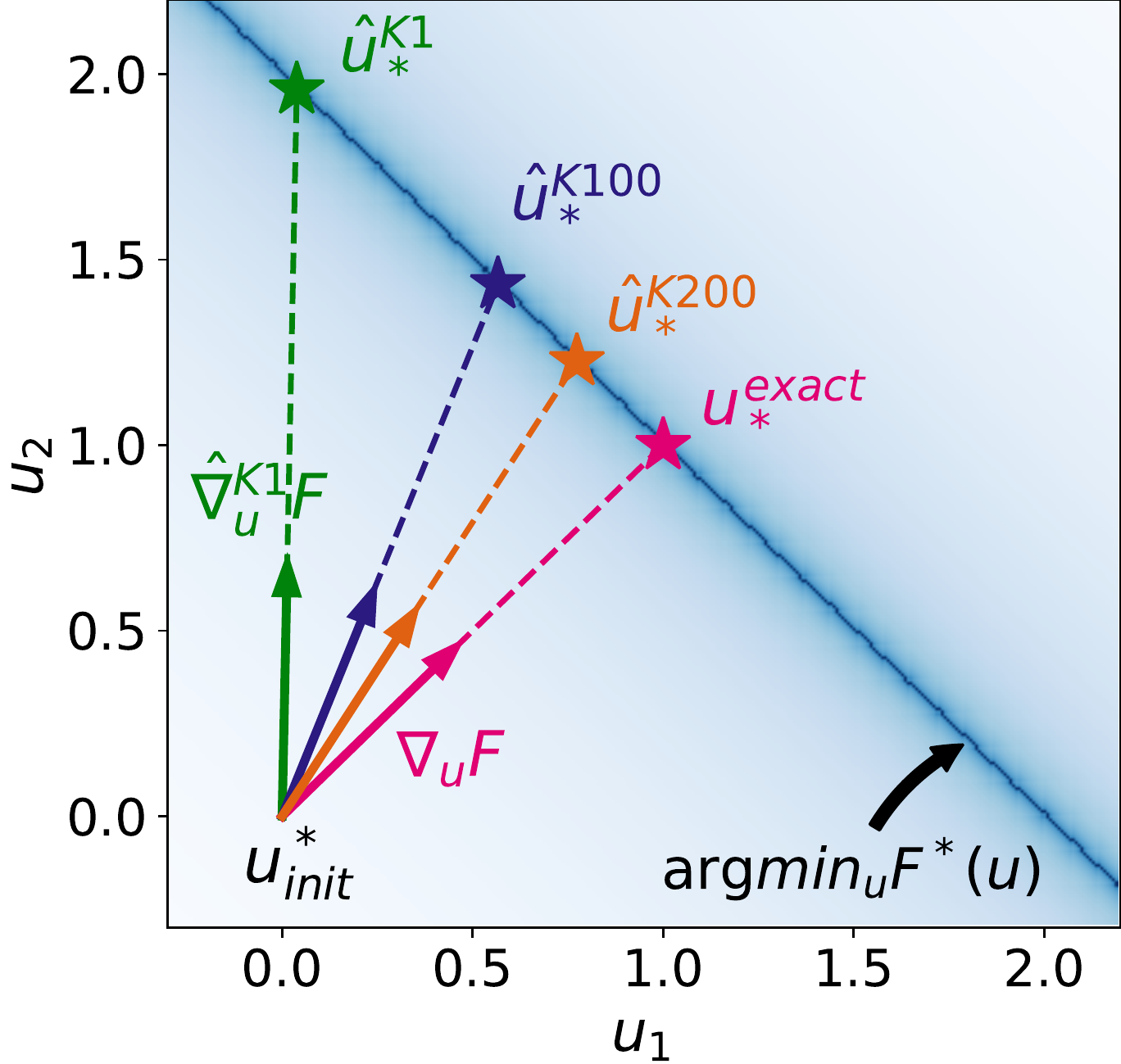}
    \caption{Approx. hypergrads}
    \label{fig:outer-hypergrad-approx}
    \end{subfigure}
    \vspace{-0.2cm}
    \caption{\footnotesize
    \textbf{Antidistillation task for linear regression with an overparameterized outer objective.}
    We learn the $y$-component of 13 synthetic datapoints such that a regressor trained on those points will fit a single original dataset point, shown by the {\color{OliveGreen}green dot at $(0,2)$}.
    Fig.~\textbf{(a)} shows the learned datapoints (outer parameters) obtained via different truncated Neumann approximations to the hypergradient.
    Fig.~\textbf{(b)} shows the norms of the outer parameters $\|\outParams\|^2$ as a function of $K$ (for Neumann/unrolling) or $\epsilon$ (for the damped Hessian inverse).
    We observe that better hypergradient approximations (e.g., larger $K$ or smaller $\epsilon$) lead to smaller norm outer parameters, because they account for both high- and low-curvature directions of the inner objective.
    Fig.~\textbf{(c)} visualizes outer optimization trajectories in the outer parameter space, to provide intuition for the behavior in \textbf{(a)} and \textbf{(b)}.
    We consider antidistillation with 1 original datapoint and 2 learned datapoints.
    We show the true hypergradient $\nabla_{\outParams} F$, approximations using truncated Neumann series $\hat{\nabla}_{\outParams} F$, and the converged outer parameters for each setting, e.g., $\hat{\outParams}^{K1}_\star$.
    }
    \label{fig:overparam-regression}
    \vspace{-0.3cm}
\end{figure*}

\paragraph{Results.}
When we estimate the response Jacobian with a small number of steps of unrolling, the inner optimizer will do its best to fit the data using the low-frequency basis terms, which correspond to high-curvature directions of the inner loss.
In Figure~\ref{fig:antidist-neumann}, we show the solutions obtained for different truncations of the Neumann series; also note that, because this problem is quadratic, the $K$-step unrolled hypergradient coincides with the $K$-step Neumann series hypergradient.
To demonstrate the relationship $\alpha \sum_{j=0}^K (\boldI - \alpha \boldH)^j \approx (\boldH + \frac{1}{\alpha K} \boldI)^{-1}$ empirically, we also show the distilled datasets obtained with the damped inverse Hessian hypergradient approximation, for various damping factors $\epsilon \in \{ \text{1e4, 1e5, 1e6, 1e7} \}$ (where $\epsilon = \frac{1}{\alpha K}$), in Appendix~\ref{app:extended-anti-distillation} Figure~\ref{fig:overparam-epsilon}, where we observe similar behavior to the Neumann series.
In Figure~\ref{fig:antidist-norms} we plot the norms of the converged inner and outer parameters for a range of Neumann truncation lengths $K$ and damping factors $\epsilon$.
We see that although they are not equivalent, the Neumann and damped Hessian approximate hypergradients behave similarly.
In Appendix~\ref{app:exp-details}, Figure~\ref{fig:overparam-mlp}, we show similar behavior when using an MLP rather than a linear model.
Also note that this effect occurs in general for quadratic problems that are outer-overparameterized; it does not depend on inner overparameterization (see Appendix~\ref{app:extended-anti-distillation}).

Figure~\ref{fig:outer-hypergrad-approx} visualizes optimization trajectories in the outer parameter space to provide additional intuition for the behavior of the outer parameter norms in Figure~\ref{fig:antidist-norms}.
We consider antidistillation with 1 original and 2 learned datapoints.
We use $\hat{\nabla}_{\outParams}^K F(\outParams, \inParams)$ to denote the approximate hypergradient obtained via implicit differentiation with the $K$-term truncated Neumann series.
We show the true hypergradient $\nabla_{\outParams} F$ (which we define in terms of the pseudoinverse), approximations using truncated Neumann series $\hat{\nabla}_{\outParams} F$, and the converged outer parameters for each setting, e.g., $\hat{\outParams}^{K1}_\star$.
We observe that: 1) when the outer problem is overparameterized, approximate hypergradients converge to valid solutions in $\argmin_{\outParams} F^\star(\outParams)$; and 2) the exact hypergradient converges to the min-norm outer solution.

\vspace{-0.1cm}
\section{Related Work}
\label{sec:related-work}

\paragraph{Overparameterization.}
Overparameterization has long been studied in single-level optimization, generating key insights such as neural network behavior in the infinite-width limit~\citep{jacot2018neural,sohl2020infinite,lee2019wide}, double descent phenomena~\citep{nakkiran2019deep,belkin2018reconciling}, the ability to fit random labels~\citep{zhang2016understanding}, and the inductive biases of optimizers.
However, prior work has not investigated implicit bias in bilevel optimization.
While we focus quadratic problems, note that in single-level optimization, analyses based on the Neural Tangent Kernel (NTK)~\citep{sohl2020infinite,lee2019wide} are often extendable to nonlinear networks.
Implicit bias has a long history in machine learning: many works have observed and studied the connection between early stopping and $L_2$ regularization~\citep{strand1974theory,morgan1989generalization,friedman2003gradient,yao2007early}.
Interest in implicit bias has increased over the past few years~\citep{nacson2019stochastic,soudry2018implicit,suggala2018connecting,poggio2019theoretical,ji2019implicit,ali2019continuous,ali2020implicit}.

\paragraph{Prior Work using Warm Start.}
Many papers perform joint optimization of the inner and outer parameters (e.g., data augmentations together with a base model), such as~\citep{hataya2020faster,hataya2020meta,ho2019population,mounsaveng2021learning,peng2018jointly,tang2020onlineaugment}.
~\citep{ghadimi2018approximation},~\citep{hong2020two}, and ~\citep{ji2020bilevel} propose bilevel algorithms that use warm-starting; however, they focus on analyzing the convergence rates of their algorithms and do not consider inner underspecification.

\paragraph{Gap Between Theory \& Practice.}
Existing BLO theory typically assumes unique solutions to the inner (and sometimes outer) problem, and focuses on showing that approximation methods get (provably) close to the solution.
~\citet{shaban2019truncated} provide conditions where optimization with an approximate hypergradient $\hat{\boldh}$ from truncated unrolling converges to a BLO solution, but they assume uniqueness of the inner solution.
~\citet{grazzi2020iteration,grazzi2020convergence} study iteration complexity and convergence of hypergradient approximations in the strongly-convex inner problem setting.
Several other works focus on convergence rate analyses~\citep{ji2021bilevel,yang2021provably,ji2021lower}.

\paragraph{Hyperparameter Optimization (HO).}
There are three main approaches for gradient-based HO: 1) differentiating through unrolls of the inner problem, sometimes called \textit{iterative differentiation}~\citep{domke2012generic,maclaurin2015gradient,shaban2019truncated}; 2) using \textit{implicit differentiation} to compute the response Jacobian assuming that the inner optimization has converged~\citep{larsen1996design,bengio2000gradient,foo2008efficient,pedregosa2016hyperparameter}; and 3) using a \textit{hypernetwork}~\cite{ha2016hypernetworks} to approximate the best-response function locally, $\hat{\inParams}_{\phi}(\boldlam) \approx \inParams^\star(\boldlam)$, such that the outer gradient can be computed using the chain rule through the hypernetwork, $\frac{\partial \LV}{\partial \boldlam} = \frac{\partial \LV}{\partial \hat{\inParams}\phi(\boldlam)} \frac{\partial \hat{\inParams}\phi(\boldlam)}{\partial \boldlam}$.
Hypernetworks have been applied to HO in~\citep{lorraine2018stochastic,mackay2019self,bae2020delta}.
~\citet{mackay2019self} and ~\citet{bae2020delta} have observed that STNs learn online hyperparameter schedules (e.g., for dropout rates and augmentations) that can outperform any fixed hyperparameter value.
We believe that warm-start effects may partially explain the observed improvements from hyperparameter schedules.

\paragraph{Truncation Bias.}
We restrict our focus to bilevel problems in which the outer parameters \textit{affect the fixed points} of the inner problem---this includes dataset distillation and hyperparameter optimization for most regularizers, but not the \textit{optimization hyperparameters} such as the learning rate and momentum.
Truncation bias has been shown to lead to critical failures when used for tuning such hyperparameters~\citep{wu2018understanding,metz2019understanding}.
In contrast, greedy adaptation of regularization hyperparameters has been successful empirically, via population-based training~\citep{jaderberg2017population}, hypernetwork-based HO~\citep{lorraine2018stochastic,mackay2019self,bae2020delta}, and online implicit differentiation~\citep{lorraine2020optimizing,hataya2020meta}.

\paragraph{Hysteresis.}
~\citet{luketina2016scalable} are among the first we are aware of, who considered the effect of hysteresis (e.g., path-dependence) on the model obtained via alternating optimization of the model parameters and hyperparameters.
The HO algorithm they introduce, dubbed T1-T2, uses the IFT with the identity matrix as an approximation to the inverse Hessian (e.g., equivalent to using $K=0$ terms of the Neumann series).
Interestingly, in contrast to more recent HO papers (particularly ones that focus on data augmentation), ~\citet{luketina2016scalable} found that re-training the model from scratch performed better than the online model.
One potential explanation for this is the choice of hyperparameters: they adapted L2 coefficients and Gaussian noise added to inputs and activations, rather than augmentations.

\section{Conclusion}
\label{sec:conclusion}
Most work on bilevel optimization has made the simplifying assumption that the solutions to the inner and outer problems are unique.
However, this does not hold in many practical applications, such as dataset distillation for overparameterized neural networks.
We investigated overparameterized bilevel optimization, where either the inner or outer problems may admit non-unique solutions.
We formalized warm- and cold-start equilibria, which correspond to common BLO algorithms.
We analyzed the properties of these equilibria, and algorithmic choices such as the number of Neumann series terms used to approximate the hypergradient.
We presented several tasks illustrating how these choices can significantly affect the solutions obtained in practice.
More generally, we highlighted the importance of, and laid groundwork for, analyzing the effects of overparameterization in nested optimization problems.

\section*{Acknowledgements}
We thank James Lucas, Guodong Zhang, and David Acuna for helpful feedback on the manuscript.
Resources used in this research were provided, in part, by the Province of Ontario, the Government of Canada through CIFAR, and companies sponsoring the Vector Institute (\url{www.vectorinstitute.ai/partners}).

{\small
\bibliographystyle{icml2022}
\bibliography{paper}
}

\clearpage

\appendix

\onecolumn

\section*{Appendix}

This appendix is structured as follows:
\begin{itemize}
    \item In Section~\ref{app:notation}, we provide an overview of the notation we use.
    \item In Section~\ref{app:derivations}, we provide derivations of formulas used in the main text.
    \item In Section~\ref{app:proofs}, we provide proofs of the theorems in the main text.
    \item In Section~\ref{app:proximal-br}, we derive the response Jacobian for a proximal inner objective, recovering a form of the IFT using the damped Hessian inverse.
    \item In Section~\ref{app:unrolling-and-neumann}, we provide an overview of the result from ~\cite{lorraine2020optimizing} that shows that differentiating through $K$ steps of unrolling starting from optimal inner parameters $\inParams^\star$ is equivalent to approximating the inverse Hessian with the first $K$ terms of the Neumann series.
    \item In Section~\ref{app:exp-details}, we provide experimental details and extended results, as well as a \href{https://docs.google.com/document/d/1whS9xmU_gfwsqtpUchf3ltlXSxcA_8iPoMwAfNqXNMs/edit?usp=sharing}{link to an animation} of the bilevel training dynamics for the dataset distillation task from Section~\ref{sec:dataset-distillation}.
\end{itemize}

\clearpage

\section{Notation}
\label{app:notation}
\begin{table}[H]
\begin{center}
    \begin{tabular}{c c}
        \toprule
        $F$ & Outer objective \\[2pt]
        $f$ & Inner objective \\[2pt]
        $\outParams$ & Outer variables \\[2pt]
        $\inParams$ & Inner variables \\[2pt]
        $\mathcal{U}$ & Outer parameter space \\[2pt]
        $\mathcal{W}$ & Inner parameter space \\[2pt]
        $\rightrightarrows$ & Multi-valued mapping between two sets \\[2pt]
        $\cS(\outParams)$ & Set-valued response mapping for $\outParams$: $\cS(\outParams) = \argmin_{\inParams} f(\outParams, \inParams)$ \\[2pt]
        $\boldPhi$ & Design matrix, where each row corresponds to an example \\[2pt]
        $||\cdot||_2^2$ & (Squared) Euclidean norm \\[2pt]
        $||\cdot||_F^2$ & (Squared) Frobenius norm \\[2pt]
        $\boldA^+$ & Moore-Penrose pseudoinverse of $\boldA$ \\[2pt]
        $\outParams^\star_k$ & A fixpoint of the outer problem obtained via the $k$-step unrolled hypergradient \\[2pt]
        $\hat{\nabla}_{\outParams}^{K10} F$ & Hypergradient approximation using $K=10$ terms of the Neumann series \\[2pt]
        $\alpha$ & Learning rate for inner optimization or step size for the Neumann series \\[2pt]
        $\beta$ & Learning rate for outer optimization \\[2pt]
        $k$ & Number of unrolling iterations / Neumann steps \\[2pt]
        $F^\star(\outParams)$ & $F(\outParams, \inParams^\star_{\outParams})$ where $\inParams^\star_{\outParams} \in \cS(\outParams)$ \\[2pt]
        $\inParams^\star$ & An inner solution, $\inParams^\star \in \cS(\outParams) = \argmin_{\inParams} f(\outParams, \inParams)$ \\[2pt]
        $\inParams_0$ & The inner parameter initialization \\[2pt]
        $\mathcal{N}(\boldA)$ & Nullspace of $\boldA$ \\[2pt]
        PSD & Positive semi-definite \\[2pt]
        Neumann series & If $(\boldI - \boldA)$ is contractive, then $\boldA^{-1} = \sum_{j=0}^\infty (\boldI - \boldA)^j$ \\[2pt]
        {\begin{tabular}{c}
            Implicit Differentiation\\
            Hypergradient
        \end{tabular}} & $\frac{d F}{d \outParams} = \pp{\frac{\partial \inParams^\star}{\partial \outParams}}^\top \pp{\frac{\partial F(\outParams, \inParams^\star)}{\partial \inParams}} = - \pp{\frac{\partial^2 f(\inParams^\star, \outParams)}{\partial \inParams \partial \inParams^\top}}^{-1} \pp{\frac{\partial^2 f(\inParams^\star, \outParams)}{\partial \outParams \partial \inParams}} \pp{\frac{\partial F(\outParams, \inParams^\star)}{\partial \inParams}}$ \\[8pt]
        $\textnormal{HypergradApprox}(\outParams, \inParams)$ & Compute hypergradient approximation at given inner/outer parameters \\[2pt]
        Optimistic BLO & $\outParams^\star \in \argmin_{\outParams} F(\outParams, \inParams^\star_{\outParams}) \quad \text{s.t.} \quad \inParams^\star_{\outParams} \in \argmin_{\inParams \in \cS(\outParams)} F(\outParams, \inParams)$ \\[2pt]
        Pessimistic BLO & $\outParams^\star \in \argmin_{\outParams} F(\outParams, \inParams^\star_{\outParams}) \quad \text{s.t.} \quad \inParams^\star_{\outParams} \in \argmax_{\inParams \in \cS(\outParams)} F(\outParams, \inParams)$ \\[2pt]
        BLO & Bilevel optimization \\[2pt]
        HO & Hyperparameter optimization \\[2pt]
        NAS & Neural architecture search \\[2pt]
        DD & Dataset distillation \\[2pt]
        BR & Best-response\\[2pt]
        IFT & Implicit Function Theorem\\[2pt]
        \bottomrule
    \end{tabular}
\end{center}
\caption{Summary of the notation and abbreviations used in this paper.}
\label{tab:TableOfNotation}
\end{table}

\section{Derivations}
\label{app:derivations}

\subsection{Minimum-Norm Response Jacobian}
\label{app:min-norm-jacobian}
Suppose we have the following inner problem:
$$
f(\inParams, \outParams) = \frac{1}{2} \norm{\boldPhi \inParams - \outParams}^2_2
$$
The gradient, Hessian, and second-order mixed partial derivatives  of $f$ are:
$$\frac{\partial f(\inParams, \outParams)}{\partial \inParams} = \boldPhi^\top \boldPhi \inParams - \boldPhi^\top \outParams
\qquad
\frac{\partial^2 f(\inParams, \outParams)}{\partial \inParams \partial \inParams^\top} = \boldH = \boldPhi^\top \boldPhi
\qquad
\frac{\partial^2 f(\inParams, \outParams)}{\partial \outParams \partial \inParams} = \boldPhi^\top
$$
The minimum-norm response Jacobian is:
\begin{align}
    \frac{\partial \inParams^\star(\outParams)}{\partial \outParams} = \argmin_{\boldH \boldM = \boldPhi} || \boldM ||_F^2
\end{align}
We need a solution to the linear system $\boldH \boldM = \boldPhi$.
Assuming this system is satisfiable, the matrix $\boldQ = \boldH^+ \boldPhi$ is a solution that satisfies $||\boldQ||_F^2 \leq ||\boldM||_F^2$ for any matrix $\boldM$:
\begin{align}
    \boldQ = \boldH^+ \boldPhi = (\boldPhi^\top \boldPhi)^+ \boldPhi = (\boldPhi^\top \boldPhi)^{-1} \boldPhi = \boldPhi^+
\end{align}
Thus, the minimum-norm response Jacobian is the Moore-Penrose pseudoinverse of the feature matrix, $\boldPhi^+$.

\subsection{Iterated Projection.}
\label{app:iterated-projection}

Alternating projections is a well-known algorithm for computing a point in the intersection of convex sets.
Dijkstra's projection algorithm is a modified version which finds a specific point in the intersection of the convex sets, namely the point obtained by projecting the initial iterate onto the intersection.
Iterated algorithms for projection onto the intersection of a set of convex sets are studied in~\citep{stovsic2016projection}.
~\citet{angelos1998limit} show that successive projections onto hyperplanes can converge to limit cycles.
~\citet{mishachev2019realization} study algorithms based on sequential projection onto hyperplanes defined by equations of a linear system (like Kaczmarz), and show that with specifically-chosen systems of equations, the limit polygon of such an algorithm can be any predefined polygon.

\paragraph{Closed-Form Projection Onto the Solution Set}
\label{app:projection}

Here, we provide a derivation of the formula we use to compute the analytic projections onto the solution sets given by different hyperparameters in Figure~\ref{fig:projection}.
This derivation is known in the literature on the Kaczmarz algorithm~\citep{karczmarz1937angenaherte}; we provide it here for clarity.

Consider homogeneous coordinates for the data $\mathbf{x} = [x, 1]$ such that we can write the weights as $\mathbf{w} = [w_1, w_2]$ and the model as $\mathbf{w}^\top \mathbf{x} = y$.
Given an initialization $\mathbf{w}_0$, we would like to find the point $\mathbf{w}^*$ such that:
\begin{equation}
\mathbf{w}^* = \arg \min_{\{ \mathbf{w} \mid \mathbf{w}^\top \mathbf{x} = y \}} \frac{1}{2} || \mathbf{w} - \mathbf{w}_0 ||^2    
\end{equation}

To find the solution to this problem, we write the Lagrangian:
\begin{equation}
L(\mathbf{w}, \lambda) = \frac{1}{2} || \mathbf{w} - \mathbf{w}_0 ||^2 + \lambda (\mathbf{w}^\top \mathbf{x} - y) 
\end{equation}

The solution to this problem is the stationary point of the Lagrangian. Taking the gradient and equating its components to 0 gives:
\begin{align}
\nabla_{\mathbf{w}} L(\mathbf{w}, \lambda) &= \mathbf{w} - \mathbf{w}_0 + \lambda \mathbf{x} = 0 \\
\nabla_{\lambda} L(\mathbf{w}, \lambda) &= \mathbf{w}^\top \mathbf{x} - y = 0
\end{align}
From the first equation, we have:
\begin{equation}
\mathbf{w} = \mathbf{w}_0 - \lambda \mathbf{x}    
\end{equation}

Plugging this into the second equation gives:
\begin{align}
\mathbf{x}^\top \mathbf{w} - y &= 0 \\
\mathbf{x}^\top (\mathbf{w}_0 - \lambda \mathbf{x}) - y &= 0 \\
\mathbf{x}^\top \mathbf{w}_0 - \lambda \mathbf{x}^\top \mathbf{x} - y &= 0 \\
\lambda \mathbf{x}^\top \mathbf{x} &= \mathbf{x}^\top \mathbf{w}_0 - y \\
\lambda &= \frac{\mathbf{x}^\top \mathbf{w}_0 - y}{\mathbf{x}^\top \mathbf{x}}
\end{align}

Finally, plugging this expression for $\lambda$ back into the equation $\mathbf{w} = \mathbf{w}_0 - \lambda \mathbf{x}$, we have:
\begin{equation}
\mathbf{w} = \mathbf{w}_0 - \frac{\mathbf{x}^\top \mathbf{w}_0 - y}{\mathbf{x}^\top \mathbf{x}} \mathbf{x}    
\end{equation}

\section{Proofs}
\label{app:proofs}

First, we prove a well-known result on the implicit bias of gradient descent used to optimize convex quadratic functions (Lemma~\ref{lemma:gd}), which we use in this section.

\begin{tcolorbox}[colback=tabblue!15,boxrule=0pt,colframe=white,coltext=black,arc=2pt,outer arc=0pt,valign=center]
\begin{lemma}[Gradient Descent on a Quadratic Function Finds a Min-Norm Solution]
\label{lemma:gd}
Suppose we have a lower-bounded, convex quadratic function:
\begin{equation}
    g(\inParams) = \frac{1}{2} \inParams^\top \boldA \inParams + \boldb^\top \inParams + c \,,
\end{equation}
where $\boldA \in \mathbb{R}^{|\inSpace| \times |\inSpace|}$ is symmetric positive semidefinite, $\boldb \in \mathbb{R}^{|\inSpace|}$, and $c \in \mathbb{R}$.
If we start from initialization $\inParams_0$, then gradient descent with an appropriate learning rate will converge to a solution $\inParams^\star \in \argmin_{\inParams} g(\inParams)$ which has minimum $L_2$ distance from $\inParams_0$:
\begin{equation}
    \inParams^\star \in \argmin_{\inParams \in \argmin_{\inParams} g(\inParams)} \frac{1}{2} \norm{\inParams - \inParams_0}^2_2 \,.
\end{equation}
\end{lemma}
\end{tcolorbox}
\begin{proof}
Note that $\nabla_{\inParams} g(\inParams) = \boldA \inParams + \boldb$.
By the lower-bounded assumption, the minimum of $g$ is reached when $\nabla_{\inParams} g(\inParams) = 0$; one solution, which is the minimum-norm solution with respect to the origin, is $\inParams = - \boldA^+ \boldb$, where $\boldA^+$ denotes the Moore-Penrose pseudoinverse of $\boldA$.
The minimum-cost subspace is defined by:
\begin{equation}
\argmin_{\inParams} g(\inParams) = \{ \inParams^\star + \inParams' \mid \inParams' \in \mathcal{N}(\boldA) \} \,,
\end{equation}
where $\inParams^\star$ is any specific minimizer of $g$ and $\mathcal{N}(\boldA)$ denotes the nullspace of $\boldA$.
The closed-form solution for the minimizer of $g$ which minimizes the $L_2$ distance from some initialization $\inParams_0$ is:
\begin{equation}
\inParams^\star = - \boldA^+ \boldb + (\boldI - \boldA^+ \boldA) \inParams_0\,.
\end{equation}
Note that $(\boldI - \boldA^+ \boldA) \inParams_0$ is in the nullspace of $\boldA$, since $\boldA (\boldI - \boldA^+ \boldA) \inParams_0 = \boldA \inParams_0 - \boldA \boldA^+ \boldA \inParams_0 = \boldA \inParams_0 - \boldA \inParams_0 = \boldzero$.
Next, we derive a closed-form expression for the result of $k$ steps of gradient descent with a fixed learning rate $\alpha$.
We write the recurrence:
\begin{align}
    \inParams_{k+1} &= \inParams_k - \alpha \nabla_{\inParams} g(\inParams_k) \\
            &= \inParams_k - \alpha (\boldA \inParams_k + \boldb)
\end{align}
We can subtract the optimum from both sides, yielding:
\begin{align}
    \inParams_{k+1} + \boldA^+ \boldb &= \inParams_k - \alpha (\boldA \inParams_k + \boldb) + \boldA^+ \boldb \\
                      &= \inParams_k - \alpha \boldA \inParams_k - \alpha \boldb + \boldA^+ \boldb \\
                      &= \inParams_k - \alpha \boldA \inParams_k - \alpha \boldA \boldA^+ \boldb + \boldA^+ \boldb \\
                      &= (\boldI - \alpha \boldA) \inParams_k + \boldA^+ \boldb - \alpha \boldA \boldA^+ \boldb \\
                      &= (\boldI - \alpha \boldA) \inParams_k + (\boldI - \alpha \boldA)\boldA^+ \boldb \\
                      &= (\boldI - \alpha \boldA) (\inParams_k + \boldA^+ \boldb)
\end{align}
Thus, to obtain $\inParams_{k+1} + \boldA^+ \boldb$, we simply multiply $\inParams_k + \boldA^+ \boldb$ by $(\boldI - \alpha \boldA)$.
This allows us to write $\inParams_{k+1}$ as a function of the initialization $\inParams_0$:
\begin{align}
    \inParams_{k+1} + \boldA^+ \boldb &= (\boldI - \alpha \boldA)^k (\inParams_0 + \boldA^+ \boldb) \\
    \inParams_{k+1} &= - \boldA^+ \boldb + (\boldI - \alpha \boldA)^k (\inParams_0 + \boldA^+ \boldb)\\
    & =  - \boldA^+ \boldb + (\boldI - \alpha \boldA)^k ((\boldI - \boldA^+ \boldA)\inParams_0 + \boldA^+ \boldA \inParams_0 + \boldA^+ \boldb)\\
    & = - \boldA^+ \boldb + (\boldI - \alpha \boldA)^k ((\boldI - \boldA^+ \boldA)\inParams_0)  + (\boldI - \alpha \boldA)^k (\boldA^+ \boldA \inParams_0 + \boldA^+ \boldb)\\
    &= - \boldA^+ \boldb + (\boldI - \boldA^+ \boldA)\inParams_0 \,, \label{eq:min-norm-gd}
\end{align}
where the last equation follows from:
\begin{align}
    (\boldI - \alpha \boldA) (\boldI - \boldA^+ \boldA)\inParams_0 = (\boldI - \alpha \boldA -  \boldA^+\boldA + \alpha\boldA\boldA^+\boldA)\inParams_0 =  (\boldI - \boldA^+\boldA) \inParams_0   
\end{align}
and thus $(\boldI - \alpha \boldA)^k (\boldI - \boldA^+ \boldA)\inParams_0 = (\boldI - \boldA^+\boldA) \inParams_0$,
and the term $((\boldI - \boldA^+ \boldA)\inParams_0)  + (\boldI - \alpha \boldA)^k (\boldA^+ \boldA \inParams_0 + \boldA^+ \boldb$ goes to 0 as $k \to \infty$ because $\boldA^+ \boldA \inParams_0$ and $\boldA^+ \boldb$ are in the span of $\boldA$.
Gradient descent will converge for learning rates $\alpha < 2 \lambda^{-1}_{\text{max}}$, where $\lambda_{\text{max}}$ is the maximum eigenvalue of $\boldA$.
Thus, from Eq.~\ref{eq:min-norm-gd} we see that gradient descent on the quadratic converges to the solution which minimizes the $L_2$ distance to the initialization $\inParams_0$.
\end{proof}

We also prove the following Lemma~\ref{lemma:psd}, which we use in Theorem~\ref{thm:cold-start-equilibrium}.

\begin{tcolorbox}[colback=tabblue!15,boxrule=0pt,colframe=white,coltext=black,arc=2pt,outer arc=0pt,valign=center]
\begin{lemma}[]
\label{lemma:psd}
  Suppose $f$ and $F$ satisfy Assumption~\ref{assumption:quadratic}.
  Then, the function $F^\star(\outParams) \triangleq F(\outParams, \inParams^\star)$, where $\inParams^\star$ is the minimum-displacement inner solution from $\inParams_0 = \boldzero$, is a convex quadratic in $\outParams$ with a positive semi-definite curvature matrix.
\end{lemma}
\end{tcolorbox}
\begin{proof}
    The min-norm solution to the inner optimization problem for a given outer parameter $\outParams$, $\argmin_{\inParams} f(\outParams, \inParams)$, can be expressed as:
  \begin{equation}
  \inParams^\star = -\boldA^+ (\boldB \outParams + \boldd) \,,
  \end{equation}
  where $\boldA^+$ is the Moore-Penrose pseudoinverse of $\boldA$.
  Plugging this inner solution into the outer objective $F$, we have:
  \begin{align}
      F(\outParams, \inParams^\star)
      &=
      \frac{1}{2} {\inParams^\star}^\top \boldP \inParams^\star + \boldf^\top \inParams^\star + h \\
      &=
      \frac{1}{2} \pp{-\boldA^+ \boldB \outParams - \boldA^+ \boldd}^\top \boldP \pp{- \boldA^+ \boldB \outParams - \boldA^+ \boldd} + \boldf^\top (- \boldA^+ \boldB \outParams - \boldA^+ \boldd) + h \\
      &=
      \frac{1}{2} \outParams^\top \underbrace{\boldB^\top \boldA^+ \boldP \boldA^+ \boldB}_{\text{PSD}} \outParams + \outParams^\top (\boldB^\top \boldA^+ \boldP \boldA^+ \boldd - \boldB^\top \boldA^+ \boldf) \\
      & \quad + \pp{\frac{1}{2} \boldd^\top \boldA^+ \boldP \boldA^+ \boldd - \boldf^\top \boldA^+ \boldd + h}\,.
  \end{align}
  The final equation is a quadratic form in $\outParams$; we wish to show that the curvature matrix, $\boldB^\top \boldA^+ \boldP \boldA^+ \boldB$, is positive semi-definite.
  Note that $\boldA^+ \boldP \boldA^+$ is PSD because $\boldP$ is PSD by assumption, and thus for any vector $\boldv$ we have:
  \begin{equation}
      \boldv^\top (\boldA^+ \boldP \boldA^+) \boldv 
      =
      \boldv^\top (\boldA^+)^\top \boldP \boldA^+ \boldv
      =
      \pp{\boldA^+ \boldv}^\top \boldP \pp{\boldA^+ \boldv} \geq 0 \,.
  \end{equation}
  Next, because $\boldA^+ \boldP \boldA^+$ is a PSD matrix, it can be expressed in the form $\boldM^\top \boldM$ for some PSD matrix $\boldM$ (e.g., the matrix square root of $\boldA^+ \boldP \boldA^+$).
  Then, for any vector $\outParams$, we have:
  \begin{equation}
      \outParams^\top \boldB^\top \boldA^+ \boldP \boldA^+ \boldB \outParams
      =
      \outParams^\top \boldB^\top \boldM^\top \boldM \boldB \outParams
      =
      \pp{\boldM \boldB \outParams}^\top \pp{\boldM \boldB \outParams}
      =
      \norm{\boldM \boldB \outParams}^2_2
      \geq
      0
  \end{equation}
  Thus, $\boldB^\top \boldA^+ \boldP \boldA^+ \boldB$ is PSD.
\end{proof}

\subsection{Proof of Statement~\ref{thm:cold-start-equilibrium}}
\label{app:cold-start-equilibrium-proof}

\begin{tcolorbox}[colback=tabblue!15,boxrule=0pt,colframe=white,coltext=black,arc=2pt,outer arc=0pt,valign=center]
\begin{statement}[{{\color{cyan}Cold-start BLO converges to a cold-start equilibrium.}}]
    Suppose $f$ and $F$ satisfy Assumption~\ref{assumption:quadratic}, and assume that the inner parameters are initialized at $\inParams_0$.
    Then, given appropriate learning rates for the inner and outer optimizations, the cold-start algorithm (Algorithm~\ref{alg:cold-start}) using exact hypergradients converges to a \textit{cold-start equilibrium}.
\end{statement}
\end{tcolorbox}
\begin{proof}
    By assumption, the inner objective $f$ is a convex quadratic in $\inParams$ for each fixed  $\outParams$, with positive semi-definite curvature matrix $\boldA$.
    By Lemma~\ref{lemma:gd}, with an appropriately-chosen learning rate, the iterates of gradient descent in the inner loop of Algorithm~\ref{alg:cold-start} converge to the solution with minimum $L_2$ norm from the inner initialization $\inParams_0$:
    $\inParams_{k+1}^* = \argmin_{\inParams \in \cS(\outParams_k)} \frac{1}{2} \norm{\inParams - \inParams_0}^2$.
    Because $\inParams_{k+1}^* \in \argmin_{\inParams} f(\inParams, \outParams_k)$ and assuming that we compute the exact hypergradient, each outer step performs gradient descent on the objective $F^\star(\outParams) \equiv F(\outParams, \inParams^\star)$.
    By Lemma~\ref{lemma:psd}, the outer objective $F^\star(\outParams)$ is quadratic in $\outParams$ 
    with a PSD curvature matrix, so that with an appropriate outer learning rate, the outer loop of Algorithm~\ref{alg:cold-start} will converge to a solution of $F^\star(\outParams)$.
    Thus, we will have a final iterate $\outParams^\star \in \argmin_{\outParams} F(\outParams, \inParams^\star)$ for which the corresponding inner solution is $\inParams^\star \in \argmin_{\inParams \in \cS(\outParams^\star)} \frac{1}{2} \norm{\inParams^\star - \inParams_0}^2$, yielding a pair of outer and inner parameters $(\outParams^\star, \inParams^\star)$ that are a cold-start equilibrium.
\end{proof}

\subsection{Proof of Theorem~\ref{thm:min-norm-outer}}
\label{app:proof-min-norm-outer}

\begin{tcolorbox}[colback=tabblue!15,boxrule=0pt,colframe=white,coltext=black,arc=2pt,outer arc=0pt,valign=center]
\begin{theorem}[Cold-Start Outer Parameter Norm.]
Suppose $f$ and $F$ satisfy Assumption~\ref{assumption:quadratic}, and suppose we run cold-start BLO (Algorithm~\ref{alg:cold-start}) using the exact hypergradient, starting from outer parameter initialization $\outParams_0$.
Assume that for each outer iteration, the \textit{inner parameters} are re-initialized to $\inParams_0 = \boldzero$ and optimized with an appropriate learning rate to convergence.
Then cold-start BLO---with an appropriate learning rate for the outer optimization---converges to an outer solution $\outParams^\star$ with minimum $L_2$ distance from $\outParams_0$:
\begin{equation}
\outParams^\star = \argmin_{\outParams \in \argmin_{\outParams} F^\star(\outParams)} \frac{1}{2} \norm{\outParams - \outParams_0}^2
\end{equation}
where we define $F^\star(\outParams) \triangleq F(\outParams, \inParams^\star_{\outParams})$, where $\inParams^\star_{\outParams}$ is the minimum-displacement inner solution from $\inParams_0$, that is, $\inParams^\star_{\outParams} = \argmin_{\inParams \in \cS(\outParams)} \| \inParams - \inParams_0 \|^2$.
\end{theorem}
\end{tcolorbox}
\begin{proof}
  Because the inner parameters are initialized at $\inParams_0 = \boldzero$, the solution to the inner optimization problem found by gradient descent for a given outer parameter $\outParams$, $\argmin_{\inParams} f(\outParams, \inParams)$, can be expressed in closed-form as:
  \begin{equation}
    \inParams^\star = -\boldA^+ (\boldB \outParams + \boldd)\,,
  \end{equation}
  where $\boldA^+$ denotes the Moore-Penrose pseudoinverse of $\boldA$.
  Plugging this min-norm inner solution into the outer objective $F$, we have:
  \begin{align}
      F^\star(\outParams) \equiv F(\outParams, \inParams^\star)
      &=
      \frac{1}{2} {\inParams^\star}^\top \boldP \inParams^\star + \boldf^\top \inParams^\star + h \\
      &=
      \frac{1}{2} \outParams^\top \underbrace{\boldB^\top \boldA^+ \boldP \boldA^+ \boldB}_{\text{PSD}} \outParams + \outParams^\top (\boldB^\top \boldA^+ \boldP \boldA^+ \boldd - \boldB^\top \boldA^+ \boldf)
      + \pp{\frac{1}{2} \boldd^\top \boldA^+ \boldP \boldA^+ \boldd - \boldf^\top \boldA^+ \boldd + h}
  \end{align}
  Then $F^\star(\outParams)$ is quadratic in $\outParams$, and by Lemma~\ref{lemma:psd}, the curvature matrix $\boldB^\top \boldA^+ \boldP \boldA^+ \boldB$ is positive semi-definite.
  Let $\boldZ \equiv  \boldB^\top \boldA^+ \boldP \boldA^+ \boldB$.
  Similarly to the analysis in Lemma~\ref{lemma:gd}, the iterates $\outParams_k$ of gradient descent with learning rate $\alpha$ are given by:
  \begin{align} \label{eq:gd-iterates}
      \outParams_k = \outParams^\star + (\boldI - \alpha \boldZ)^k (\outParams_0 - \outParams^\star)\,,
  \end{align}
  where
  \begin{align}
      \outParams^\star = \argmin_{\outParams \in \argmin_{\outParams} F^\star(\outParams)} \frac{1}{2} \norm{\outParams - \outParams_0}^2\,.
  \end{align}
  From Eq.~\eqref{eq:gd-iterates}, we see that the iterates of gradient descent converge exponentially to $\outParams^\star$, which is the outer solution with minimum $L_2$ distance from the outer initialization $\outParams_0$.
\end{proof}

\subsection{Proof of Remark~\ref{remark:cold-full-warm}}
\label{app:cold-full-warm-equiv-proof}

\begin{tcolorbox}[colback=tabblue!15,boxrule=0pt,colframe=white,coltext=black,arc=2pt,outer arc=0pt,valign=center]
\begin{remark}[Equivalence of Full Warm-Start and Cold-Start in the Strongly Convex Regime]
    When the inner problem $f(\outParams, \inParams)$ is strongly convex in $\inParams$ for each $\outParams$, then the solution to the inner problem is unique. In this case, {\color{red} full warm-start} (Algorithm~\ref{alg:warm-start} with $T \to \infty$) and {\color{cyan} cold-start} (Algorithm~\ref{alg:cold-start}), using exact hypergradients, are equivalent.
\end{remark}
\end{tcolorbox}
\begin{proof}
  If $f(\outParams, \inParams)$ is strongly convex in $\inParams$ for each $\outParams$, then it has a unique global minimum for each $\outParams$. Thus, given an appropriate learning rate for the inner optimization, repeated application of the update $\Xi$ will converge to this unique solution for any inner parameter initialization. That is, $\Xi^{(\infty)}(\outParams, \inParams_{\text{init}}) = \argmin_{\inParams} f(\outParams, \inParams)$ for any initialization $\inParams_{\text{init}} \in \inSpace$. In particular, the fixpoint will be identical for cold-start and full warm-start, $\Xi^{(\infty)}(\outParams, \inParams_0) = \Xi^{(\infty)}(\outParams, \inParams_k)$ for any $\inParams_0$ and $\inParams_k$. Therefore, the inner solutions are identical, and yield identical hypergradients, so the iterates of the full warm-start and cold-start algorithms are equivalent.
\end{proof}

\subsection{Proof of Statement~\ref{thm:inclusion}}
\label{app:inclusion-proof}

\begin{tcolorbox}[colback=tabblue!15,boxrule=0pt,colframe=white,coltext=black,arc=2pt,outer arc=0pt,valign=center]
\begin{statement}[Inclusion of Partial Warm-Start Equilibria]
Every partial warm-start equilibrium (with $T=1$) is a full warm-start equilibrium ($T \to \infty$).
In addition, if $\Xi(\outParams, \inParams) = \inParams - \alpha \nabla_{\inParams} f(\outParams, \inParams)$ with a fixed (non-decayed) step size $\alpha$, then the corresponding full-warm start equilibria are also partial warm-start equilibria.
\end{statement}
\end{tcolorbox}
\begin{proof}
    Let $(\outParams^\star, \inParams^\star)$ be an arbitrary partial warm-start equilibrium for $T=1$.
    By definition, $\outParams^\star \in \argmin_{\outParams} F(\outParams, \inParams^\star)$ and $\inParams^\star = \Xi^{(1)}(\outParams^\star, \inParams^\star)$.
    Thus, $\nabla_{\inParams} f(\outParams, \inParams) = 0$, which entails that $\inParams^\star = \Xi^{(\infty)}(\outParams^\star, \inParams^\star)$.
    Hence, $(\outParams^\star, \inParams^\star)$ is a full warm-start equilibrium.
    
    Next, let $(\outParams^\star, \inParams^\star)$ be an arbitrary full warm-start equilibrium.
    By definition, this means that $\inParams^\star = \Xi^{(\infty)}(\outParams^\star, \inParams^\star)$.
    By assumption, $\Xi(\outParams, \inParams) = \inParams - \alpha \nabla_{\inParams} f(\outParams, \inParams)$ with a fixed step size $\alpha$.
    The fixed step size combined with the $T \to \infty$ limit excludes the possibility that there is a finite-length cycle such that the inner optimization arrives back to the initial point $\inParams^\star$ after a finite number of gradient steps $T$.
    Thus, we must have $\Xi^{(1)}(\outParams, \inParams) = \inParams$, so $(\outParams^\star, \inParams^\star)$ is a partial warm-start equilibrium.
\end{proof}

\section{Proximal Best-Response}
\label{app:proximal-br}

Consider the proximal objective $\hInObj(\outParams, \inParams) = \inObj(\outParams, \inParams) + \frac{\epsilon}{2} || \inParams - \inParams' ||^2$.
Here, we will treat $\inParams'$ as a constant (e.g., we won't consider its dependence on $\outParams$).
Let $\inParams^\star(\outParams) \in \argmin_\inParams \hInObj(\outParams, \inParams)$ be a fixed point of $\hInObj$.
We want to compute the response Jacobian $\frac{\partial \inParams^\star(\outParams)}{\partial \outParams}$.
Since $\inParams^\star$ is a fixed point, we have:
\begin{align}
    \frac{\partial \hInObj(\outParams, \inParams^\star(\outParams))}{\partial \inParams} &= 0 \\
    \frac{\partial \inObj(\outParams, \inParams^\star(\outParams))}{\partial \inParams} + \epsilon (\inParams^\star(\outParams) - \inParams') &= 0 \\
    \frac{\partial}{\partial \outParams} \frac{\partial \inObj(\outParams, \inParams^\star(\outParams))}{\partial \inParams} + \epsilon \frac{\partial \inParams^\star(\outParams)}{\partial \outParams} &= 0 \\
    \frac{\partial^2 \inObj(\outParams, \inParams^\star(\outParams))}{\partial \outParams \partial \inParams} + \frac{\partial^2 \inObj}{\partial \inParams \partial \inParams^\top} \frac{\partial \inParams^\star(\outParams)}{\partial \outParams} + \epsilon \frac{\partial \inParams^\star(\outParams)}{\partial \outParams} &= 0 \\
    \left( \frac{\partial^2 \inObj}{\partial \inParams \partial \inParams^\top} + \epsilon \boldI \right) \frac{\partial \inParams^\star(\outParams)}{\partial \outParams} &= - \frac{\partial^2 \inObj(\outParams, \inParams^\star(\outParams))}{\partial \outParams \partial \inParams} \\
    \frac{\partial \inParams^\star(\outParams)}{\partial \outParams} &= - \left( \frac{\partial^2 \inObj}{\partial \inParams \partial \inParams^\top} + \epsilon \boldI \right)^{-1} \frac{\partial^2 \inObj}{\partial \outParams \partial \inParams}
\end{align}

\section{Equivalence Between Unrolling and Neumann Hypergradients}
\label{app:unrolling-and-neumann}
In this section, we review the result from ~\cite{lorraine2020optimizing}, which shows that when we are at a converged solution to the inner problem $\inParams^\star \in \cS(\outParams)$, then computing the hypergradient by differentiating through $k$ steps of unrolled gradient descent on the inner objective is equivalent to computing the hypergradient with the $k$-term truncated Neumann series approximation to the inverse Hessian.

In this derivation, the inner and outer objectives are arbitrary---we do not need to assume that they are quadratic.
The SGD recurrence for unrolling the inner optimization is:
\begin{align}
    \inParams_{i+1}
    &=
    \inParams_i(\outParams) - \alpha \frac{\partial f(\inParams_i(\outParams), \outParams)}{\partial \inParams_i(\outParams)}
\end{align}
Then,
\begin{align}
    \frac{\partial \inParams_{i+1}}{\partial \outParams}
    &= \frac{\partial \inParams_i(\outParams)}{\partial \outParams} - \alpha \frac{\partial}{\partial \inParams_i(\outParams)} \left( \frac{\partial f(\inParams_i(\outParams), \outParams)}{\partial \inParams_i(\outParams)} \frac{\partial \inParams_i(\outParams)}{\partial \outParams} + \frac{\partial f(\inParams_i(\outParams), \outParams)}{\partial \outParams} \right) \\
    &= \frac{\partial \inParams_i(\outParams)}{\partial \outParams} - \alpha \frac{\partial^2 f(\inParams_i(\outParams), \outParams)}{\partial \inParams_i(\outParams) \partial \inParams_i(\outParams)} \frac{\partial \inParams_i(\outParams)}{\partial \outParams} - \alpha \frac{\partial^2 f(\inParams_i(\outParams), \outParams)}{\partial \inParams_i(\outParams) \partial \outParams} \\
    &= - \alpha \frac{\partial^2 f(\inParams_i(\outParams), \outParams)}{\partial \inParams_i(\outParams) \partial \outParams} + \left( \boldI - \alpha \frac{\partial^2 f(\inParams_i(\outParams), \outParams)}{\partial \inParams_i(\outParams) \partial \inParams_i(\outParams)} \right) \frac{\partial \inParams_i(\outParams)}{\partial \outParams}
\end{align}

We can similarly expand out $\frac{\partial \inParams_i(\outParams)}{\partial \outParams}$ as:
\begin{equation}
    \frac{\partial \inParams_i(\outParams)}{\partial \outParams}
    = - \alpha \frac{\partial^2 f(\inParams_{i-1}(\outParams), \outParams)}{\partial \inParams_{i-1} \partial \outParams} + \left( \boldI - \alpha \frac{\partial^2 f(\inParams_{i-1}(\outParams), \outParams)}{\partial \inParams_{i-1}(\outParams) \partial \inParams_{i-1}(\outParams)} \right) \frac{\partial \inParams_{i-1}(\outParams)}{\partial \outParams}
\end{equation}
Plugging in this expression for $\frac{\partial \inParams_i(\outParams)}{\partial \outParams}$ into the expression for $\frac{\partial \inParams_{i+1}(\outParams)}{\partial \outParams}$, we have:
\begin{align}
    \frac{\partial \inParams_{i+1}(\outParams)}{\partial \outParams}
    &= - \alpha \frac{\partial^2 f(\inParams_i(\outParams), \outParams)}{\partial \inParams_i(\outParams) \partial \outParams} + \left( \boldI - \alpha \frac{\partial^2 f(\inParams_i(\outParams), \outParams)}{\partial \inParams_i(\outParams) \partial \inParams_i(\outParams)} \right) \\
    & \qquad \times \left[ - \alpha \frac{\partial^2 f(\inParams_{i-1}(\outParams), \outParams)}{\partial \inParams_{i-1}(\outParams) \partial \outParams} + \left( \boldI - \alpha \frac{\partial^2 f(\inParams_{i-1}(\outParams), \outParams)}{\partial \outParams_{i-1}(\outParams) \partial \inParams_{i-1}(\outParams)} \right) \frac{\partial \inParams_{i-1}(\outParams)}{\partial \outParams} \right]
\end{align}
Expanding, we have:
\begin{align}
    \frac{\partial \inParams_{i+1}(\outParams)}{\partial \outParams}
    &= - \alpha \frac{\partial^2 f(\inParams_i(\outParams), \outParams)}{\partial \inParams_i(\outParams) \partial \outParams} + \left( \boldI - \alpha \frac{\partial^2 f(\inParams_i(\outParams), \outParams)}{\partial \inParams_i(\outParams) \partial \inParams_i(\outParams)} \right) \left( - \alpha \frac{\partial^2 f(\inParams_{i-1}(\outParams), \outParams)}{\partial \inParams_{i-1}(\outParams) \partial \outParams} \right) \\
    & \qquad + \prod_{k < j} \left( \boldI - \alpha \frac{\partial^2 f(\inParams, \outParams)}{\partial \inParams \partial \inParams^\top} \bigg \rvert_{\outParams, \inParams_{i - k}(\outParams)} \right) \frac{\partial \inParams_{i-1}(\outParams)}{\partial \outParams}
\end{align}
Telescoping this sum, we have:
\begin{align}
    \frac{\partial \inParams_{i+1}(\outParams)}{\partial \outParams}
    = \sum_{j \leq i} \left[ \prod_{k < j} \left( \boldI - \alpha \frac{\partial^2 f(\inParams, \outParams)}{\partial \inParams \partial \inParams^\top} \bigg \rvert_{\outParams, \inParams_{i - k}(\outParams)} \right) \right] \left( - \alpha \frac{\partial^2 f(\outParams)}{\partial \inParams \partial \outParams} \bigg \rvert_{\outParams, \inParams_{i-j}(\outParams)} \right)
\end{align}

If we start unrolling from a stationary point of the proximal objective, then all the $\inParams_i$ will be equal, so this simplifies to:
\begin{align}
    \frac{\partial \inParams_{i+1}(\outParams)}{\partial \outParams} = \left[ \sum_{j \leq i} \left( \boldI - \alpha \frac{\partial^2 f(\inParams, \outParams)}{\partial \inParams \partial \inParams^\top} \right)^j \right] \left( - \alpha \frac{\partial^2 f(\inParams, \outParams)}{\partial \inParams \partial \outParams} \right)
\end{align}

This recovers the Neumann series approximation to the inverse Hessian.

\section{Experimental Details and Extended Results}
\label{app:exp-details}

\paragraph{Compute Environment.}
All experiments were implemented using JAX~\citep{jax2018github}, and were run on NVIDIA P100 GPUs.
Each instance of the dataset distillation and antidistillation task took approximately 5 minutes of compute on a single GPU.

\subsection{Details and Extended Results for Dataset Distillation.}
\label{app:extended-dataset-distillation}

\paragraph{Details.}
For our dataset distillation experiments, we trained a 4-layer MLP with 200 hidden units per layer and ReLU activations.
For warm-start joint optimization, we computed hypergradients by differentiating through $K=1$ steps of unrolling, and updated the outer parameters (learned datapoints) and MLP parameters using alternating gradient descent, with one step on each.
We used SGD with learning rate 0.001 for the inner optimization and Adam with learning rate 0.01 for the outer optimization.

\paragraph{Link to Animations.}
Here is a \href{https://docs.google.com/document/d/1whS9xmU_gfwsqtpUchf3ltlXSxcA_8iPoMwAfNqXNMs/edit?usp=sharing}{link to an document containing animations} of the bilevel training dynamics for the dataset distillation task.
We visualize the dynamics of warm-started bilevel optimization in the setting where the inner problem is overparameterized, by animating the trajectories of the learned datapoints over time, and showing how the decision boundary of the model changes over the course of joint optimization.

\paragraph{Extended Results.}
Here, we show additional dataset distillation results, using a similar setup to Section~\ref{sec:dataset-distillation}.
Figure~\ref{fig:three-class-distillation} shows the results where we fit a three-class problem (e.g., three concentric rings) using three learned datapoints.
Figure~\ref{fig:three-class-two-datapoints} shows the results for fitting three classes using \textit{only two datapoints}, where both the coordinates and \textit{soft labels} are learned.
For each training datapoint, in addition to learning its $(x, y)$ coordinates, we learn a $C$-dimensional vector (where $C$ is the number of classes, in this example $C=3$) representing the un-normalized class label: this vector passed through a softmax when we perform cross-entropy training of the inner model.
Joint adaptation of the model parameters and learned data is able to fit three classes by changing the learned class label for one of the datapoints during training.

\begin{figure}[htbp]
    \centering
    \begin{subfigure}{.32\textwidth}
      \includegraphics[width=\linewidth]{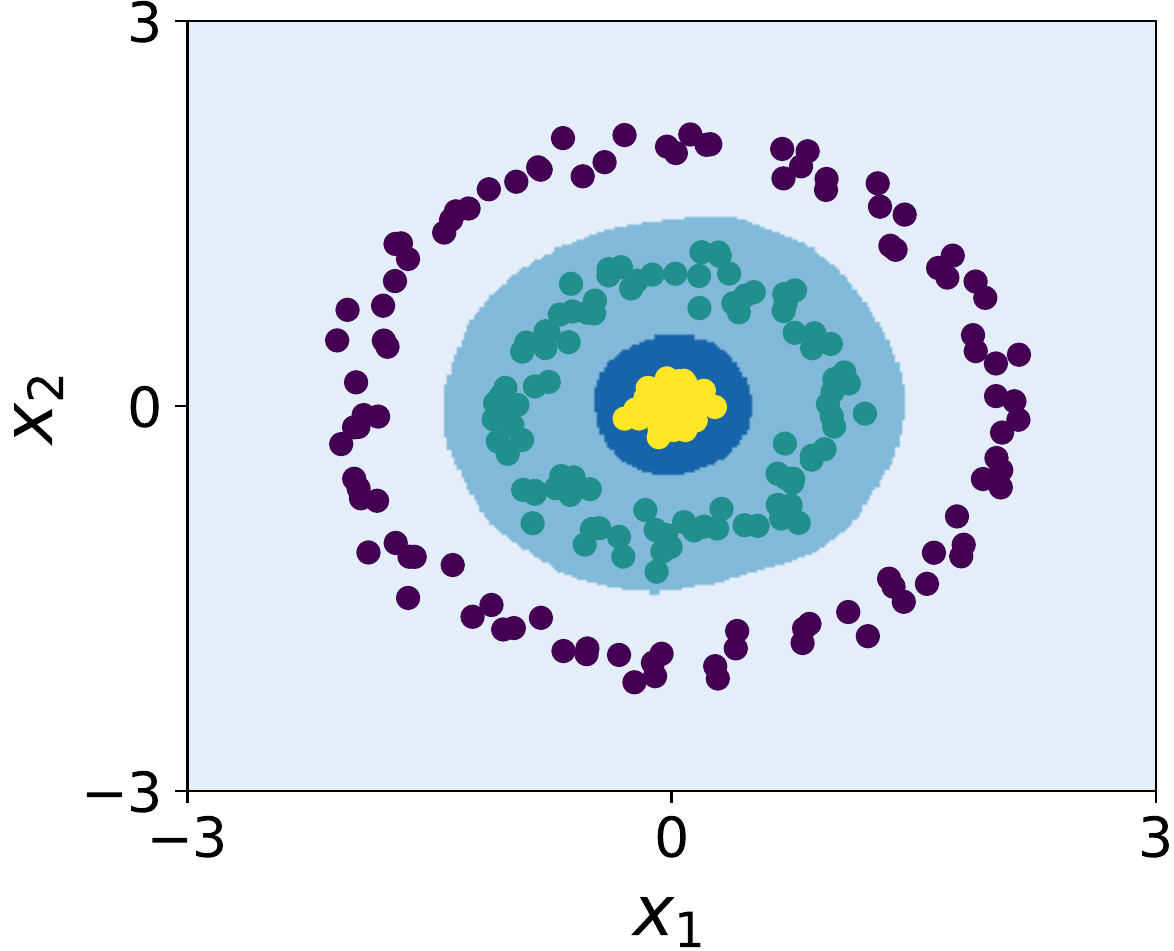}
      \caption{Training on original data.}
      \label{fig:three-class-orig}
    \end{subfigure}
    \hfill
    \begin{subfigure}{.32\textwidth}
      \includegraphics[width=\linewidth]{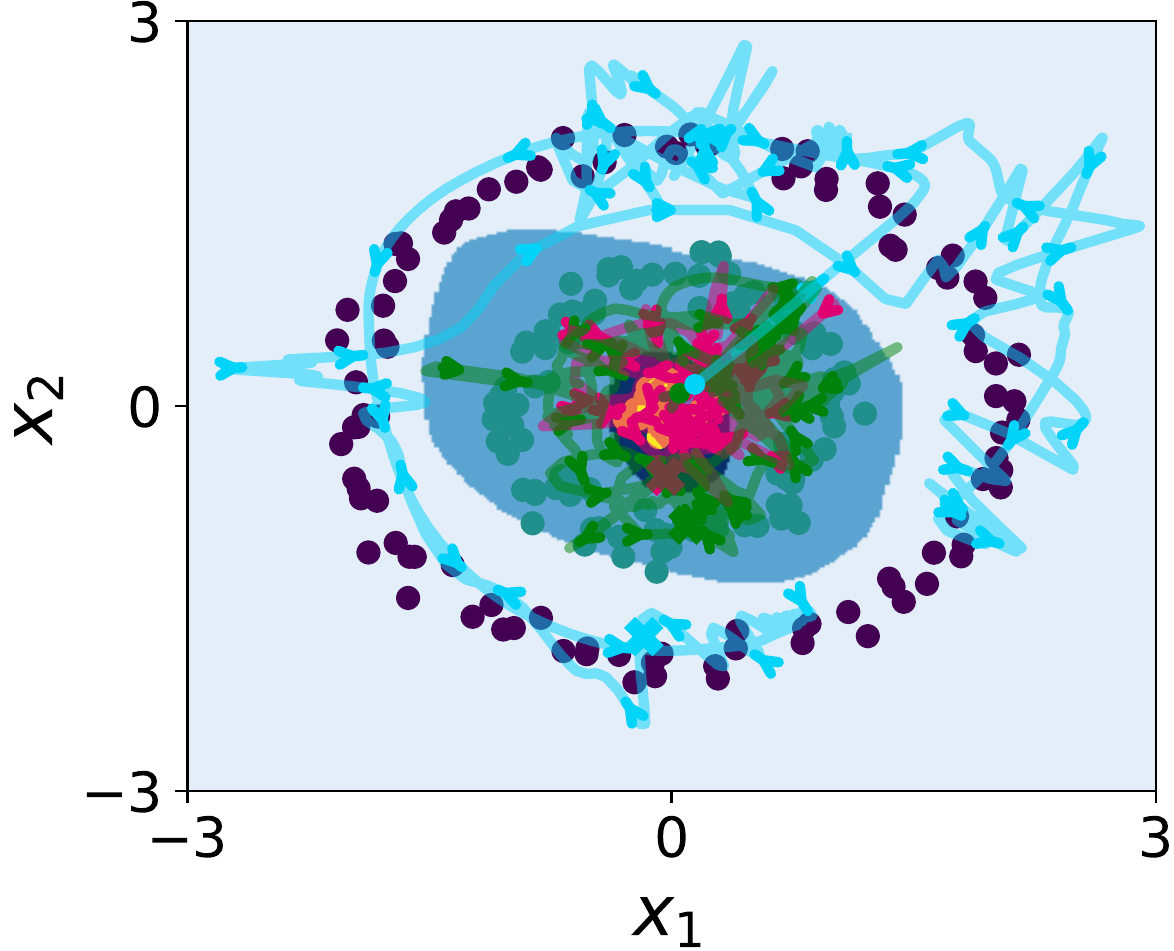}
      \caption{Warm-start joint optimization.}
      \label{fig:three-class-warm}
    \end{subfigure}
    \hfill
    \begin{subfigure}{.32\textwidth}
      \centering
      \includegraphics[width=\linewidth]{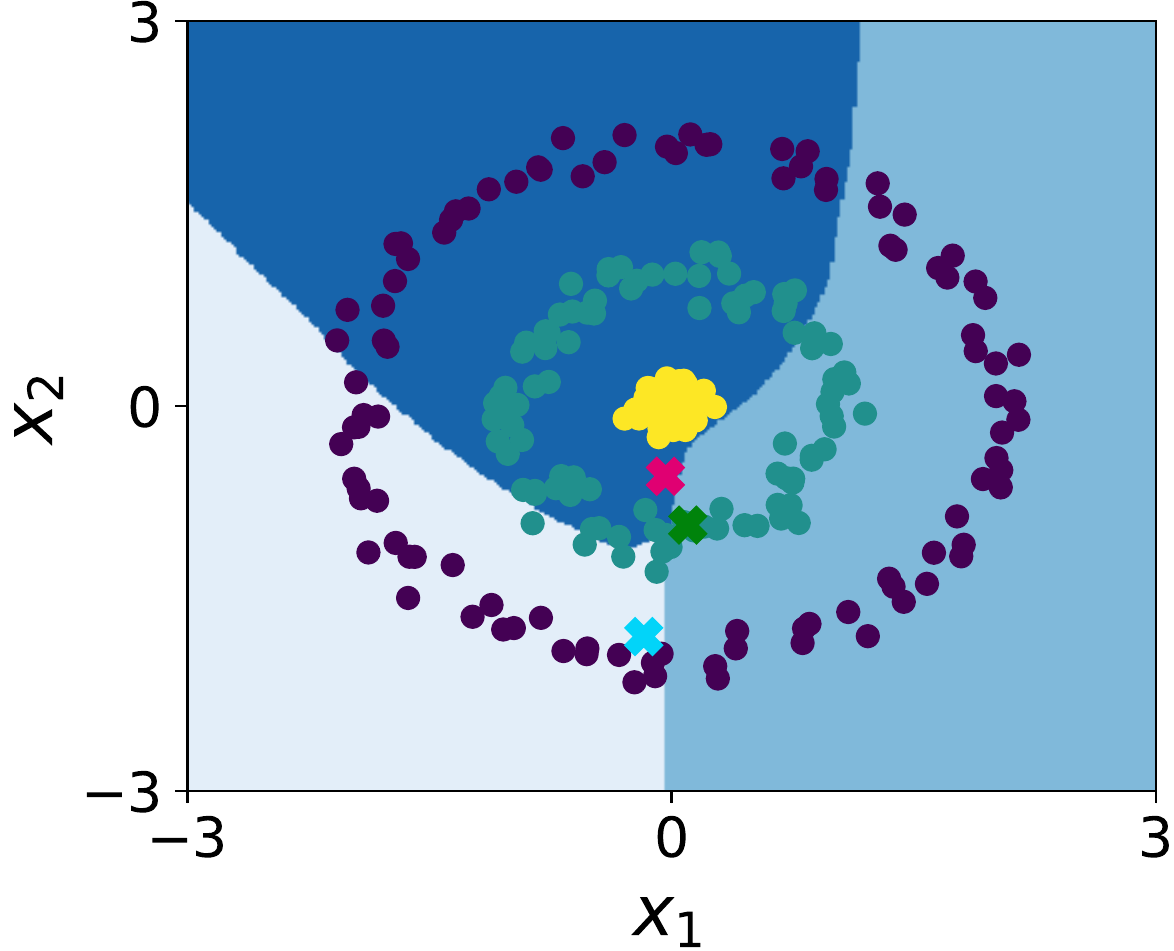}
      \caption{Re-training on final points.}
      \label{fig:three-class-retrain}
    \end{subfigure}
\caption{\small \textbf{Dataset distillation fitting three classes with three learned datapoints.}
Similarly to the results in Sec.~\ref{sec:dataset-distillation}, when using warm-start joint optimization, the three learned datapoints adapt during training to trace out the data in their respective classes, guiding the network to learn a decision boundary that performs well on the original data.
Cold-start re-training yields a model that correctly classifies the three learned datapoints, but has poor performance on the original data.}
\label{fig:three-class-distillation}
\end{figure}

\begin{figure}[htbp]
    \centering
    \begin{subfigure}{.32\textwidth}
      \includegraphics[width=\linewidth]{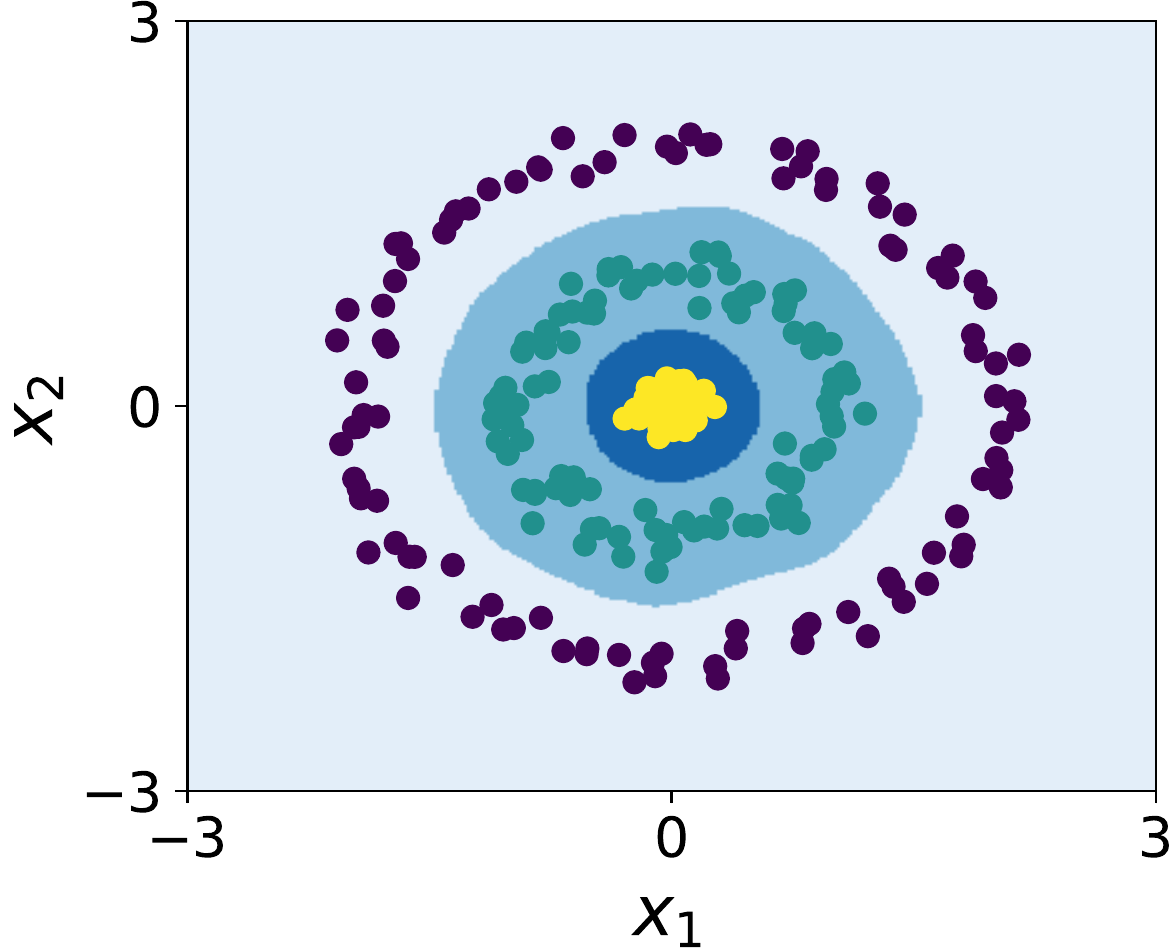}
      \caption{Training on original data.}
      \label{fig:three-class-2-synth-orig}
    \end{subfigure}
    \hfill
    \begin{subfigure}{.32\textwidth}
      \includegraphics[width=\linewidth]{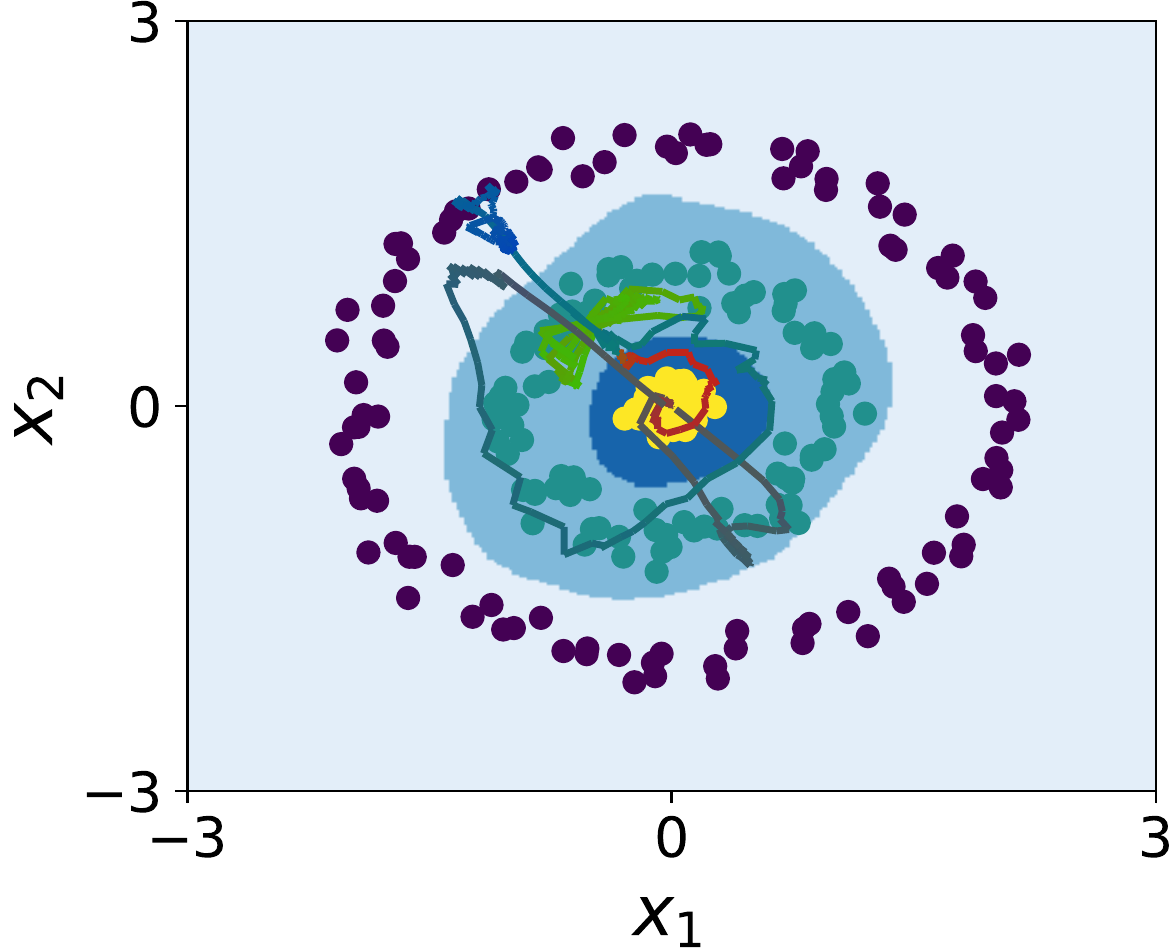}
      \caption{Warm-start joint optimization.}
      \label{fig:three-class-2-synth-warm}
    \end{subfigure}
    \hfill
    \begin{subfigure}{.32\textwidth}
      \centering
      \includegraphics[width=\linewidth]{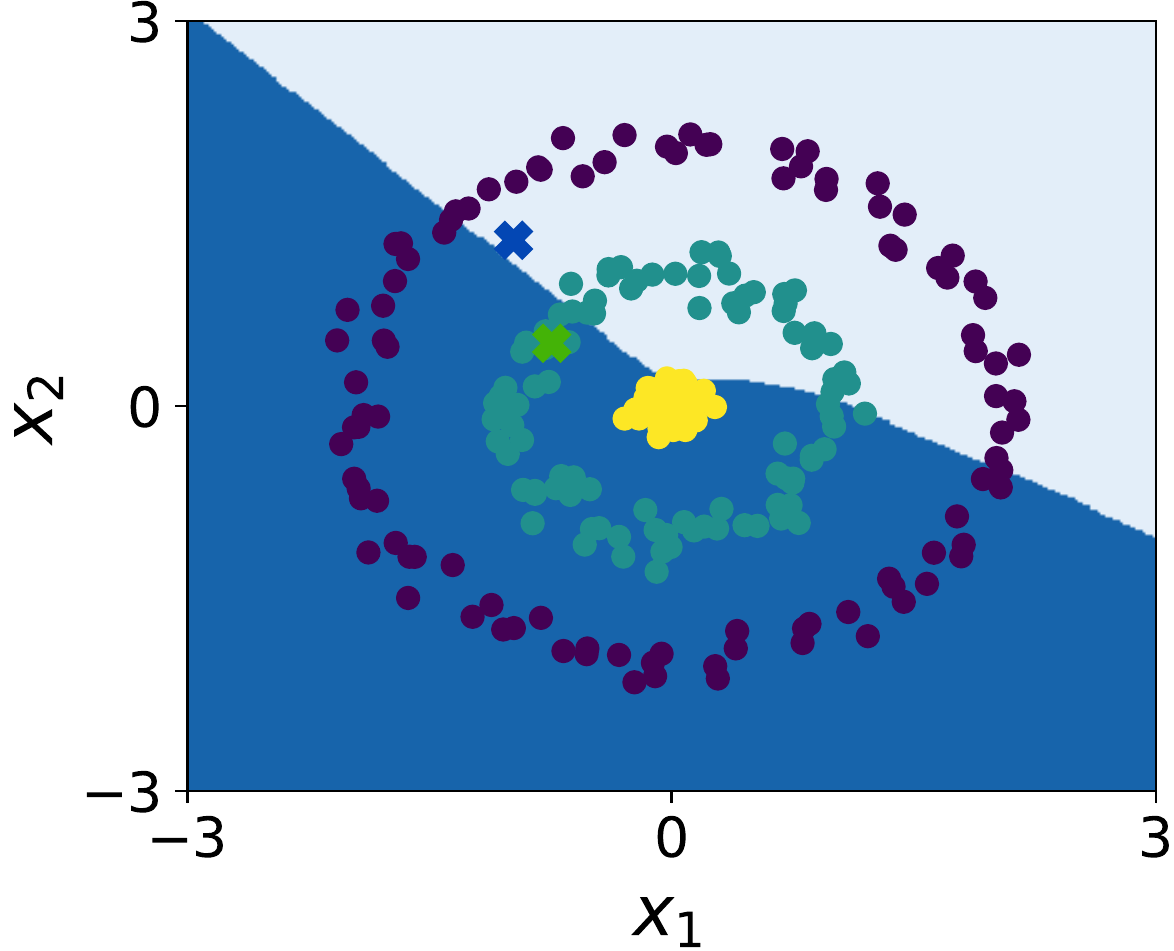}
      \caption{Re-training on final points.}
      \label{fig:three-class-2-synth-retrain}
    \end{subfigure}
\caption{\small \textbf{Dataset distillation results fitting three classes with \textit{only two} learned datapoints.}
In the warm-start plot, the color of the trajectory of each learned datapoint indicates its soft class membership, with magenta, green, and blue corresponding to the inner, middle, and outer rings, respectively.
Darker/gray colors indicate soft labels that place approximately equal probability on each class.
We see that although we only have two learned datapoints, the class labels change over the course of training such that all three classes are covered.
Cold-start re-training yields a model that correctly classifies the two learned datapoints, but has poor performance on the original data.}
\label{fig:three-class-two-datapoints}
\end{figure}

\subsection{Details and Extended Results for Anti-Distillation}
\label{app:extended-anti-distillation}

\begin{figure*}
    \centering
    \includegraphics[width=0.4\linewidth]{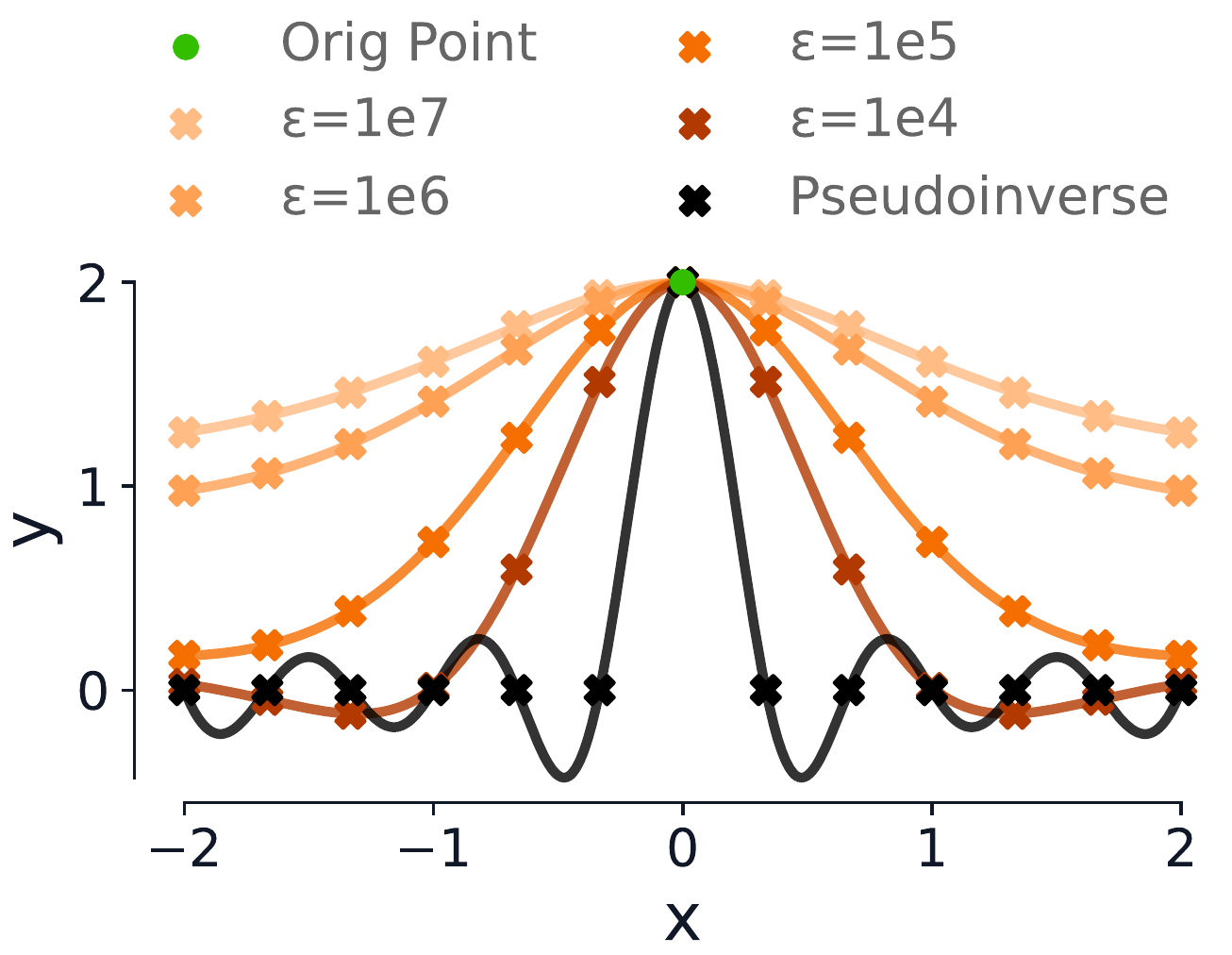}
    \vspace{-0.2cm}
    \caption{\small
    \textbf{Antidistillation task for linear regression with an overparameterized outer objective.} This plot uses the same setup as Figure~\ref{fig:overparam-regression}, but shows learned datapoints obtained using the damped Hessian inverse $(\boldH + \epsilon \boldI)^{-1}$ to compute the approximate hypergradient, for various damping factors $\epsilon$.
    }
    \label{fig:overparam-epsilon}
    \vspace{-0.1cm}
\end{figure*}

\paragraph{Details.}
For the anti-distillation results in Section~\ref{sec:overparam-outer}, we used a Fourier basis consisting of 10 $\sin$ terms, 10 $\cos$ terms, and a bias, yielding 21 total inner parameters.
For the exponential-amplitude Fourier basis used in Section~\ref{sec:overparam-outer}, we used SGD with learning rates 1e-7 and 1e-2 for the inner and outer parameters, respectively; for the $1/n$ amplitude Fourier basis (discussed below, and used for Figure~\ref{fig:fourier-over-k}), we used SGD with learning rates 1e-3 and 1e-2 for the inner and outer parameters, respectively.

\paragraph{Approximate Hypergradient Visualization.}
For the visualization in Figure~\ref{fig:outer-hypergrad-approx}, the experiment setup is as follows: we have a single original dataset point at $(x, y)$ coordinate $(1,1)$, and we learn the $y$-coordinates of two synthetic datapoints, which we initialize at $(x, y)$-coordinates $(0, 0)$ and $(2, 0)$, respectively.
The inner model trained on these two synthetic datapoints is a linear regressor with one weight and one bias, e.g., $y = wx + b$.
The outer problem is overparameterized, because there are many valid settings of the learned datapoints such that the inner model trained on those datapoints will fit the point $(1,1)$.
For example, three valid solutions for the learned datapoints are $\{ (0,0), (2,2) \}$, $\{ (0,1), (2,1) \}$, and $\{ (0,2), (2,0)\}$.
The set of such valid solutions is visualized in Figure~\ref{fig:outer-hypergrad-approx}, as well as the gradients using different hypergradient approximations, outer optimization trajectories, and converged outer solutions with each approximation.
We focused on Neumann series approximations with different $k$, and ran the Neumann series at the converged inner parameters in each case.

\begin{wrapfigure}[13]{r}{0.36\linewidth}
    \vspace{-0.6cm}
    \centering
    \includegraphics[width=\linewidth]{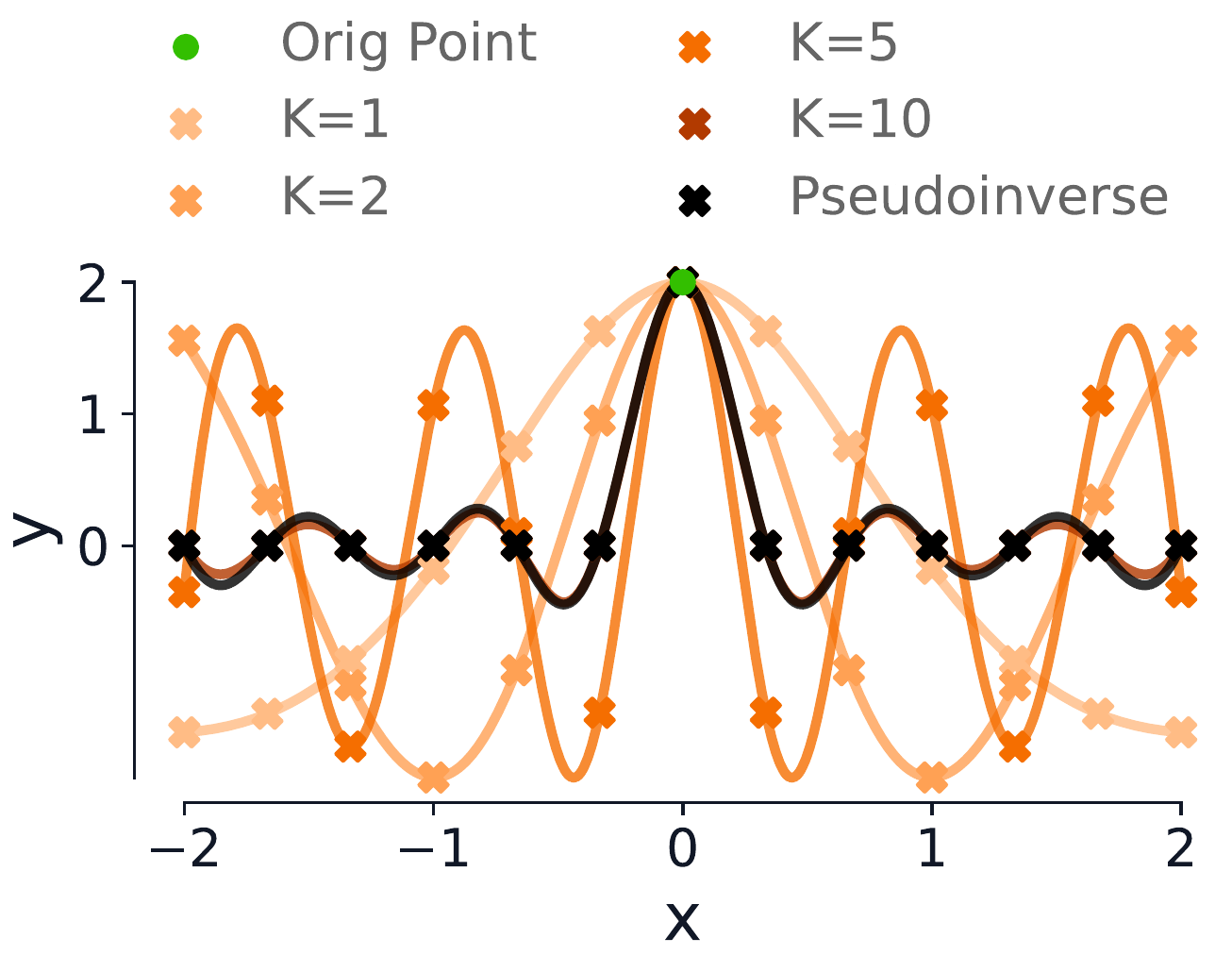}
    \vspace{-0.6cm}
    \caption{\small Truncated CG used to compute hypergradients for the anti-distillation task.}
    \label{fig:fourier-cg}
\end{wrapfigure}
\paragraph{Conjugate Gradient.}
Here, we present results using truncated conjugate gradient to approximate the inverse Hessian vector product when using the implicit function theorem to compute the hypergradient.
We use the same problem setup as in Section~\ref{sec:overparam-outer}, where we have a single validation datapoint and 13 learned datapoints.
Because our Fourier basis consists of 10 $\sin$ terms, 10 $\cos$ terms, and a bias, we have 21 inner parameters, and using 21 steps of CG is guaranteed to yield the true inverse-Hessian vector product (barring any numerical issues).
In Figure~\ref{fig:fourier-cg}, we show the effect of using truncated CG iterations on the learned datapoints: we found that while Neumann iterations, truncated unrolling, and proximal optimization all yield nearly identical results, CG produces a very different inductive bias.

\paragraph{An Alternate Set of Fourier Basis Functions.}
In Section~\ref{sec:overparam-outer}, we used a Fourier basis in which the lower-frequency terms had exponentially larger amplitudes than the high-frequency terms.
Figure~\ref{fig:fourier-over-k} presents results using an alternative feature transform, where lower-frequency terms have linearly larger amplitudes than high-frequency terms.
\begin{align*}
    \boldphi(x)
    =
    \begin{bmatrix}
      1
      &
      \smash{
      \underbrace{
      \begin{matrix}
      \cos(x)
      &
      \frac{1}{2} \cos(2x)
      &
      \cdots
      &
      \frac{1}{L}
      \cos(Lx)
      \end{matrix}}_{L \cos \text{terms}}
      }
      &
      \smash{
      \underbrace{
      \begin{matrix}
      \sin(x)
      &
      \frac{1}{2} \sin(2x)
      &
      \cdots
      &
      \frac{1}{L} \sin(Lx)
      \end{matrix}}_{L \sin \text{terms}}
      }
    \end{bmatrix}
\end{align*}
\vspace{0.1cm}

Thus, the inner problem aims to learn a weight vector $\inParams \in \mathbb{R}^{2N+1}$ that minimizes:
\begin{align}
    f(\outParams, \inParams)
    =
    \frac{1}{2} \| \boldPhi \inParams - \outParams \|^2_2 
    =
    \frac{1}{2} \sum_{i=1}^n \left( \inParams^\top \boldphi(x_i) - u_i \right)^2
\end{align}
where
\begin{align}
    \inParams^\top \boldphi(x) = w_0 + \sum_{\ell=1}^L \left( w_{\ell} \left( \frac{1}{\ell} \right) \cos(\ell x) + w_{L+\ell} \left( \frac{1}{\ell} \right) \sin(\ell x) \right)
\end{align}

\begin{figure}[!htbp]
    \centering
    \begin{subfigure}{.32\textwidth}
    \includegraphics[width=\linewidth]{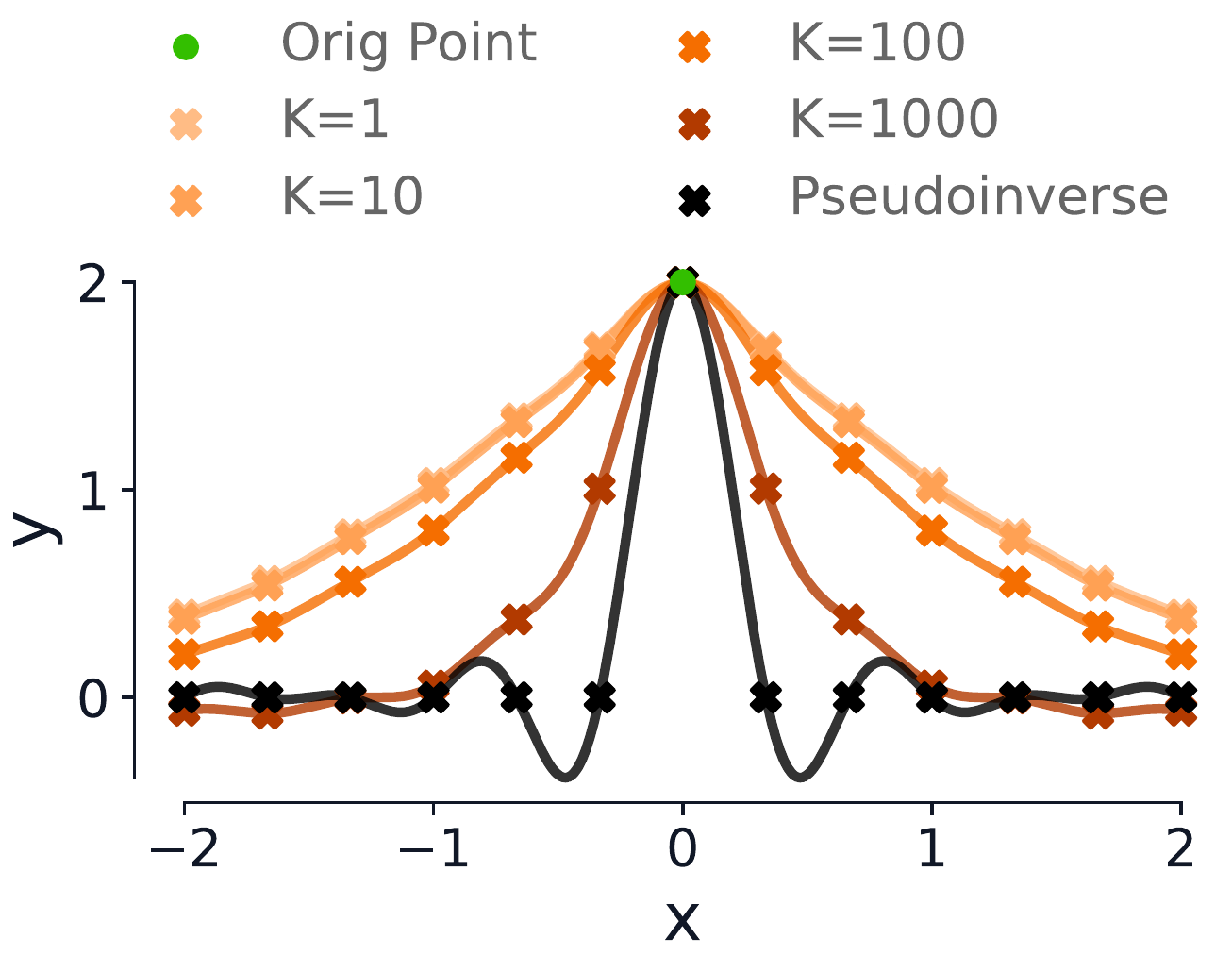}
      \caption{Neumann/unrolling}
      \label{fig:implicit-1-over-k-neumann}
    \end{subfigure}
    \hfill
    \begin{subfigure}{.32\textwidth}
    \includegraphics[width=\linewidth]{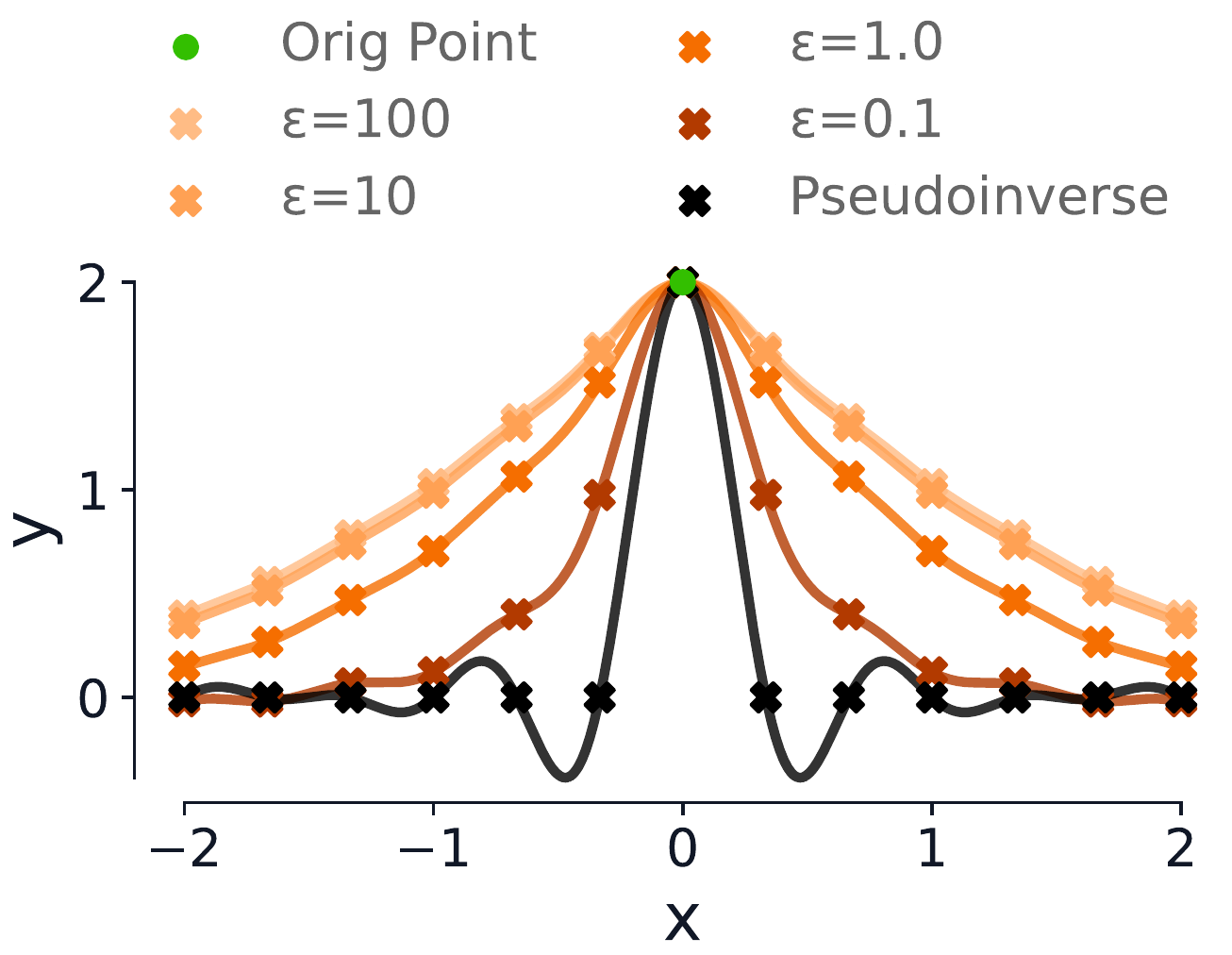}
      \caption{Damped Hessian}
      \label{fig:fourier-1-over-k-damped-H}
    \end{subfigure}
    \hfill
    \begin{subfigure}{.32\textwidth}
    \includegraphics[width=\linewidth]{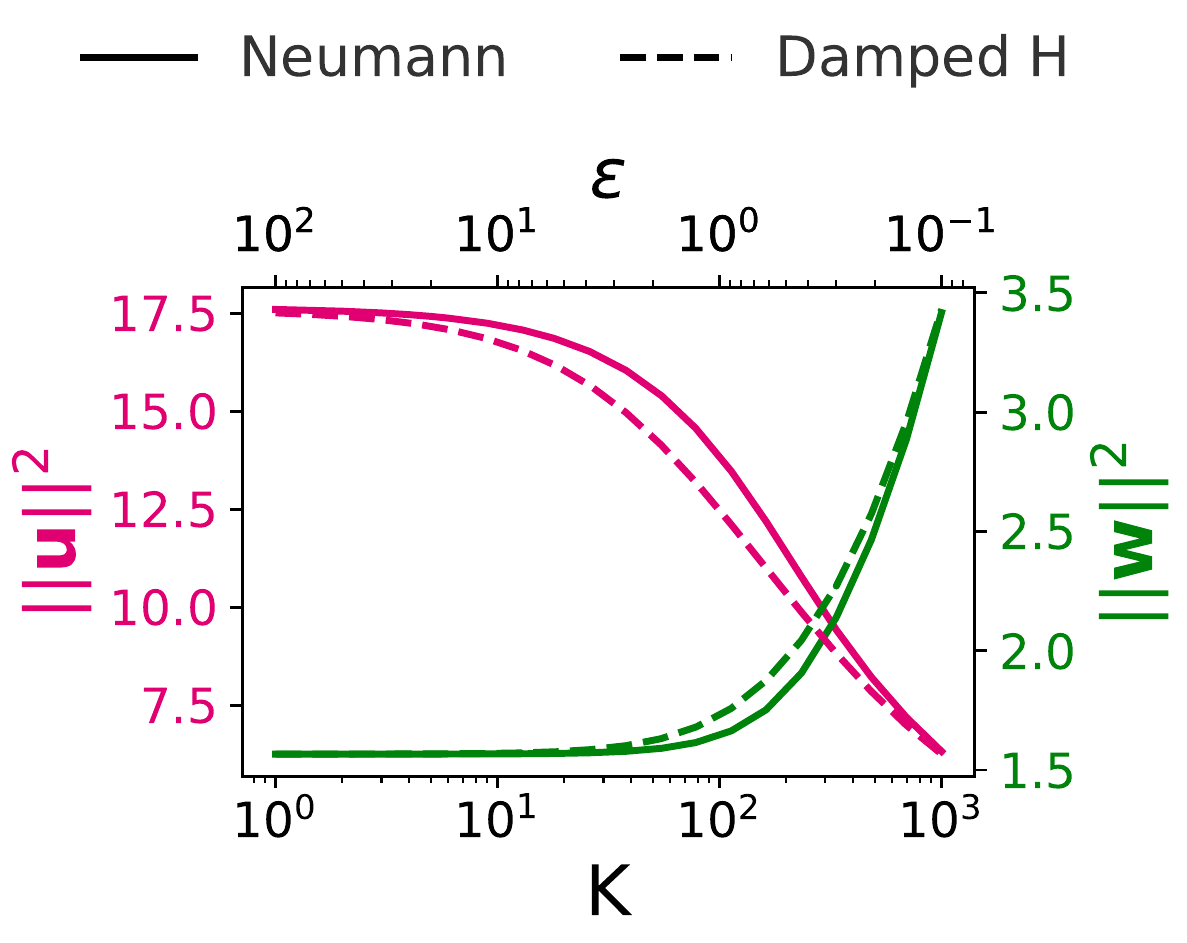}
      \caption{Outer parameter norms.}
      \label{fig:implicit-1-over-k-norms}
    \end{subfigure}
    \vspace{-0.2cm}
    \caption{\small Fourier-basis 1D linear regression, where the outer objective is overparameterized. We learn 13 synthetic datapoints such that a regressor trained on those points will fit a single original datapoint, shown by the {\color{dkgreen}green circle at $(0,2)$}. The synthetic datapoints are initialized at linearly-spaced $x$-coordinates in the interval $[-2, 2]$, with $y$-coordinate 0, and we only learn $y$ coordinates.
    In the Fourier basis we use, lower frequency components have larger amplitudes.
    Here, we use the $1/n$ amplitude scheme.
    }
    \label{fig:fourier-over-k}
\end{figure}

\paragraph{Underparameterized Inner Problem.}
The antidistillation phenomena we describe occur in outer-overparameterized settings; they do not rely on inner overparameterization.
In Figure~\ref{fig:underparam-fourier-over-k}, we show a version of experiment with an underparameterized inner problem (such that there is a unique inner optimum for each outer parameter, and the Hessian of the inner objective is invertible).

\begin{figure}[!htbp]
    \centering
    \begin{subfigure}{.32\textwidth}
    \includegraphics[width=\linewidth]{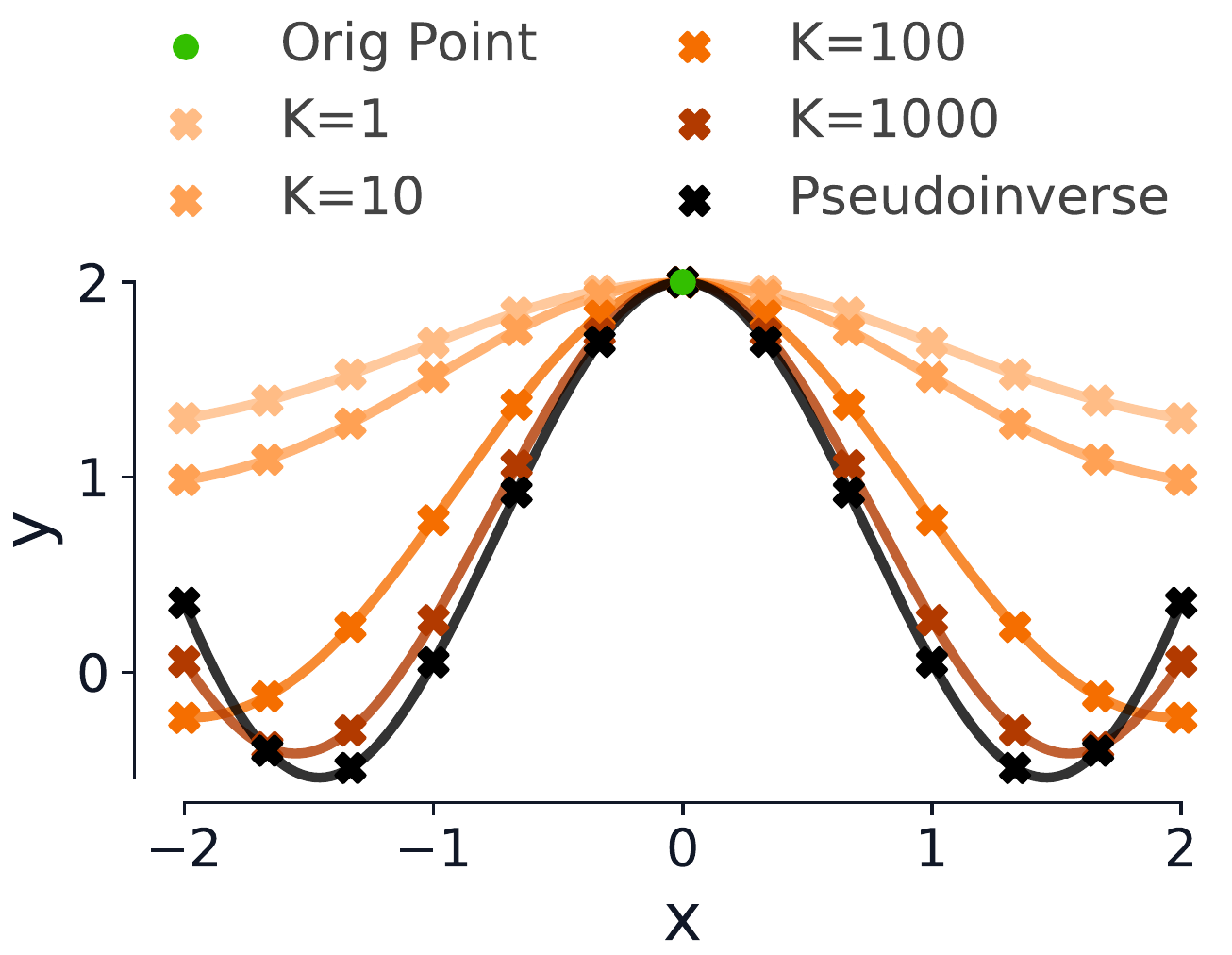}
      \caption{Neumann/unrolling}
      \label{fig:underparam-neumann}
    \end{subfigure}
    \hfill
    \begin{subfigure}{.32\textwidth}
    \includegraphics[width=\linewidth]{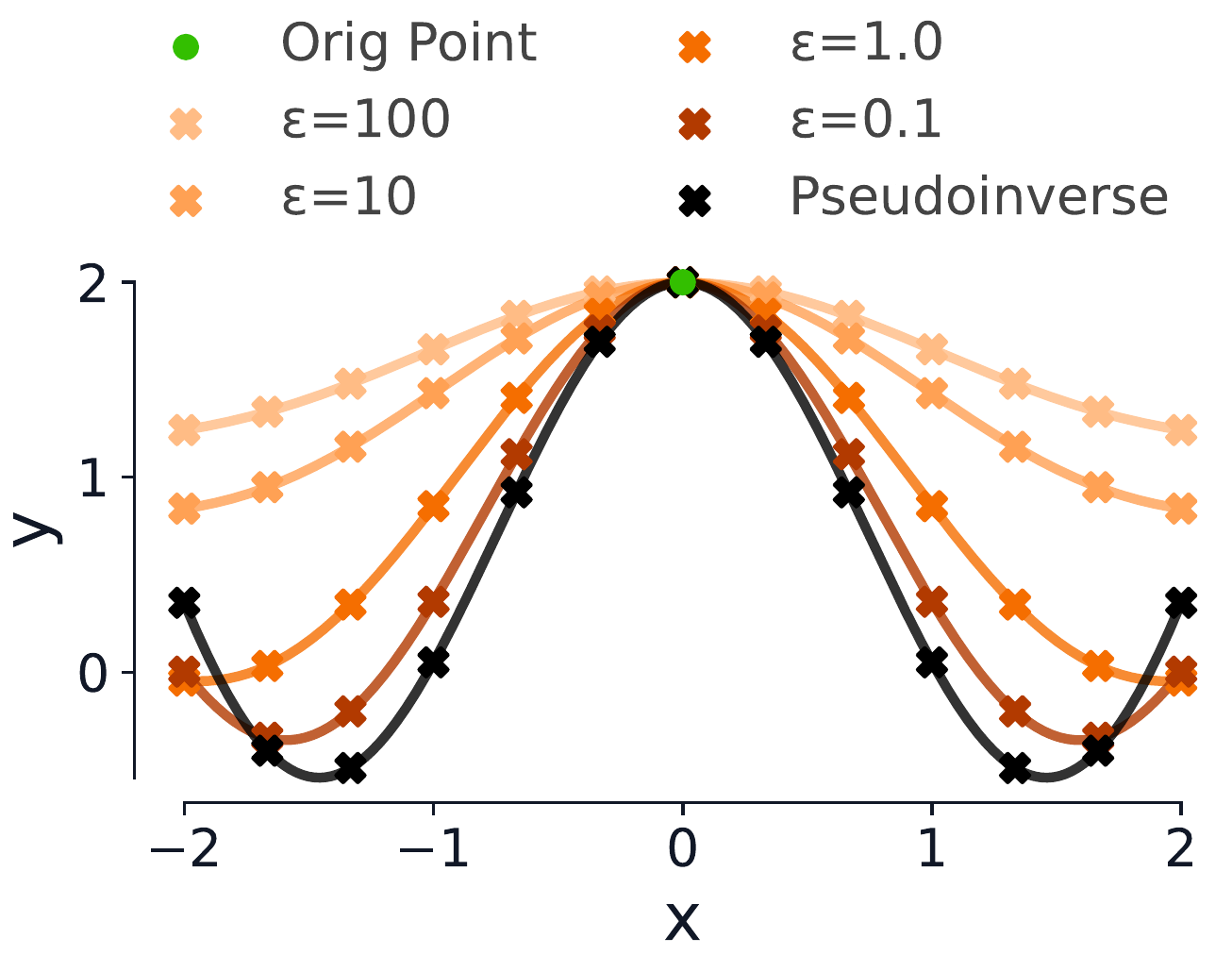}
      \caption{Damped Hessian inverse}
      \label{fig:underparam-damped-H}
    \end{subfigure}
    \hfill
    \begin{subfigure}{.32\textwidth}
    \includegraphics[width=\linewidth]{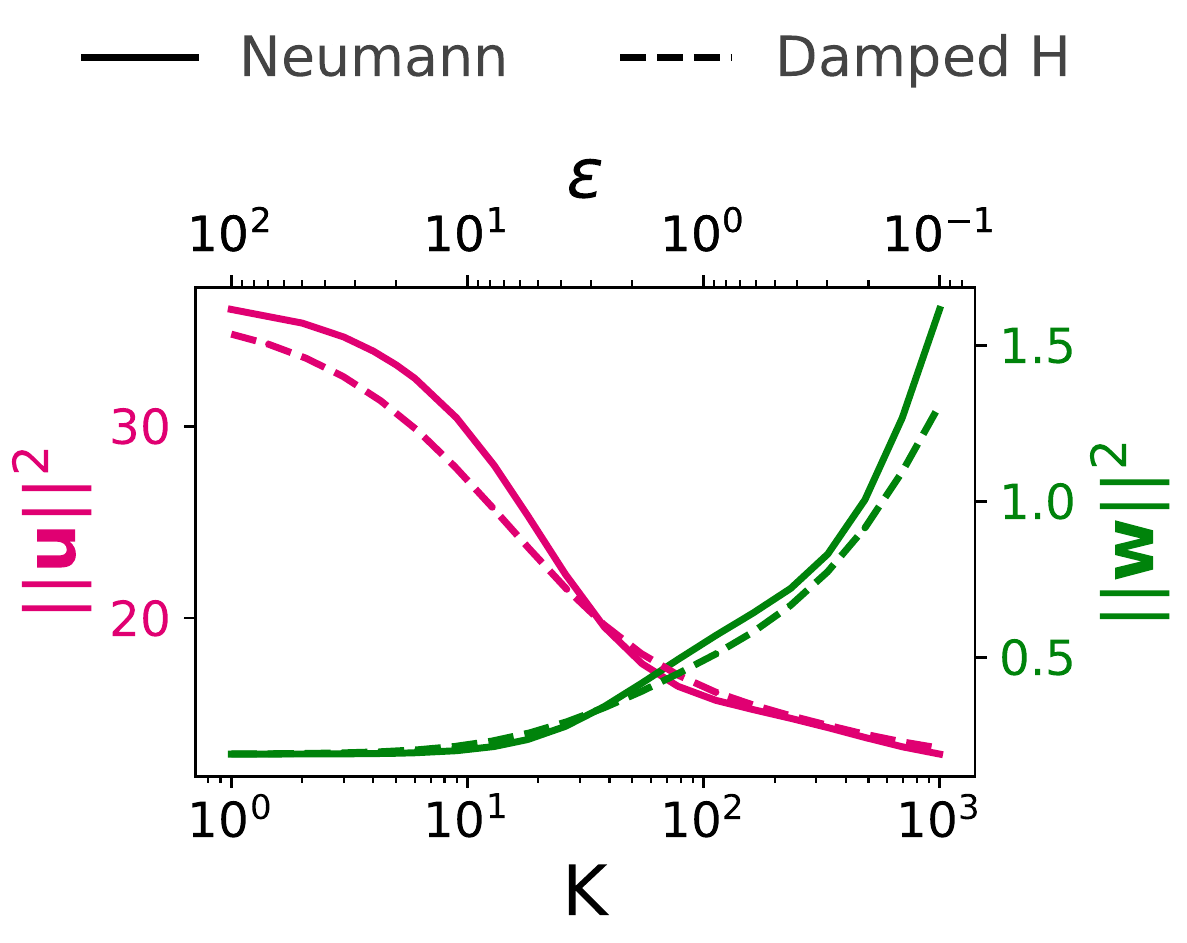}
      \caption{Outer parameter norms.}
      \label{fig:underparam-norms}
    \end{subfigure}
    \caption{\small Fourier-basis 1D linear regression, where the outer objective is overparameterized, \textit{but the inner objective is underparameterized}. We learn 13 synthetic datapoints such that a regressor trained on those points will fit a single original datapoint, shown by the {\color{dkgreen}green circle at $(0,2)$}. The synthetic datapoints are initialized at linearly-spaced $x$-coordinates in the interval $[-2, 2]$, with $y$-coordinate 0, and we only learn the $y$-coordinates.
    In the Fourier basis we use, lower frequency components have larger amplitudes.
    }
    \label{fig:underparam-fourier-over-k}
\end{figure}

\paragraph{Using an MLP.}
Finally, we show that similar conclusions hold when training a multi-layer perceptron (MLP) on the anti-distillation task.
We used a 2-layer MLP with 10 hidden units per layer and ReLU activations.
We used SGD with learning rate 0.01 for both the inner and outer parameters.
Figure~\ref{fig:overparam-mlp} shows the learned datapoints and model fits resulting from running several different steps of Neumann iterations or unrolling, as well as the norms of the inner and outer parameters as a function of $K$.
For the Neumann experiments, we first optimize the MLP for 5000 steps to reach approximate convergence of the inner problem, before running $K$ Neumann iterations---the MLP is re-trained from scratch for 5000 steps for each outer parameter update (e.g., cold-started).

\begin{figure}[!htbp]
    \centering
    \begin{subfigure}{.32\textwidth}
    \includegraphics[width=\linewidth]{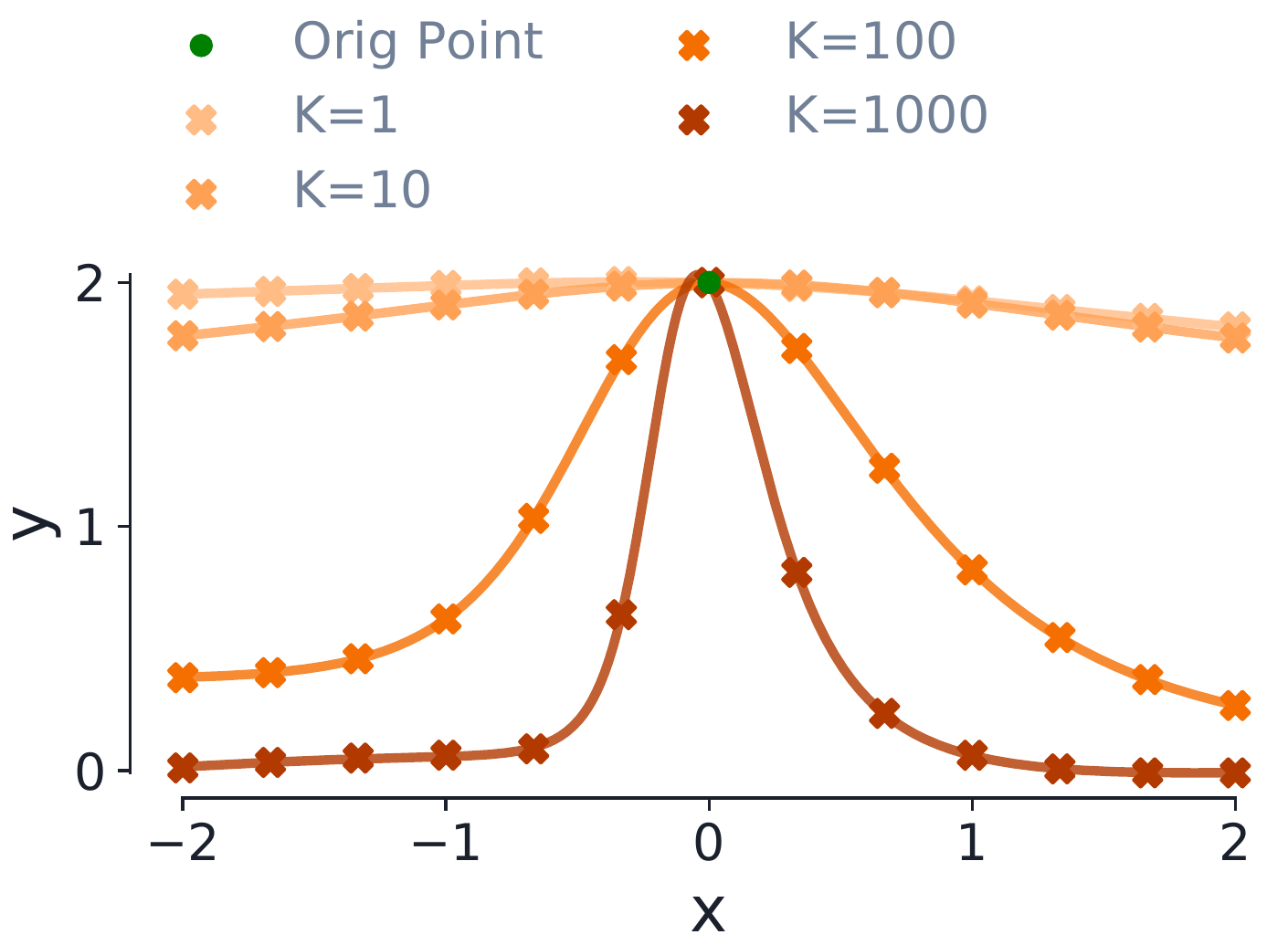}
       \caption{Neumann}
    \end{subfigure}
    \hfill
    \begin{subfigure}{.32\textwidth}
    \includegraphics[width=\linewidth]{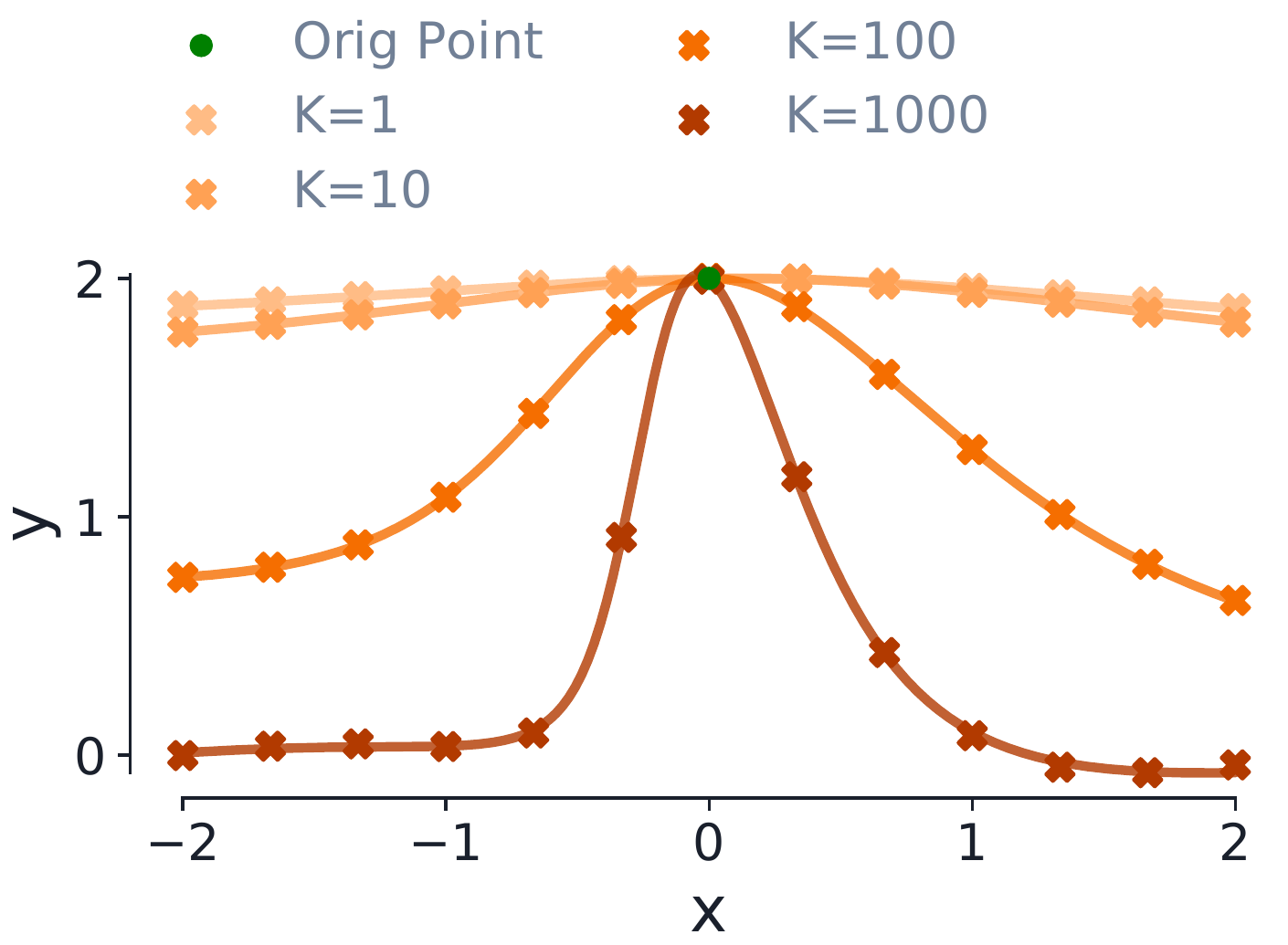}
       \caption{Unrolling}
    \end{subfigure}
    \hfill
    \begin{subfigure}{.32\textwidth}
    \includegraphics[width=\linewidth]{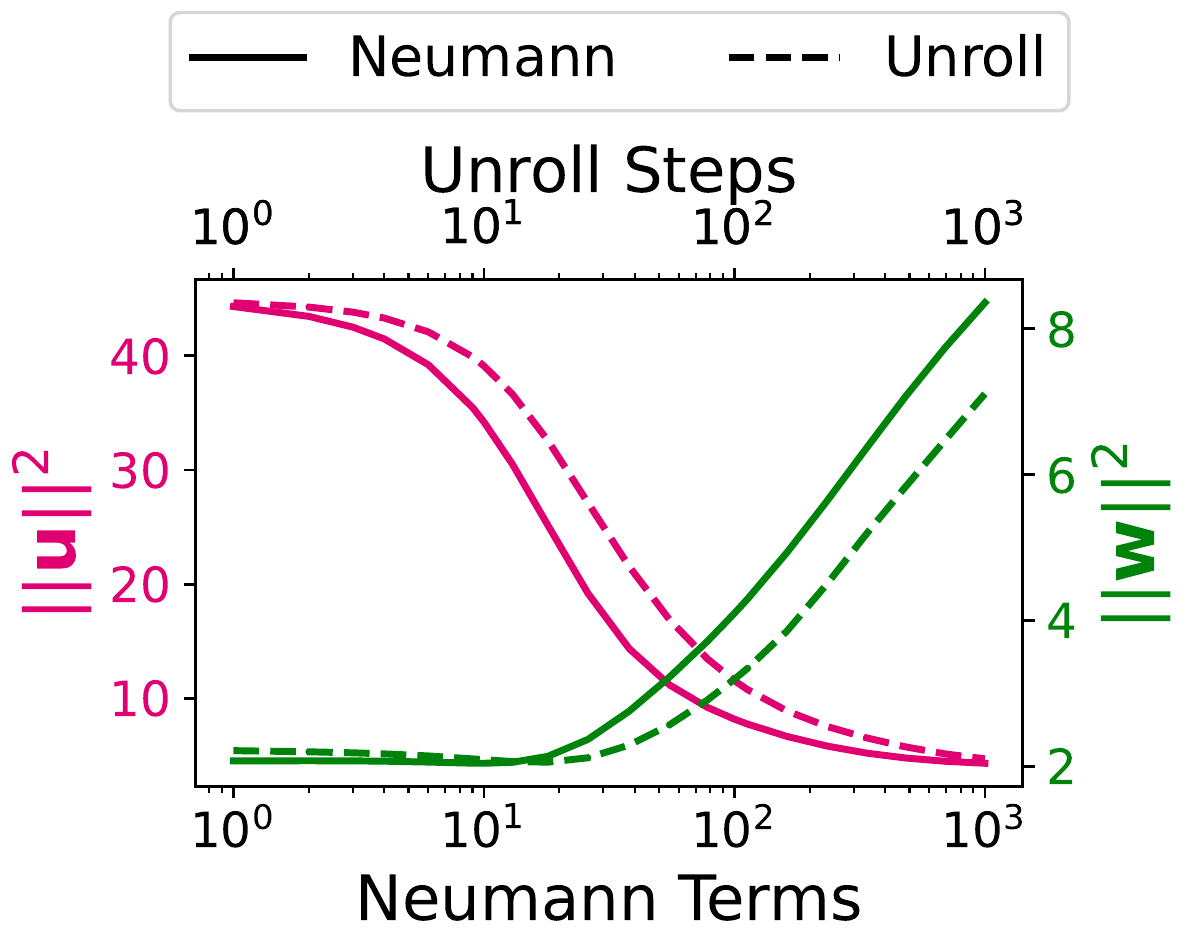}
       \caption{Parameter norms.}
    \end{subfigure}
    \vspace{-0.2cm}
    \caption{\small 1D linear regression with an overparameterized outer objective, where we train an MLP rather than performing Fourier-basis function regression. Note that here we cannot analytically compute the damped Hessian inverse as in the linear regression case, so we only include results for full-unrolls with truncated Neumann approximations of the inverse Hessian, and alternating gradient descent with various numbers of unroll steps.}
    \label{fig:overparam-mlp}
\end{figure}

\end{document}